\documentclass[letterpaper,11 pt]{article} 
\usepackage{amsmath,amsfonts,amssymb,color}
\usepackage{mathtools}
\usepackage{dsfont}
\usepackage[dvipsnames]{xcolor}
\definecolor{mygreen}{rgb}{0.0, 0.5, 0.0}
\definecolor{winered}{rgb}{0.8,0,0}
\definecolor{myblue}{rgb}{0,0,0.8}
\usepackage{tikz}
\usepackage{amsthm}
\usepackage{empheq}

\newtheorem{theorem}{Theorem}
\newtheorem{lemma}{Lemma}
\newtheorem{proposition}{Proposition}
\newtheorem{corollary}{Corollary}
\newtheorem{remark}{Remark}
\newtheorem{assumption}{Assumption}

\DeclareMathOperator*{\argmin}{\arg\!\min}
\DeclarePairedDelimiterX{\norm}[1]{\lVert}{\rVert}{#1}

\usepackage{algorithm}
\usepackage{algpseudocode}
\usepackage{epsfig}
\usepackage{array}
\usepackage{multirow}
\usepackage{epstopdf}
\usepackage{tikz}
\usepackage{relsize}
\usetikzlibrary{shapes,arrows}
\usepackage[hidelinks]{hyperref}
\hypersetup{
    colorlinks=true,
    linkcolor={winered},
    citecolor={mygreen}
}
\usepackage{cite}
\usepackage[margin=1in]{geometry}
\title{\LARGE Linear Convergence in Federated Learning: \\  Tackling Client Heterogeneity and Sparse Gradients}
\author{Aritra Mitra, Rayana Jaafar, George J. Pappas, and Hamed Hassani
\thanks{The authors are with the Department of Electrical and Systems Engineering, University of Pennsylvania. Email: {\tt \{amitra20, rayanaj, pappasg, hassani\}@seas.upenn.edu}. This work was supported by NSF Award 1837253, NSF CAREER award CIF 1943064, and the Air Force Office
of Scientific Research Young Investigator Program (AFOSR-YIP) under award FA9550-20-1-0111.}}
\date{}
\begin{document}
\maketitle
\thispagestyle{empty}
\pagestyle{empty}
\begin{abstract}
We consider a standard federated learning (FL) architecture where a group of clients periodically coordinate with a central server to train a statistical model. We develop a general algorithmic framework called \texttt{FedLin} to tackle some of the key challenges intrinsic to FL, namely objective heterogeneity, systems heterogeneity, and infrequent and imprecise communication. Our framework is motivated by the observation that under these challenges, various existing FL algorithms suffer from a fundamental speed-accuracy conflict: they either guarantee linear convergence but to an incorrect point, or convergence to the global minimum but at a sub-linear rate, i.e., fast convergence comes at the expense of accuracy. In contrast, when the clients' local loss functions are smooth and strongly convex, we show that \texttt{FedLin} guarantees linear convergence to the global minimum, despite arbitrary objective and systems heterogeneity. We then establish matching upper and lower bounds on the convergence rate of \texttt{FedLin}  that highlight the effects of intermittent communication. Finally, we show that \texttt{FedLin} preserves linear convergence rates under aggressive gradient sparsification, and quantify the effect of the compression level on the convergence rate. Our work is the first to provide tight linear convergence rate guarantees, and constitutes the first comprehensive analysis of gradient sparsification in FL. 
\end{abstract}
\section{Introduction}
In a canonical federated learning (FL) architecture, a set $\mathcal{S}$ of clients periodically communicate with a central server to find a global statistical model that solves the following problem  \cite{mcmahan,surv1}:
\begin{equation}
  \min_{x\in\mathbb{R}^d}   f(x), \hspace{2mm} \textrm{where} \hspace{1mm} f(x) =\frac{1}{m}\sum_{i=1}^{m} f_i(x).
\label{eqn:objective}
\end{equation}
Here, $m$ is the number of clients, $f_i:\mathbb{R}^{d}\rightarrow \mathbb{R}$ is the local objective function of client $i$, and $f(x)$ is the global objective function. The above formulation can be used to model a variety of common machine learning tasks such as classification, inference, and regression by casting them as specific instances of the general empirical risk minimization (ERM) framework; within this framework, $f_i(x)$ can be thought of as the local loss function of client $i$ defined over its own private data. Some of the core distinguishing tenets of the FL paradigm are as follows \cite{surv1,surv2,mcmahan,bonawitz}. First, due to privacy considerations, clients cannot directly share their local training data with the server.  Second, differences in the clients' data-sets may cause the local objectives of the clients to be  non-identical, and, in particular, to have different minima - this is known as \textit{objective} or \textit{statistical} heterogeneity. Third, due to variability in hardware (CPU, memory) and power (battery level), i.e., due to \textit{systems} or \textit{device} heterogeneity, client devices may have different computation speeds leading  to slow and straggling machines.\footnote{In addition to varying operating speeds, systems  heterogeneity can also cause certain clients to drop out unexpectedly. While the effect of partial client participation (due to drop outs) has been analyzed in quite a few recent works \cite{sahu,FedPAQ,scaffold}, an understanding of how differing computing speeds impacts  convergence rates is still very limited. This is precisely why we focus on the latter in this paper.} Finally, \textit{communication-efficiency} is a first-order concern, motivating the need to reduce the number of communication rounds, and also the size of the messages transmitted. The above considerations pose unique technical challenges that we aim to address in this paper. 

To reduce the number of communication rounds, a typical FL algorithm involves each client performing multiple local training steps in isolation on its own data  (using, for instance, mini-batch stochastic gradient descent) before communicating with the server. Due to such local steps, the popular \texttt{FedAvg} \cite{mcmahan} algorithm suffers from a ``client-drift effect" under objective heterogeneity \cite{li,malinovsky,charles,charles2,fedsplit,scaffold}: the local iterates of each client drift-off towards the minimum of their own local loss function, leading to slow convergence rates. Recently, several new algorithms such as \texttt{FedProx} \cite{sahu}, \texttt{SCAFFOLD} \cite{scaffold}, \texttt{FedSplit} \cite{fedsplit}, and \texttt{FedNova} \cite{FedNova} have been proposed as improvements to \texttt{FedAvg}. Despite these advances, there remain gaps in our understanding of the extent to which these algorithms match the guarantees of a centralized baseline.\footnote{By a centralized baseline, we refer to a setup where each client can communicate with every other client at all time steps via the server.} For instance, like \texttt{FedAvg}, \texttt{FedProx} \cite{sahu} and \texttt{FedNova} \cite{FedNova} exhibit a fundamental \textit{speed-accuracy conflict} under objective heterogeneity; see \cite{charles,charles2} and Section \ref{sec:Motivation}. In particular, when each $f_i(x)$ is smooth and strongly convex\footnote{A continuously differentiable function $f:\mathbb{R}^{d}\rightarrow \mathbb{R}$ is $L$-smooth if the gradient map $\nabla f:\mathbb{R}^{d}\rightarrow \mathbb{R}^d$ is $L$-Lipschitz. Moreover, $f(\cdot)$ is $\mu$-strongly convex if, given any $x,y\in\mathbb{R}^{d}$, the following inequality holds: 
$$
    f(y) - f(x) \geq\langle y-x, \nabla f(x) \rangle +\frac{\mu}{2}{\Vert y-x \Vert}^2.
$$}, the iterates generated by these algorithms converge linearly, but potentially to an incorrect point when the step-size (learning rate) is held constant. Thus, convergence to the minimum of the global loss function $f(x)$ necessitates diminishing step-sizes, which, in turn, leads to sub-linear convergence rates. In short, with \texttt{FedAvg}, \texttt{FedProx}, and \texttt{FedNova}, \textit{fast convergence comes at the expense of accuracy}. Importantly, this trade-off persists even when each client has access to exact, deterministic gradients of its local objective function, and all clients participate in every communication round.

Although \texttt{SCAFFOLD} \cite{scaffold} and \texttt{FedSplit} \cite{fedsplit} employ variance-reduction and operator-splitting techniques, respectively, to tackle objective heterogeneity, it is not known whether the rates in these papers are \textit{tight}, i.e., these algorithms come with no lower bounds. Moreover, neither \texttt{SCAFFOLD} nor \texttt{FedSplit} account for the effects of systems heterogeneity or compression, both of which are key challenges in FL. Indeed, due to systems heterogeneity, the number of local steps may vary across clients, causing some clients to make much less progress than others in each round \cite{FedNova}. As observed in \cite{FedNova}, this may even result in convergence to incorrect points. Furthermore, while empirical studies \cite{emp1,emp2} have revealed significant benefits of biased sparsification, theoretical guarantees for such methods in a federated setting have remained elusive. In this context, our goal is to build a better understanding of the following questions.

\begin{itemize}
    \item Given strongly convex and smooth objectives, is it possible to devise a local update scheme that guarantees linear convergence to the global minimum, despite both objective and systems heterogeneity? More generally, to what extent can we match centralized convergence rates? 
    
    \item Under client heterogeneity, how does  \textit{infrequent} (due to periodic transmissions) and \textit{imprecise} (due to compression) communication impact convergence rates?  
\end{itemize}

\subsection{Our Contributions}
In an effort to systematically answer the above questions, we make the following contributions in this paper. 

\begin{itemize}
    \item \textbf{A New Algorithm:} We develop a general algorithmic framework called \texttt{FedLin} that simultaneously accounts for objective heterogeneity, systems heterogeneity, and gradient sparsification. The key components of \texttt{FedLin} include a gradient correction term in the local update rule that carefully exploits memory; the use of client-specific learning rates; and error-feedback mechanisms at the clients and the server. 
    
    \item \textbf{Matching Centralized Convergence Guarantees under Objective and Systems  Heterogeneity:} For smooth and strongly convex losses, we show that \texttt{FedLin} resolves the speed-accuracy conflict by guaranteeing  linear convergence to the global minimum in the deterministic setting; see Theorem \ref{thm:strconv_dh}.  Moreover, for a general stochastic oracle model, \texttt{FedLin} matches the centralized  $O(1/T)$ rate (up to constants); see Theorem \ref{thm:noisy}. We then present matching rates for smooth, convex and non-convex settings as well in Theorems \ref{thm:conv_dh} and \ref{thm:nonconv_dh}. Importantly, our results hold under \textit{arbitrary} objective \textit{and} systems heterogeneity. In contrast, the only other work in FL (as far as we are aware) that investigates both objective and systems heterogeneity \cite{FedNova} provides results only for the non-convex setting, under a bounded dissimilarity assumption. Moreover, the \texttt{FedNova} algorithm in \cite{FedNova} suffers from the speed-accuracy conflict, while \texttt{FedLin} does not. 
    
$\bullet$ \textbf{Quantifying the Price of Intermittent Communication:} As one of our key contributions, in Theorem \ref{thm:lowerbnd}, we establish a lower bound for \texttt{FedLin} that matches the upper-bound in Theorem \ref{thm:strconv_dh} for smooth, strongly convex losses. In doing so, we provide the first  \textit{tight} linear convergence rate analysis in FL. Our lower bound highlights the price paid for performing multiple local steps between successive communication rounds, i.e., the effect of intermittent communication on the convergence rate. In particular, our analysis reveals, perhaps surprisingly, that there exist simple instances (involving quadratic losses) for which performing multiple local steps does not improve the rate of convergence. In this way, we provide valuable insights into the limitations of gradient-tracking/variance-reduction techniques that employ stale gradient-correction terms in the update rule.

\item \textbf{Analyzing the Impacts of Gradient Sparsification at Server and at Clients:} To further reduce communication, we separately explore the impacts of sparsifying gradients at the server side, and at the clients. We do so by equipping \texttt{FedLin} with a standard error-feedback mechanism  \cite{seide,stichsparse,alistarhsparse}. For the strongly convex and smooth setting, we show that when gradients are compressed at the server, our approach guarantees linear convergence to the true minimum without requiring any error-feedback; see Theorem \ref{thm:server_sprs1}. Using error-feedback, however, improves the convergence rate, as we establish in Theorem \ref{thm:serv_sprs2}. In contrast to server-side compression, when gradients are compressed at each client, then even with error feedback, we can establish linear convergence only to a neighborhood of the true minimum; see Theorem  \ref{thm:client_sprs}. Our work  (i) constitutes the first formal study of gradient sparsification in a federated setting; (ii) reveals key differences between up-link and down-link compression; and (iii) quantifies the effect of the compression level on the convergence rate. Notably, these results hold in the face of both objective and systems  heterogeneity. 
\end{itemize}
\newpage
\subsection{Related Work}
\begin{itemize}
\item \textbf{Federated Learning Algorithms:} Since its inception in \cite{mcmahan}, \texttt{FedAvg} has become the most popularly used algorithm for federated learning. Several papers provide a detailed theoretical analysis of \texttt{FedAvg} both in the homogeneous case when all clients minimize the same objective function \cite{stich,wang,spiridonoff,FedPAQ,haddadpourNIPS,woodworth1}, and also in the more challenging  heterogeneous setting \cite{khaled1,khaled2,haddadpour,li,koloskova}.\footnote{The analyses in \cite{wang} and \cite{koloskova} pertain to fully distributed peer-to-peer networks.}  In the latter scenario, to mitigate the client-drift that \texttt{FedAvg} suffers from and guarantee convergence to the true minimum, one needs to necessarily maintain a diminishing step-size, which, in turn, hurts the convergence rate; see \cite{charles} for an excellent discussion on this topic, and also
\cite[Theorem 4]{li}. In \cite{sahu}, the authors introduce \texttt{FedProx} - a local update method with an additional proximal term to control the client-drift. Nonetheless, as the authors explain in detail in \cite{charles2} (via both theory and experiments), \texttt{FedProx} continues to exhibit the speed-accuracy conflict. 

To improve upon the convergence guarantees of \texttt{FedAvg}, the authors in \cite{scaffold} introduce \texttt{SCAFFOLD}, where the clients maintain control variates to achieve variance-reduction. In \cite{scaffold}, the focus is on a general stochastic oracle model, for which, even in a single-worker non-federated setting, one cannot in general hope for linear convergence;\footnote{A notable exception is the class of strongly convex, smooth \textit{finite sum} stochastic optimization problems for which one can in fact establish linear convergence rates using variance reduction techniques such as SAG \cite{SAG},  SAGA \cite{SAGA}, and SVRG \cite{SVRG}.} thus, the speed-accuracy trade-off is not a point of discussion in \cite{scaffold}. Among other differences, \texttt{FedLin} enjoys a desirable fixed-point property that \texttt{SCAFFOLD} does not, in general; we discuss this in detail in Section \ref{sec:algo}. In \cite{fedsplit}, the authors focus on a deterministic setting  and introduce \texttt{FedSplit} - an algorithmic framework that builds on operator-splitting techniques and guarantees linear convergence, but only to a neighborhood of the global minimum. In contrast, our approach guarantees \textit{exact} convergence to the global minimum in the deterministic setting. Finally, we note that in a concurrent work \cite{gorbunov}, the authors develop a linearly converging algorithm called \texttt{S-Local-SVRG} for the setting where the objective function of each client can be expressed as a finite sum. 

Unlike our work, none of the papers \cite{scaffold}, \cite{fedsplit}, and \cite{gorbunov} 
consider systems heterogeneity or compression/sparsification. Moreover, while we provide a lower bound analysis for \texttt{FedLin} that highlights the effects of local steps and intermittent communication, the aforementioned works only provide upper bounds for their respective algorithms.\footnote{In \cite{woodworth1} and \cite{scaffold}, the authors provide lower bounds on the performance of \texttt{FedAvg} for the homogeneous and the heterogeneous setting, respectively. However, these bounds do not apply to our algorithm.}   

\item \textbf{Straggler-Tolerant Distributed Learning:} The topic of designing straggler-robust distributed learning/optimization algorithms has received considerable recent interest, both for the master-worker type architecture \cite{dutta,adacomm,SGC,ferdinand}, and also for peer-to-peer networks \cite{reisizadeh}. These works differ from our current setting since they either study the homogeneous case  where all client objectives are identical \cite{dutta,adacomm}, or do not consider the effect of local steps \cite{dutta,adacomm,SGC,ferdinand,reisizadeh} - a key feature of federated learning algorithms. In terms of methodology, a popular technique for straggler mitigation exploits \textit{redundancy}: the server carefully replicates the data a priori and then sends it out to the workers/clients for processing \cite{lee2017speeding, ferdinand,SGC}. However, this strategy is no longer applicable in a federated setting since the server cannot access client data.

So far, only a few works \cite{FLANP,FedNova} have investigated  systems heterogeneity in federated learning. The work in \cite{FLANP} incorporates
statistical characteristics of the clients' data to adaptively select
the clients in order to speed up the learning procedure. The work in \cite{FedNova} considers both objective and systems heterogeneity simultaneously, introduces  \texttt{FedNova}, and provides  guarantees only for the non-convex setting under a bounded dissimilarity assumption on the client loss functions. Compared to  these works, we provide results for all the three standard settings -  strongly convex, convex, and non-convex - while making no similarity assumptions whatsoever on the client loss functions. Importantly, while we show that \texttt{FedLin} guarantees linear convergence to the global minimum under \textit{arbitrary} objective and systems heterogeneity, no such result is reported in prior work.
\end{itemize}

For a survey of compression techniques in distributed optimization and machine learning, we direct the reader to Appendix \ref{sec:compression}.

\subsection{Notation and Terminology}
 Referring to \eqref{eqn:objective}, let $x^* \in \argmin_{x\in\mathbb{R}^{d}} f(x)$, and $x^*_i \in  \argmin_{x\in\mathbb{R}^{d}} f_i(x)$. Every FL algorithm mentioned in this paper operates in rounds $t\in\{1,\ldots,T\}$. Round $t$ commences with each client starting off from a common global model $\bar{x}_t$, and then  performing a certain number of local training steps; the nature of the local step is what differentiates one algorithm from another. We will denote by $x^{(t)}_{i,\ell}$ client $i$'s estimate of the model at the $\ell$-th local step of round $t$. 

In the context of gradient sparsification, we will denote by $\mathcal{C}_{\delta}:\mathbb{R}^d\rightarrow\mathbb{R}^{d}$ the \texttt{TOP-k} operator, where $\delta=d/k$, and $k\in\{1,\ldots,d\}$. Given any $x\in\mathbb{R}^{d}$, let $\mathcal{E}_{\delta}(x)$ be a set containing the indices of the $k$ largest-magnitude components of $x$. Then, the \texttt{TOP-k} operator we consider is as follows:
\begin{equation}
{\left(\mathcal{C}_\delta(x)\right)}_j =
    \begin{cases}
    (x)_j, & \text{if } j\in\mathcal{E}_{\delta}(x)\\
    0, & \text{otherwise.}
\end{cases}
\end{equation}
Here, we use $(x)_j$ to denote the $j$-th component of a vector $x$. It is apparent that a larger $\delta$ implies more aggressive compression. As an example of the \texttt{TOP-k} operation, if $x={\begin{bmatrix}1 & -5 & 2 & 3 \end{bmatrix}}'$, and $\delta=2$, then $\mathcal{C}_2(x)={\begin{bmatrix}0 & -5 & 0 & 3 \end{bmatrix}}'$.

\section{Motivation}
\label{sec:Motivation}
To motivate our work, in this section we will demonstrate how the speed-accuracy trade-off shows up in the simplest of settings when considering some recently proposed FL algorithms, namely \texttt{FedProx} \cite{sahu} and \texttt{FedNova} \cite{FedNova}. For this purpose, we will follow the analysis framework developed in \cite{charles} to illustrate the speed-accuracy trade-off for \texttt{FedAvg}. Even when clients have access to exact deterministic gradients of their loss functions, we will show that \texttt{FedProx} and  \texttt{FedNova} cannot guarantee convergence to the minimum of the global objective function with constant step-sizes. This, in turn, necessitates diminishing step-sizes, leading to sub-linear convergence rates.\footnote{In a follow-up work to \cite{charles}, the authors in \cite{charles2} generalize their framework to encompass proximal methods such as \texttt{FedProx} as well. As such, Propositions \ref{prop1:FedProx} and \ref{prop2:FedNova} in this section turn out to be special cases of the results in \cite{charles2}. We were not aware of \cite{charles2} while preparing the initial version of this paper.} 

\subsection{Problem Setup}
Inspired by \cite{charles}, we  consider a simple deterministic quadratic model where the local objective of client $i$ is given by $f_i(x)=\dfrac{1}{2}\|A_i^{1/2}(x-c_i)\|^2$, 
 where $A_i$ is a symmetric positive-definite matrix. Hence, $x_i^*=c_i$ is the local minimizer of $f_i(x)$. The global objective function is given by
\begin{equation}
\label{setup_cs}
    f(x)=\dfrac{1}{m}\sum\limits_{i\in \mathcal{S}}f_i(x)=\dfrac{1}{m}\sum\limits_{i\in \mathcal{S}}\dfrac{1}{2}\|A_i^{1/2}(x-c_i)\|^2,
\end{equation}
where $x^*=\Big(\sum\limits_{i\in \mathcal{S}}A_i\Big)^{-1}\Big(\sum\limits_{i\in \mathcal{S}}A_ic_i \Big)$ is the global minimizer of $f(x)$. For each $i\in\mathcal{S}$, since $A_i$ is positive-definite, note that there exists $\mu_i,L_i>0$, such that $\mu_iI \preceq A_i \preceq L_iI$. In particular, $\mu_i=\lambda_{min}(A_i)$ and $L_i=\lambda_{max}(A_i)$ satisfy the preceding inequality. Let us also  define $L=\max_{i\in\mathcal{S}} L_i$.

\subsection{Surrogate Optimization Problem}
 We begin by assuming that all clients perform the same number of local steps $H$. As a result, the sole aspect of heterogeneity that we consider is the one caused by  differences in the clients' local objective functions. To mitigate client-drift, the recently proposed \texttt{FedProx} scheme restricts the local client updates to remain close to the global model by adding a proximal term to the local objective functions. To avoid any inaccuracies introduced by stochastic gradients or client sampling, we consider a deterministic version of \texttt{FedProx} where all clients participate in each round, and have access to the exact gradients of their respective local loss functions. For a given initial global model $\bar{x}_1$, the update rule of \texttt{FedProx} is as follows: 
\begin{equation}
x_{i,\ell+1}^{(t)}=x_{i,\ell}^{(t)}-\eta \bigg(\nabla f_i(x_{i,\ell}^{(t)})+\beta(x_{i,\ell}^{(t)}-\bar{x}_t) \bigg), \hspace{1mm} \ell=0,\cdots,H-1,\hspace{1mm} i \in \mathcal{S}; \hspace{4mm} 
\bar{x}_{t+1}=\dfrac{1}{m}\sum\limits_{i \in \mathcal{S}} x_{i,H}^{(t)}.
\label{eqn:FedProx}
\end{equation}
The following result characterizes the output of \texttt{FedProx}. For its proof, see Appendix \ref{app:proofprop1_prox}.
\begin{proposition}
\label{prop1:FedProx}
For any step-size $\eta>0$, $T$ rounds of \texttt{FedProx} amount to performing $T$ rounds of parallel GD on the \textbf{surrogate} optimization problem given by
\begin{align}\label{surr_problem_prox}
    \min_x \tilde{f}(x)=\min_x \dfrac{1}{m}\sum \limits_{i\in \mathcal{S}} \tilde{f}_i(x)=\min_x \dfrac{1}{m}\sum \limits_{i\in \mathcal{S}} \dfrac{1}{2}\bigg\|\bigg(\sum\limits_{\ell=0}^{H-1}[I-\eta (A_i+\beta I)]^{\ell}A_i\bigg)^{1/2}(x-c_i)\bigg \|^2.
\end{align}
\end{proposition}
In Appendix \ref{app:proofprop1_prox}, we also provide additional insights on the client surrogate functions $\tilde{f}_i(x)$ and how they relate to the true client functions $f_i(x)$. Proposition \ref{prop1:FedProx} shows that even when the clients perform the same number of local updates, \texttt{FedProx} ends up minimizing a surrogate objective function \eqref{surr_problem_prox}  whose minimum may not, in general, coincide with the minimum of \eqref{setup_cs}, namely $x^*$,  i.e., \texttt{FedProx} may converge to wrong solutions (see the numerical results in Section \ref{numb_Prox_Nova}).\footnote{Two notable cases where the minimum of the surrogate objective function matches that of the true objective function are as follows: (i) $H$=1, i.e., when the clients perform just one local step, and their iterates evolve synchronously;  and (ii) $f_i(x)=f(x), \forall i\in\mathcal{S}, \forall x\in\mathbb{R}^{d}$, i.e., when the clients have identical loss functions.} Moreover, as equation \eqref{surr_problem_prox} shows, the learning rate $\eta$ has a direct impact on the amount of \emph{distortion} introduced to the original optimization problem \eqref{setup_cs}. Finally, it should be noted that with $\beta=0$, \texttt{FedProx} reduces to \texttt{FedAvg}, and our previous observations continue to hold. 

To capture systems heterogeneity, suppose client $i$ performs $\tau_i$ local iterations per round. Under this model, the recently proposed \texttt{FedNova} algorithm  \cite{FedNova} aggregates a normalized version of the cumulative local gradients to update the global model. In what follows, we will analyze the output of \texttt{FedNova} for the case where clients employ gradient descent in their local steps.\footnote{Although \texttt{FedNova} can technically accommodate any local solver whose  accumulated gradients are expressible as a linear combination of local gradients, we deliberately choose gradient descent, the simplest of solvers, to isolate the impact of normalized aggregation - the essence of the \texttt{FedNova} scheme. We do so to highlight that a weighted aggregation of local gradients, on its own, may not be enough to resolve the speed-accuracy trade-off.} To that end, define $\tau_{eff}\triangleq {1}/m\sum\limits_{i\in \mathcal{S}} \tau_i$, and $\alpha_i\triangleq \tau_{eff}/\tau_i$, $\forall i \in \mathcal{S}$. Then, for an initial global model $\bar{x}_1$, the update rule of \texttt{FedNova} is as follows: 
\begin{equation}
x_{i,\ell+1}^{(t)}=x_{i,\ell}^{(t)}-\eta \nabla f_i(x_{i,\ell}^{(t)}), \hspace{1mm} \ell=0,\cdots,\tau_i-1,\hspace{1mm} i \in \mathcal{S}; \hspace{4mm}
\bar{x}_{t+1}=\bar{x}_t-\dfrac{\eta}{m}\sum\limits_{i \in \mathcal{S}} \alpha_i \sum \limits_{\ell=0}^{\tau_i-1} \nabla f_i(x_{i,\ell}^{(t)}).
\label{eqn:FedNova}
\end{equation}
The following result characterizes the output of \texttt{FedNova}. For its proof, see Appendix \ref{app:proofprop2_nova}.
\begin{proposition}
\label{prop2:FedNova}
For any step-size  $\eta>0$, $T$ rounds of \texttt{FedNova} amount to performing $T$ rounds of parallel GD on the \textbf{surrogate} optimization problem given by
\begin{align}\label{surr_problem_nova}
    \min_x \tilde{f}(x)=\min_x \dfrac{1}{m}\sum \limits_{i\in \mathcal{S}} \tilde{f}_i(x)=\min_x \dfrac{1}{m}\sum \limits_{i\in \mathcal{S}} \dfrac{1}{2}\bigg\|\bigg(\sum\limits_{\ell=0}^{\tau_i-1}[I-\eta A_i]^{\ell}\alpha_iA_i\bigg)^{1/2}(x-c_i)\bigg \|^2.
\end{align}
\end{proposition}
\textbf{Main Takeaways:} Proposition \ref{prop2:FedNova} shows that in the presence of both objective and systems  heterogeneity, \texttt{FedNova} minimizes a surrogate objective function whose minimum may be far off from $x^*$. Aligning with our conclusions for \texttt{FedProx}, observe from \eqref{surr_problem_nova} that using a larger learning rate $\eta$ in \texttt{FedNova} introduces more \emph{distortion} to the original problem. In Section \ref{numb_Prox_Nova}, we will further show that there are particular instances of problem  \eqref{setup_cs}, for which, a diminishing step-size $\eta$ is \textit{necessary} to ensure convergence to the true minimum $x^*$ when employing \texttt{FedProx} or \texttt{FedNova}; diminishing step-sizes lead to sub-linear convergence rates, contributing to the speed-accuracy trade-off. Proposition \ref{prop1:FedProx} suggests that this trade-off is not caused by differences in the computing speeds of the clients, but rather originates from heterogeneity in the clients' objectives. \textit{Proposition \ref{prop2:FedNova} complements this result by showing that a weighted aggregation of accumulated local gradients may not suffice to resolve the speed-accuracy trade-off.}

The main message we want to convey here is that even for deterministic settings, there are non-trivial challenges posed by objective and systems heterogeneity that only get amplified when one additionally considers biased compression. For such scenarios, it is not at all apparent whether (and to what extent) one can match even the basic centralized benchmark of achieving linear convergence for smooth, strongly convex loss functions. To focus on the above unresolved issues, we will primarily consider a deterministic model in this paper. \textit{Nonetheless, the general approach we develop applies to the stochastic setting as well, as demonstrated by Theorem \ref{thm:noisy} in Section \ref{sec:speed_acc}.}


\begin{figure}[t]
  \centering
  \includegraphics[width=1\linewidth]{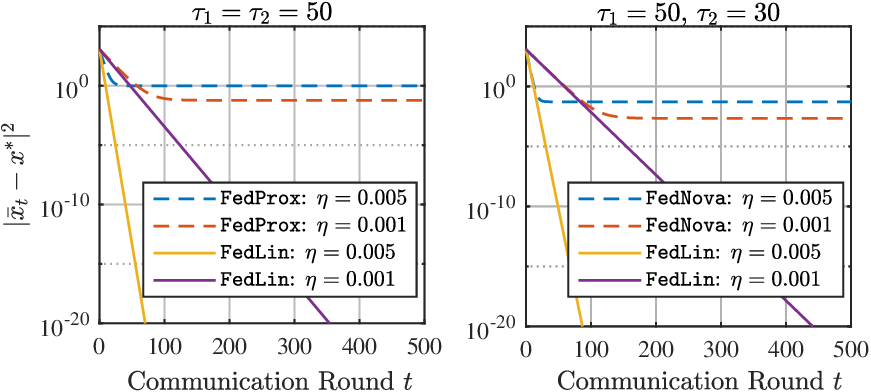}
  \caption{Simulations comparing \texttt{FedProx}, \texttt{FedNova},  and \texttt{FedLin} for $2$ clients with quadratic losses as in \eqref{eqn:2client}. \textit{\textbf{Left}}: Clients perform the same number of local steps, $H=50$. For \texttt{FedProx}, lower values of $\eta$ increase the accuracy at the expense of a slower convergence rate. \textit{\textbf{Right}}: Client $1$ performs $\tau_1=50$ local steps and client $2$ performs $\tau_2=30$ local steps per communication round. For lower values of $\eta$, \texttt{FedNova} becomes more accurate at the expense of a slower convergence  rate. \texttt{FedLin} significantly outperforms both \texttt{FedProx} and \texttt{FedNova} in terms of convergence speed and accuracy.}
\label{fig1_CaseStudy}
\end{figure}

\subsection{Numerical Illustration}
\label{numb_Prox_Nova}
We now consider an instance of problem \eqref{setup_cs} where two clients with simple scalar quadratic loss functions aim to minimize the global objective function using \texttt{FedProx} (see equation \eqref{eqn:FedProx}), \texttt{FedNova} (see equation \eqref{eqn:FedNova}), and our proposed algorithm \texttt{FedLin} that we develop in Section \ref{sec:algo}.  Specifically, the local objective functions of the clients are as follows:
\begin{equation}
    f_1(x)=\dfrac{1}{2}(x-3)^2; \hspace{2mm} 
    f_2(x)=(x-50)^2.
\label{eqn:2client}
\end{equation}

For the above instance, let the minima of the surrogate objective functions in \eqref{surr_problem_prox} and \eqref{surr_problem_nova} be denoted by  $\tilde{x}^*_1$ and $\tilde{x}^*_2$, respectively. In other words, while $\tilde{x}^*_1$ corresponds to the convergence point of \texttt{FedProx}, $\tilde{x}^*_2$ corresponds to that of \texttt{FedNova}. To clearly demonstrate that $\tilde{x}^*_1$ and $\tilde{x}^*_2$ may not coincide with the true minimum ${x}^*$, we consider two simple scenarios. We begin by assuming that both clients perform two local iterations, i.e. $\tau_1=\tau_2=2$. Then, for $0<\eta <{1}/{(2+\beta)}$, the convergence error of \texttt{FedProx} is given by
\begin{align}
    \label{error_prox}
    |\tilde{x}^*_1-x^*|=\dfrac{94\eta}{3[6-\eta (5+3\beta)]}. 
\end{align}

For some fixed value of $\beta$, and any fixed value of $\eta > 0$, note that the convergence error of \texttt{FedProx} (as shown above)  will never be eliminated. Also, observe that as $\eta \xrightarrow[]{} 0$, the convergence error in $\eqref{error_prox}$ vanishes. It follows that a decaying learning rate $\eta$ is \emph{necessary} to achieve convergence to the global minimum $x^*$. Moreover, lower values of $\eta$ decrease the convergence error at the expense of a slower convergence rate. For $\beta=0$, \texttt{FedProx} becomes equivalent to \texttt{FedAvg}, and the mismatch between $\tilde{x}^*_1$ and $x^*$  remains for any fixed value of $\eta > 0$. Next, to model systems heterogeneity, we now assume that clients perform a different number of local steps. In particular, suppose $\tau_1=2$ and $\tau_2=3$. Then, for $0<\eta<{1}/{2}$, the convergence error of \texttt{FedNova} is given by
\begin{align}
    \label{error_nova}
   |\tilde{x}^*_2-x^*|= \dfrac{94\alpha_1\eta^2}{3[\eta^2\alpha_1-\eta(3\alpha_1+4 \alpha_2)+3\alpha_1+4\alpha_2]},
\end{align}
where $\alpha_1=5/6$ and $\alpha_2=5/4$. Once again, no fixed value of $\eta > 0$ can eliminate the convergence error of \texttt{FedNova}, as shown above. Our results above indicate that even when clients perform only two or three local iterations on simple scalar quadratics, both \texttt{FedProx} and \texttt{FedNova} might inevitably converge to the stationary point of a mismatched optimization problem. As we discussed earlier, these results further suggest that a different type of local solver is needed to eliminate the observed trade-off between speed and accuracy.

In Figure \ref{fig1_CaseStudy}, we see how \texttt{FedProx} and \texttt{FedNova} converge to incorrect minimizers even for simple instances with two clients and deterministic, quadratic losses. 
Moreover, both schemes exhibit a trade-off between speed and accuracy: lower values of the learning rate $\eta$ increase the convergence accuracy at the expense of a slower convergence rate. On the other hand, \texttt{FedLin} eliminates the speed-accuracy trade-off by converging fast to the true minimum, despite heterogeneity in both the clients' local loss functions and operating speeds. Indeed, the theoretical results provided in the following sections confirm this observation. 

\section{Proposed  Algorithm: \texttt{FedLin}}
\begin{algorithm}[t]
\caption{\texttt{FedLin}} \label{algo:Main}
\begin{algorithmic}[1]
\State \textbf{Input:} Client step-sizes (learning rates) $\eta_i,i\in\mathcal{S}$, compression levels $\delta_c$ and $\delta_s$, initial iterate $\bar{x}_1\in\mathbb{R}^{d}$, $g_1=\nabla f(\bar{x}_1)$, initial compression errors $\rho_{i,1}=0, \forall i\in\mathcal{S}$ and $e_1=0$
\For {$t=1,\ldots,T$} 
\For {$i=1,\ldots,m$} 
\For {$\ell=0,\ldots,\tau_i-1$} 
\vspace{2mm}
\State 
$x^{(t)}_{i,\ell+1} \gets x^{(t)}_{i,\ell}-\eta_i(\nabla f_i(x^{(t)}_{i,\ell})-\nabla f_i(\bar{x}_{t})+g_t)$; \hspace{1mm} $x^{(t)}_{i,0}=\bar{x}_t$ \Comment{\textcolor{blue!80}{Local training steps}}
\vspace{2mm}
\EndFor
\State Client $i$ transmits $x^{(t)}_{i,\tau_i}$ to server
\EndFor
\State Server transmits $\bar{x}_{t+1}=1/m\sum_{i\in\mathcal{S}}x^{(t)}_{i,\tau_i}$ to each client
\For {$i=1,\ldots,m$} 
\State Client $i$ transmits $h_{i,t+1}=\mathcal{C}_{\delta_c}(\rho_{i,t}+\nabla f_i(\bar{x}_{t+1}))$ to server \Comment{\textcolor{blue!80}{Client-level sparsification}}
\State Client $i$ updates error: $\rho_{i,t+1}\gets\rho_{i,t}+\nabla f_i(\bar{x}_{t+1})-h_{i,t+1}$ 
\EndFor
\State Server transmits $g_{t+1}=\mathcal{C}_{\delta_s}(e_t+1/m\sum_{i\in\mathcal{S}}h_{i,t+1})$ to each client \Comment{\textcolor{blue!80}{Server-level sparsification}}
\State Server updates error: $e_{t+1}\gets e_t+1/m\sum_{i\in\mathcal{S}}h_{i,t+1}-g_{t+1}$ 
\EndFor
\end{algorithmic}
\end{algorithm}
\label{sec:algo}
In this section, we develop our proposed algorithm \texttt{FedLin},  formally described in Algorithm \ref{algo:Main}. \texttt{FedLin} is initialized from a common global iterate $\bar{x}_1\in\mathbb{R}^{d}$. For simplicity, we assume that $g_1 = \nabla f(\bar{x}_1)$, i.e., every client has access to the true gradient of $f(\cdot)$ initially; we can allow $g_1$ to be arbitrary as well without affecting the convergence guarantees. Complying with the structure of existing federated learning algorithms, \texttt{FedLin} proceeds in rounds $t\in\{1,\ldots,T\}$. In each round $t$, starting from a common global model $\bar{x}_t$, each client $i$ performs $\tau_i$ local training steps in parallel, as per line 5 of Algorithm \ref{algo:Main}.\footnote{Recall that $x^{(t)}_{i,\ell}$ denotes client $i$'s estimate of the model at the $\ell$-th local iteration within the $t$-th round, where $\ell\in\{0,\ldots,\tau_i-1\}.$ We use $\bar{x}_t=1/m\sum_{i\in\mathcal{S}}x^{(t)}_{i,0}$ and $\bar{x}_{t+1}=1/m\sum_{i\in\mathcal{S}}x^{(t)}_{i,\tau_i}$ to represent the global estimate of the model at the beginning,  and at the end of round $t$, respectively.} The key features of our local update rule are as follows:  \textit{exploiting past gradients} to account for objective heterogeneity, using \textit{client-specific step-sizes} to tackle systems heterogeneity, and employing \textit{error-feedback} to account for gradient sparsification. We now discuss each of these features in detail. 

\textbf{Exploiting Memory:} To gain intuition regarding the local step in line 5, note that the ideal local update at client $i$ takes the form  $x^{(t)}_{i,\ell+1} = x^{(t)}_{i,\ell}-\eta_i \nabla f(x^{(t)}_{i,\ell})$. However, this requires client $i$ to have access to the gradients of all other clients - which it does not, since clients do not communicate between rounds. To get around this, client $i$ \textit{exploits memory}, and uses the gradient of the global function $\nabla f(\bar{x}_t)$ from the beginning of round $t$ (when the clients last communicated) as a guiding direction in its update rule.  However, since $\nabla f(\bar{x}_t)$ is evaluated at a stale point $x^{(t)}_{i,0}=\bar{x}_t$, client $i$ subtracts off the stale gradient $\nabla f_i(\bar{x}_t)$ from $\nabla f(\bar{x}_t)$, and adds in the gradient of its own function $\nabla f_i(x^{(t)}_{i,\ell})$ evaluated at the most recent local iterate $x^{(t)}_{i,\ell}$.\footnote{As long as $x^{(t)}_{i,\ell}$ is not too far off from $\bar{x}_t$, and each $f_i(x)$ is smooth, one should expect $\nabla f(\bar{x}_t)$ to serve as a reasonable proxy for $\nabla f(x^{(t)}_{i,\ell})$.} This results in the update rule:  
$$x^{(t)}_{i,\ell+1} = x^{(t)}_{i,\ell}-\eta_i(\nabla f_i(x^{(t)}_{i,\ell})-\nabla f_i(\bar{x}_{t})+\nabla f(\bar{x}_t)).$$

Our  local update rule in line 5 is precisely of the above form, where $g_t$ is an inexact version of $\nabla f(\bar{x}_t)$, the inexactness being a consequence of gradient sparsification. 

\textbf{Client-Specific Learning Rates:} Since the number of local training steps can differ across clients due to systems heterogeneity, the amount of `client-drift' can also vary across clients. When each client $i$ performs $\tau_i$ local-steps, our analysis reveals that the bound on the drift-term $\Vert x_{i,\ell} - \bar{x}_t \Vert$ scales linearly in $\tau_i$ (see Lemma \ref{lemma:drift_conv_ct} in Appendix \ref{app:proofs_CT}). Accordingly, to compensate for such drift at client $i$, the step-size $\eta_i$ needs to be chosen to vary inversely with the number of local steps $\tau_i$. In fact, the requirement that $\eta_i \propto 1/{\tau_i}$ also turns out to be necessary (see Theorem \ref{thm:lowerbnd}), providing further motivation for the choice of client-specific learning rates in \texttt{FedLin}. Roughly speaking, the intuition here is that faster devices have a tendency to drift more than the slower ones. Accordingly, controlling the drift of faster devices necessitates lower learning rates for those devices.

\textbf{Error-Feedback:} Let us now discuss the compression module.  We employ standard error-feedback mechanisms \cite{seide,stichsparse,alistarhsparse} at both the server and the clients to account for gradient sparsification. At client $i$, $\rho_{i,t}$ represents the accumulated error due to gradient sparsification. At the end of round $t$, instead of simply compressing $\nabla f_i (\bar{x}_{t+1})$,  client $i$ instead compresses $\nabla f_i (\bar{x}_{t+1})+\rho_{i,t}$, to account for gradient coordinates that were not transmitted in the past. It then updates the accumulated error as per line 12. An analogous description applies to the error-feedback mechanism at the server, where $e_t$ represents the aggregate error at the beginning of round $t$.

\textbf{Overall Workflow of \texttt{FedLin}:} In each round $t$, clients perform local steps in parallel starting from $\bar{x}_t$, as per line 5. Subsequently, each client uploads its local model to the server, the server averages them to generate the global estimate $\bar{x}_{t+1}$, and then broadcasts $\bar{x}_{t+1}$ to each client. Each client $i$ then uploads a compressed gradient vector  $h_{i,t+1}$ to the server as per line 11, and the server broadcasts back (as per line 14) a compressed version $g_{t+1}$ of $1/m\sum_{i\in\mathcal{S}}h_{i,t+1}$ for each client to use in its local updates in round $t+1$. The parameters of \texttt{FedLin} are the client step-sizes $\{\eta_i\}_{ i\in\mathcal{S}}$, and the compression levels $\delta_c$ and $\delta_s$ at the clients and at the server, respectively.

\textbf{Related Algorithmic Approaches:} In the related but different setting of distributed optimization, we note that the idea of exploiting past gradients has been used to design \textit{gradient-tracking} algorithms \cite{qu,nedic, xi, pu, xin}. In the context of FL, this idea is also related to the variance-reduction technique employed in \texttt{SCAFFOLD} \cite{scaffold}. A major difference of \texttt{FedLin} with the above works is that none of them consider the effect of systems heterogeneity or biased compression. In particular, accounting for the inexact gradient term $g_t$ in our update rule introduces new technical challenges that we address in this paper.

\textit{Fixed-point property of} \texttt{FedLin}. There are some additional basic differences between \texttt{FedLin} and \texttt{SCAFFOLD}. To see this, consider the update rule of \texttt{FedLin} without sparsification:  $x^{(t)}_{i,\ell+1} = x^{(t)}_{i,\ell}-\eta_i(\nabla f_i(x^{(t)}_{i,\ell})-\nabla f_i(\bar{x}_{t})+\nabla f(\bar{x}_t))$. Now suppose the global model $\bar{x}_t$ at the beginning of round $t$ has already converged to $x^*$. Since $x^{(t)}_{i,0}=\bar{x}_t, \forall i \in \mathcal{S}$, and $\nabla f(x^*)=0$, it is easy to see that the iterates of the clients do not evolve any further, as one would ideally want. \textit{Thus, the global optimum  $x^*$ can be viewed as a fixed-point of the update rule for \texttt{FedLin} .} Adapting to our notation, and considering the case when there is no noise in the gradients, the update rule of \texttt{SCAFFOLD} takes the form $x^{(t)}_{i,\ell+1} = x^{(t)}_{i,\ell}-\eta(\nabla f_i(x^{(t)}_{i,\ell})-c_i+c)$, where $c_i$ is a `control-variate' maintained by client $i$, and $c$ is the average of the $c_i$'s. Importantly, the control variates $\{c_i\}_{i\in\mathcal{S}}$ used in round $t$ of \texttt{SCAFFOLD} contain stale terms from round $t-1$. As a result, even if $\bar{x}_t=x^*$, it may very well be that $(\nabla f_i(\bar{x}_t)-c_i+c) \neq 0$, causing the iterates of the clients to move away from $x^*$, and requiring further rounds of communication to average out the imbalance; this is clearly undesirable. Thus, the fixed-point property we discussed for \texttt{FedLin} does not hold in general for \texttt{SCAFFOLD}. 

In the following sections, we will show that \texttt{FedLin} guarantees linear convergence rates despite objective heterogeneity, systems heterogeneity, and gradient sparsification.  

\section{Matching Centralized Rates under Objective and Systems Heterogeneity}
\label{sec:speed_acc}
In this section, we will investigate the performance of 
\texttt{FedLin} in the face of both objective and systems heterogeneity. To focus solely on the aspect of client heterogeneity, we will assume throughout this section that gradients are neither sparsified by the clients, nor by the server, i.e., $\delta_{c}=\delta_{s}=1$. Based on the second assumption, observe that $\rho_{i,t}=0,e_t=0,\forall i\in\mathcal{S},\forall t\in\{1,\ldots,T\}.$ Thus, the local update rule for each client in line 5 of \texttt{FedLin} simplifies to
\begin{equation}
    x^{(t)}_{i,\ell+1} = x^{(t)}_{i,\ell}-\eta_i(\nabla f_i(x^{(t)}_{i,\ell})-\nabla f_i(\bar{x}_{t})+\nabla f(\bar{x}_t)).
\end{equation}
Let us denote by $\kappa=L/\mu$ the condition number of an $L$-smooth and $\mu$-strongly convex function. Also,  let $\eta_i={\bar{\eta}}/{\tau_i}, \forall i\in\mathcal{S}$, where $\bar{\eta} \in (0,1)$ is a flexible parameter that we will specify based on context. We are now ready to state the main results of this section. 
\begin{theorem} \label{thm:strconv_dh} (\textbf{Strongly convex case}) Suppose each $f_i(x)$ is $L$-smooth and $\mu$-strongly convex. Moreover, suppose $\tau_i\geq 1,\forall i\in\mathcal{S}$, and $\delta_c=\delta_s=1$. Then, with $\eta_i=\frac{1}{6L\tau_i}, \forall i \in \mathcal{S}$, \texttt{FedLin} guarantees:
\begin{equation}
    f(\bar{x}_{T+1})-f(x^*) \leq {\left(1-\frac{1}{6\kappa}\right)}^T (f(\bar{x}_{1})-f(x^*)). 
\label{eqn:str_rate}
\end{equation}
\end{theorem}

\begin{theorem} \label{thm:conv_dh} (\textbf{Convex case}) Suppose each $f_i(x)$ is $L$-smooth and   convex. Moreover, suppose $\tau_i \geq 1,\forall i\in\mathcal{S}$, and $\delta_c=\delta_s=1$. Then, \texttt{FedLin} guarantees the following. 
\begin{itemize}
    \item With $\eta_i=\frac{1}{10L\tau_i}, \forall i \in \mathcal{S}$, 
\begin{equation}
    f\left(\frac{1}{T}\sum\limits_{t=1}^{T}\bar{x}_t\right)-f(x^*) \leq \frac{10L}{T}\left({\Vert\bar{x}_1-x^*\Vert}^2-{\Vert\bar{x}_{T+1}-x^*\Vert}^2\right). 
\label{eqn:conv_ct_1}
\end{equation}
\item With $\eta_i=\frac{1}{6L\tau_i}, \forall i \in \mathcal{S}$,
\begin{equation}
  f(\bar{x}_T)-f(x^*)  \leq \frac{f(\bar{x}_1)-f(x^*)}{c(f(\bar{x}_1)-f(x^*))(T-1)+1}, \textrm{where}\hspace{1mm} c=\frac{1}{12{\Vert\bar{x}_1-x^*\Vert}^2L}.
\label{eqn:conv_ct_2}
\end{equation}
\end{itemize}
\end{theorem}

\begin{theorem} \label{thm:nonconv_dh} (\textbf{Non-convex case}) Suppose each $f_i(x)$ is $L$-smooth.  Moreover, suppose $\tau_i\geq 1,\forall i\in\mathcal{S}$, and $\delta_c=\delta_s=1$. Then, with $\eta_i=\frac{1}{26L\tau_i}, \forall i \in \mathcal{S}$, \texttt{FedLin} guarantees:
\begin{equation}
    \min_{t\in[T]} {\Vert \nabla f(\bar{x}_t) \Vert}^2 \leq \frac{52L}{T} (f(\bar{x}_1)-f(\bar{x}_{T+1})).
\end{equation} \nonumber
\end{theorem}

\begin{theorem} \label{thm:nonconv_PL} (\textbf{Non-convex PL case}) Suppose each $f_i(x)$ is $L$-smooth, and $f(x)$ satisfies the Polyak-Lojasiewicz (PL) condition \cite{karimi} with a constant $\mu > 0$, i.e., for any $x\in\mathbb{R}^{d}, 
    {\Vert \nabla f(x) \Vert}^2 \geq 2\mu (f(x)-f(x^*))$, 
where $x^* \in \argmin_{x \in \mathbb{R}^d} f(x)$. Moreover, suppose $\tau_i\geq 1,\forall i\in\mathcal{S}$, and $\delta_c=\delta_s=1$. Then, with $\eta_i=\frac{1}{26L\tau_i}, \forall i \in \mathcal{S}$, \texttt{FedLin} guarantees:
\begin{equation}
    f(\bar{x}_{T+1})-f(x^*) \leq {\left(1-\frac{1}{26\kappa}\right)}^T (f(\bar{x}_{1})-f(x^*)). \nonumber
\end{equation}
\end{theorem}

\textbf{Noisy Case Analysis:} We now analyze the performance of \texttt{FedLin} under a general stochastic oracle model. For each $i\in \mathcal{S}$ and $x\in\mathbb{R}^d$, let $q_i(x)$ be an unbiased estimate of the gradient $\nabla f_i(x)$ with variance bounded above by $\sigma^2$. We consider the following noisy update rule for \texttt{FedLin}:
$$
x^{(t)}_{i,\ell+1} = x^{(t)}_{i,\ell}-\eta_i(q_i(x^{(t)}_{i,\ell})- q_i(\bar{x}_{t})+q(\bar{x}_t)),$$ 

where $q(x) \triangleq 1/m \sum_{i\in\mathcal{S}} q_i(x), \forall x\in\mathbb{R}^d.$  We then have the following result. 

\begin{theorem} \label{thm:noisy}  (\textbf{Strongly convex case with noise}) Consider the above stochastic oracle model. Suppose each $f_i(x)$ is $L$-smooth and $\mu$-strongly convex. Moreover, suppose $\tau_i\geq 1,\forall i\in\mathcal{S}$, and $\delta_c=\delta_s=1$. For each $i\in\mathcal{S}$, let  $\eta_i=\frac{\bar{\eta}}{\tau_i}$, where $\bar{\eta}\in (0,1)$ satisfies $\bar{\eta}  
< \frac{1}{6L}$. Then, $\forall t\in [T]$, \texttt{FedLin} guarantees:
\begin{equation}
    \mathbb{E}[{\Vert \bar{x}_{t+1}-x^* \Vert}^2] \leq \left(1-  \frac{\bar{\eta} \mu}{2} \right) \mathbb{E}[{\Vert \bar{x}_{t}-x^* \Vert}^2] +25 \bar{\eta}^2 \sigma^2.
\end{equation}
\end{theorem}

The proofs of Theorems \ref{thm:strconv_dh}, \ref{thm:conv_dh}, \ref{thm:nonconv_dh}, \ref{thm:nonconv_PL}, and \ref{thm:noisy} are all provided in  Appendix \ref{app:proofs_CT}. 

\textbf{Main Takeaways:} From Theorems \ref{thm:strconv_dh}, \ref{thm:conv_dh}, \ref{thm:nonconv_dh}, and \ref{thm:nonconv_PL}, we note that \texttt{FedLin} matches the guarantees of centralized gradient descent (up to constants) for the strongly convex, convex, and non-convex settings, respectively, despite \textit{arbitrary} heterogeneity in the client objectives, and \textit{arbitrary} device heterogeneity. For the strongly convex setting in particular, \texttt{FedLin} resolves the speed-accuracy conflict by guaranteeing linear convergence to the global minimum. For the non-convex settings, convergence is guaranteed to a first-order stationary point of $f(x)$. \textit{To the best of our knowledge, ours is the only paper to provide such comprehensive theoretical guarantees that account for both forms of client heterogeneity.} In fact, all our results continue to hold even when the operating speeds of the client machines vary across rounds, i.e., $\tau_i$ is allowed to be a function of $t$. Each client $i$ can simply adjust its learning rate $\eta_i \propto 1/ \tau_i(t)$ \textit{locally} to account for such variations.

The bound for the noisy case in Theorem \ref{thm:noisy} resembles that of centralized SGD \cite{nemirovski}: With a time-varying parameter $\bar{\eta_t} = O(1/t)$, we get the standard $O(1/T)$ rate after $T$ rounds. Note that to arrive at this result, the assumptions we make on the stochastic oracle model are exactly the same as for centralized SGD, namely unbiased gradients with bounded variance. 

\textbf{Comparison with Related Work}: In the recent paper \cite{fedsplit}, the authors propose \texttt{FedSplit}, where the local steps involve solving a proximal optimization problem. For a deterministic setting with strongly-convex and smooth loss functions, the authors show (see \cite[Theorem 1]{fedsplit}) that \texttt{FedSplit} guarantees linear convergence up to an error floor $b$, where $b$ is a residual error that arises from the inexactness of the proximal evaluation step. Importantly, $b$ does not shrink to $0$. Thus, like \texttt{FedAvg} and \texttt{FedProx}, \texttt{FedSplit} fails to guarantee linear convergence to $x^*$. Empirically, we observe that \texttt{FedSplit} diverges on certain instances; see Appendix \ref{fedsplit:num}. Compared to these algorithms, we see from Theorem \ref{thm:strconv_dh} that \texttt{FedLin} guarantees linear convergence to $x^*$. Notably, the linear convergence rate we obtain in Theorem \ref{thm:strconv_dh} under \textit{both} objective and systems heterogeneity is the \emph{best rate we know of in FL}, and matches that of \texttt{SCAFFOLD} \cite{scaffold} where only objective heterogeneity is considered.

The model of systems heterogeneity that we study here is taken from \cite{FedNova}. From an algorithmic standpoint, the key to the \texttt{FedNova} approach in \cite{FedNova} is to perform \textit{normalized aggregation} of local gradients. Importantly, in \texttt{FedNova}, clients have identical step-sizes, and the normalized aggregation takes place at the end of each round. In contrast, the client-specific step-sizes in \texttt{FedLin} affect how the local iterates evolve \textit{within} each round - a subtle, but key technical distinction from \texttt{FedNova}. Moreover, in Section  \ref{sec:Motivation} we saw, both theoretically and empirically, that normalized aggregation of gradients, on its own, does not resolve the speed-accuracy trade-off. In terms of results, \cite{FedNova} provides guarantees  only for the non-convex setting  under a bounded dissimilarity assumption. Our results, on the other hand, cover all the three standard settings - strongly convex, convex, and non-convex - without requiring any bounded dissimilarity assumption.

In a concurrent work \cite{FLANP}, the authors propose a novel straggler-resilient federated learning technique termed \texttt{FLANP} that is based on the idea of adaptively selecting the participating clients during the training process. When the samples of all clients come from the same distribution, the authors show that selecting the fastest clients first, and then gradually selecting the slower ones, leads to overall improvements in the learning time. Unlike the setting in \cite{FLANP} where fast machines can only help towards learning the global model, objective heterogeneity may cause fast machines to drift more (towards the minima of their own loss function) in our formulation, i.e.,  \textit{fast machines can potentially hurt as well}. This naturally requires a more cautious approach to handling systems heterogeneity: we penalize faster machines more with lower learning rates. 

\subsection{The Price of Intermittent  Communication}
\label{sec:comm_eff}
In this section, we take a closer look at the effect of performing multiple local steps on the convergence rate. To do so, we assume that all clients perform the same number of local steps $H$, i.e., there is no communication for $H$ consecutive time-steps between two communication rounds. We seek to answer: \textit{What is the price paid for such intermittent communication?} To answer this question, consider a centralized baseline where each client can communicate with every other client at all times (i.e., even between rounds). In this case, since each client can always access $\nabla f(x)$, running standard gradient descent on each client in parallel provides a guarantee of the form
\begin{equation}
f(\bar{x}_{T+1})-f(x^*) \leq \exp(-\frac{1}{\kappa}TH) (f(\bar{x}_{1})-f(x^*))
\label{eqn:centralized}
\end{equation}
after $T$ rounds, with $H$ synchronized local iterations within each round.  However, in the heterogeneous case, based on Theorem \ref{thm:strconv_dh}, we can only guarantee
\begin{equation}
f(\bar{x}_{T+1})-f(x^*) \leq \exp(-\frac{1}{6\kappa}T) (f(\bar{x}_{1})-f(x^*))
\label{eqn:het_rate}
\end{equation}
after $T$ communication rounds.\footnote{To arrive at \eqref{eqn:het_rate} from \eqref{eqn:str_rate}, we used the inequality ${(1-z)}^r \leq \exp(-zr)$, where   $z\in(0,1)$ and $r\geq 0$.}  Comparing \eqref{eqn:het_rate} with \eqref{eqn:centralized}, observe that we lose out by a factor of $H$ in the exponent relative to the centralized case. Both in the centralized scenario and in \texttt{FedLin}, each client queries the gradient of its local objective $H$ times in each round, thereby making $TH$ gradient queries over $T$ rounds. The key difference, however, lies in the fact that in the centralized scenario, given any two clients $i,j\in\mathcal{S}$,  $\nabla f_i(\cdot) $ and $\nabla f_j(\cdot)$ are always evaluated at the same synchronized point. In contrast, with \texttt{FedLin}, these gradients are computed at potentially different points $x^{(t)}_{i,\ell}$ and $x^{(t)}_{j,\ell}$ at each local step $\ell$. Consequently, after $T$ rounds, while a total of $mTH$ gradient calls are made both in the centralized case and in \texttt{FedLin}, the former translates to $m$ parallel and identical implementations of $TH$ gradient descent steps on $f(x)$, while the latter does not, in general. 

To sum up the above discussion, relative to a centralized baseline, \texttt{FedLin} incurs the same computational cost in terms of gradient queries, and reduces communication by a factor of $H$ at the expense of a convergence rate that is slower by a factor of $H$. We emphasize here that just as with $\texttt{FedLin}$, $H$ does not show up in the convergence rate (exponent) of algorithms like \texttt{FedSplit} \cite{fedsplit} and \texttt{SCAFFOLD} \cite{scaffold} either. 

\textbf{Utility of local steps:} {Is there any benefit to performing a large number of local steps?} From Proposition \ref{prop:homogeneous} in Appendix \ref{app:proofthm_speed_acc}, we see that there is indeed a clear benefit in the homogeneous setting: the convergence rate or exponent scales linearly with $H$, and hence, performing more local steps directly translates to faster convergence. This is also true for a centralized method, as seen from \eqref{eqn:centralized}.  Unfortunately, Theorem \ref{thm:strconv_dh} is unable to capture any such advantage for the heterogeneous setting. The primary reason for the slower convergence rate (relative to a centralized baseline) stems from the need to set $\eta \propto 1/H $ to mitigate client-drift under objective heterogeneity. Loosely speaking, any progress made during the course of the $H$ local iterations gets ``washed away" when the learning rate $\eta$ is scaled down by $H$. At this stage, one may conjecture that the above phenomenon is simply an artifact of a conservative analysis of \texttt{FedLin}, and that a more refined analysis will reveal the utility of performing more local steps even in the heterogeneous setting. Our next result suggests otherwise. 
\begin{theorem} \label{thm:lowerbnd} (\textbf{Lower bound for \texttt{FedLin}}) 
Suppose $\delta_c=\delta_s=1$, and $\tau_i=H, \eta_i=\eta,  \forall i\in\mathcal{S}$. Then, given any $L \geq 14$ and $H\geq 2$, there exists an instance involving $2$ clients where each $f_i(x), i\in\{1,2\}$, is $1$-strongly convex and $L$-smooth, and an initial condition $\bar{x}_1$, such that \texttt{FedLin} initialized from $\bar{x}_1$ generates a sequence of iterates $\{\bar{x}_t\}$ satisfying the following for any $T\geq 1$:
\begin{equation}
    {\Vert \bar{x}_{T+1}-x^* \Vert}^2 \geq \exp{(-4T)} {\Vert \bar{x}_{1}-x^* \Vert}^2; \hspace{4mm} f(\bar{x}_{T+1}) - f(x^*) \geq \exp(-4T) (f(\bar{x}_{1}) - f(x^*)).
\label{eqn:lbnd}
\end{equation}
\end{theorem}
\begin{proof}
See Appendix \ref{app:lowerbnd}.
\end{proof} 
\textbf{Main Takeaways:} There are several important implications of Theorem \ref{thm:lowerbnd}. First, it complements Theorem \ref{thm:strconv_dh} by providing a matching lower bound.  \textit{We believe ours is the first work to provide a \textit{tight} linear convergence rate analysis in FL}: \cite{scaffold} and \cite{fedsplit} only provide upper-bounds for \texttt{SCAFFOLD} and \texttt{FedSplit}, respectively. Second, our analysis of Theorem \ref{thm:lowerbnd} in Appendix \ref{app:lowerbnd} indicates that there are problem  instances where setting $\eta\propto 1/H$ is in fact \textit{necessary} to guarantee convergence to $x^*$. As a result, for such problem instances, no matter how many local steps $H$ each client performs, the error at the end of $T$ rounds remains bounded below by an $H$-independent quantity, as is apparent from \eqref{eqn:lbnd}. Perhaps surprisingly, we show in Appendix \ref{app:lowerbnd} that the lower bound in Theorem \ref{thm:lowerbnd} even applies to simple instances with non-identical quadratic losses (across clients) \textit{where every $f_i(x)$ has the same minimum!} This is particularly insightful since it highlights the limitations of exploiting stale gradient terms in the local update rule (as is done in both \texttt{FedLin} and \texttt{SCAFFOLD}), and suggests the need for more informed updating schemes that explicitly take into account the level of statistical heterogeneity.

\begin{remark}
Although the lower-bound in Theorem \ref{thm:lowerbnd} is derived for \texttt{FedLin}, we conjecture that unless additional assumptions are made on the level of statistical/objective heterogeneity, such a bound will apply to \textbf{any} federated learning algorithm that involves periodic communication with the server. After all, there is only so much one can hope to achieve by exploiting memory. 
\label{rem:conjecture}
\end{remark}

\section{Impact of Gradient Sparsification at Server}
\label{sec:serv_sparse}
In this section, we will address the following question. \textit{For strongly convex and smooth deterministic functions, and in the presence of both objective and systems heterogeneity, can we still hope for linear convergence to the global minimum when gradients are sparsified at the server?} 

Interestingly, we will show that not only is it possible to converge linearly to the global  minimum, it is possible to do so \textit{without any form of error-feedback}. Moreover, this claim holds regardless of how aggressive the server is in its sparsification scheme. In particular, the server might even choose to transmit just one single component of the aggregated gradient vector, i.e., our results hold even when $k=1$, implying $\delta_s=d$. 

To isolate the impact of server-level sparsification, we will assume throughout this section that gradients are not sparsified at the clients, i.e., $\delta_c=1$. Consequently, $h_{i,t+1}=\nabla f_i(\bar{x}_{t+1}), \forall i \in \mathcal{S}, \forall t \in \{1,\dots,T\}$. We will begin by considering a simpler variant of \texttt{FedLin} where $g_{t+1}$ in line 14 of \texttt{FedLin} is instead updated as follows
\begin{equation}
    g_{t+1}=\mathcal{C}_{\delta_s}\left(\frac{1}{m}\sum\limits_{i\in\mathcal{S}} \nabla f_i(\bar{x}_{t+1})\right)=\mathcal{C}_{\delta_s}\left(\nabla f(\bar{x}_{t+1})\right). 
\label{eqn:server_grad}
\end{equation}
Moreover, there is no use of error-feedback at the server side, i.e., line 15 of the algorithm is skipped. For this variant of \texttt{FedLin}, we have the following result. 
\begin{theorem} \label{thm:server_sprs1} (\textbf{Sparsification at server with no error-feedback}) Suppose each $f_i(x)$ is $L$-smooth and $\mu$-strongly convex. Moreover, suppose $\tau_i\geq 1,\forall i\in\mathcal{S}$, and $\delta_c=1$. Consider a variant of \texttt{FedLin} where line 14 is replaced by equation \eqref{eqn:server_grad}, and line 15 is skipped, i.e., there is no error-feedback.  Then, with $\eta_i=\frac{1}{2\left(2+\sqrt{\delta_s}\right)L\tau_i}, \forall i \in \mathcal{S}$, this variant of \texttt{FedLin} guarantees
\begin{equation}
    f(\bar{x}_{T+1})-f(x^*) \leq {\left(1-\frac{1}{2\delta_s \left(2+\sqrt{\delta_s}\right)\kappa}\right)}^T (f(\bar{x}_{1})-f(x^*)). 
\end{equation}
\end{theorem}
\begin{proof}
See Appendix \ref{app:servsparse1}.
\end{proof}

\textbf{Main Takeaways:} From Theorem \ref{thm:server_sprs1}, we note that even without error-feedback, it is possible to linearly converge to the global  minimum; the rate of convergence, however, is inversely proportional to $\delta^{\frac{3}{2}}_s$.  {Thus, Theorem \ref{thm:server_sprs1} quantifies the trade-off between the level of sparsification at the server, and the rate of convergence.}  Observe that when there is no compression of information at the server, i.e., when $\delta_s=1$, we exactly recover Theorem \ref{thm:strconv_dh}. 

At this stage, one may ask: \textit{Is there any potential benefit to employing error-feedback when gradients are sparsified at the server?} Our next result answers this question in the affirmative. 

\begin{theorem} \label{thm:serv_sprs2} (\textbf{Sparsification at server with error-feedback}) Suppose each $f_i(x)$ is $L$-smooth and $\mu$-strongly convex. Moreover, suppose $\tau_i\geq 1,\forall i\in\mathcal{S}$, and $\delta_c=1$. Let the step-size for client $i$ be chosen as $\eta_i = \frac{1}{72L\delta_s \tau_i}$. Then, \texttt{FedLin}  guarantees:
\begin{equation}
 f(\bar{x}_{T+1}) - f(x^*)  \leq 2 \kappa { \left(1-\frac{1}{96 \delta_s \kappa} \right)}^T \left(f(\bar{x}_1)-f(x^*)\right).
\end{equation}
\end{theorem}
\begin{proof}
See Appendix \ref{app:servsparse2}.
\end{proof}
\textbf{Main Takeaways}: Comparing the guarantee of Theorem \ref{thm:server_sprs1} with that of Theorem \ref{thm:serv_sprs2}, we note that the convergence rate is inversely proportional to $\delta^{\frac{3}{2}}_s$ in the former, and inversely proportional to $\delta_s$ in the latter. Thus, the main message here is that  \textit{employing error-feedback leads to a faster convergence rate by improving the dependence of the rate on the parameter $\delta_s$}. 
\section{Impact of Gradient Sparsification at Clients}
\label{sec:client_sparse}
In this section, we will turn our attention to the case when gradients are sparsified at the clients prior to being transmitted to the server. Throughout this section, we will assume that gradients are not compressed any further at the server side, i.e., $\delta_s=1$. To proceed, we will need to make the following bounded gradient dissimilarity assumption.
\begin{assumption}
\label{ass:bndgrad}
There exist constants $C \geq 1$ and $D \geq 0$ such that the following holds $\forall x\in\mathbb{R}^d$:
\begin{equation}
    \frac{1}{m}\sum_{i=1}^{m}{\Vert \nabla f_i(x) \Vert}^2 \leq C{\Vert \nabla f(x) \Vert}^2+D.
\end{equation}
\end{assumption}

The following is the main result of this section. 
\begin{theorem} \label{thm:client_sprs}  (\textbf{Sparsification at clients with error-feedback}) Suppose each $f_i(x)$ is $L$-smooth and $\mu$-strongly convex, and suppose  Assumption \ref{ass:bndgrad} holds.  Moreover, suppose $\tau_i\geq 1,\forall i\in\mathcal{S}$, and $\delta_s=1$. Let the step-size for client $i$ be chosen as $\eta_i=\frac{\bar{\eta}}{\tau_i}$, where $\bar{\eta}\in (0,1)$ satisfies $\bar{\eta}  
\leq \frac{1}{72L\delta_c C}$. Then, \texttt{FedLin} guarantees:
\begin{equation}
    {\Vert \bar{x}_{T+1} - x^* \Vert}^2 \leq 2 { \left(1-\frac{3}{4}\bar{\eta} \mu \right)}^T {\Vert \bar{x}_{1} - x^* \Vert}^2 + \frac{16}{3}\bar{\eta}\left(\frac{6}{\delta_c C}+\delta_c\right)\frac{D}{\mu}.
\label{eqn:clientsprs_guarantee}
\end{equation}
\end{theorem}
\begin{proof}
See Appendix \ref{app:clientsparse}.
\end{proof}

The next result is an immediate corollary of Theorem \ref{thm:client_sprs}. 

\begin{corollary}
Suppose the conditions stated in Theorem \ref{thm:client_sprs} hold. Then, \texttt{FedLin}  guarantees:
\begin{equation}
    f(\bar{x}_{T+1})-f(x^*) \leq 2\kappa { \left(1-\frac{3}{4}\bar{\eta} \mu \right)}^T \left(f(\bar{x}_1)-f(x^*)\right) + \frac{8L}{3}\bar{\eta}\left(\frac{6}{\delta_c C}+\delta_c\right)\frac{D}{\mu}.
\end{equation}
\end{corollary}

\textbf{Main Takeaways:} Intuitively, one would expect that sparsifying gradients at each client prior to aggregation at the server would inject more errors than when gradients are first accurately aggregated at the server, and then the aggregated gradient vector is sparsified: Theorems \ref{thm:server_sprs1} and \ref{thm:client_sprs} support this intuition. For the former, we neither required error-feedback nor Assumption \ref{ass:bndgrad} to guarantee linear convergence to the global  minimum; for the latter, even with error-feedback and the bounded gradient dissimilarity assumption, we can establish linear convergence to \textit{only a neighborhood of the global minimum}, in general.  From inspection of \eqref{eqn:clientsprs_guarantee}, we note that the size of this neighborhood scales linearly with $D$ - a measure of objective heterogeneity. In particular, when $D=0$,  the iterates $\bar{x}_t$ converge exactly to the global minimum $x^*$, which also happens to be the minimum of each $f_i(x)$. 

\textbf{Comparison with Related Work:} While there is a wealth of literature that explores the effect of \textit{unbiased random quantization} in distributed settings \cite{qsgd,wen,khirirat,DIANA,ADIANA,horvathq}, there are only a handful of papers \cite{FedPAQ,Fedcomgate} that also consider the effect of local steps in federated learning. Different from all these works, the \texttt{TOP-k} operator we consider here \textit{does not} provide an unbiased estimate of the true gradient, dictating the need for a new potential-function based proof technique that we develop in Appendix \ref{app:clientsparse}. Recently,  in \cite[Theorem 13]{beznosikov}, the authors show that in a single worker centralized setting, applying a \texttt{TOP-k} operator on the gradient (in standard gradient descent) leads to a linear convergence rate that exhibits a $1/\delta$ dependence on the compression parameter $\delta$, where $\delta=d/k$. Theorems \ref{thm:serv_sprs2} and \ref{thm:client_sprs} significantly generalize this result by showing that even in a multi-client federated setting with objective and systems  heterogeneity, one can still achieve convergence rates that have the same dependence on the compression parameter as in the centralized case.

\textbf{Extensions:} In Sections \ref{sec:serv_sparse} and \ref{sec:client_sparse}, we only analyzed the impact of gradient sparsification for the strongly convex and smooth setting. It should be fairly straightforward to extend these  results to the convex and non-convex cases following the proof techniques in Appendix \ref{app:proofs_CT}. Moreover, we studied the effect of compressing information at the server and at the clients separately, with the goal of identifying the key differences between each of these mechanisms; we anticipate that the combined effect of gradient sparsification at the server and at the clients should also be a simple extension of our results.  
\section{Experimental Results}
\label{sec:experiments}
In this section, we provide numerical results for \texttt{FedLin} on a least squares regression problem to confirm the developed theory. In particular, we generate various synthetic data-sets which correspond to different levels of heterogeneity in the clients' objective functions, and evaluate the proposed scheme in terms of its convergence rate and accuracy. This allows us to clearly characterize the effect of statistical heterogeneity on the convergence of \texttt{FedLin}. In addition, to account for systems heterogeneity, we randomly generate the number of local training steps at each client. In Appendix \ref{add_exp}, we  provide additional numerical results for \texttt{FedLin} on a logistic regression problem. Moreover, in Appendix \ref{app:Noisy_FedLin_exp}, we evaluate the performance of \texttt{FedLin} under a noisy oracle to confirm the findings of Theorem \ref{thm:noisy}. All the simulations are performed on a machine running Windows 8.1 with a 1.8 GHz Intel Core i7 processor, using \texttt{MATLAB} R2019.

\subsection{Problem Setup} 
For each client $i \in \mathcal{S}$, the design matrix $A_i \in \mathbb{R}^{m_i\times d}$ and the corresponding response vector $b_i \in \mathbb{R}^{m_i}$ assume a linear relationship of the form $A_ix+\varepsilon_i=b_i$, where $x \in \mathbb{R}^{d}$ is some weight vector, and $\varepsilon_i \in \mathbb{R}^{m_i}$ is a disturbance term. Hence, the least squares regression problem defining the global objective function is given by 

\begin{align}
\label{lsp}
\min_{x \in \mathbb{R}^d}f(x)=\min_{x \in \mathbb{R}^d} \dfrac{1}{m}\sum\limits_{i=1}^{m}f_i(x)=\min_{x \in \mathbb{R}^d} \dfrac{1}{m}\sum\limits_{i=1}^{m}\dfrac{1}{2}\|A_ix-b_i\|^2.
\end{align}
The client objective functions, $f_i(x)$, are both strongly convex and differentiable.  Assuming that all the design matrices are full column rank, problem \eqref{lsp} admits a unique minimizer given by
\begin{align}
\label{lsr_sol}
x^*=\bigg(\sum\limits_{i=1}^m A_i^TA_i\bigg)^{-1}\bigg(\sum\limits_{i=1}^m A_i^Tb_i\bigg).
\end{align}
\subsection{Synthetic Data} 
To generate synthetic data, for each client $i \in \mathcal{S}= \{1, \cdots,20\}$, we generate the design matrix $A_i$ and the response vector $b_i$ according to the model $b_i=A_ix_i+\varepsilon_i$,  where $x_i\in \mathbb{R}^{100}$, $A_i \in \mathbb{R}^{500 \times 100} $, $b_i \in \mathbb{R}^{500}$, and $\varepsilon_i \in \mathbb{R}^{500}$. In particular, the entries of the design matrix are modeled as $[A_i]_{jk}\stackrel{i.i.d.}{\sim}\mathcal{N}(0,1)$ for $j\in \{1, \cdots,500\}$ and $k \in \{1, \cdots, 100\}$. The disturbance vectors are drawn independently as $\varepsilon_i \sim \mathcal{N}(0,0.5I_{500})$, $\forall i \in \mathcal{S}$. In order to characterize the level  of objective heterogeneity in the clients' local loss  functions, we model the local true parameter of client $i$ as $x_i \sim \mathcal{N}(u_i,1)$, where $u_i \sim \mathcal{N}(0,\alpha)$ and $\alpha \geq 0$. Thus, $\alpha$ controls the amount of  heterogeneity across the clients' objectives. To model the effect of systems  heterogeneity, we allow the number of local steps to vary across clients. In particular, for each client $i \in \mathcal{S}$, the number of local iterations is drawn independently from a uniform distribution over the range $[2,100]$, i.e, $\tau_i \in [2,100]$, $\forall i \in \mathcal{S}$.

\begin{figure}[t]
  \centering
  \includegraphics[width=1\linewidth]{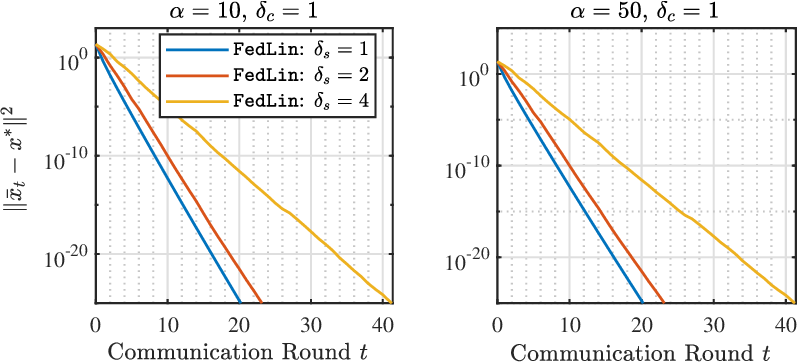}
  \caption{Simulation results for \texttt{FedLin} where gradient sparsification is implemented only at the server side (i.e. $\delta_c=1$) without any  error-feedback. The global objective function is an instance of \eqref{lsp} with $m=20$ clients. Each client performs $\tau_i$ local iterations per communication round, where $\tau_i$ is drawn independently and uniformly at random from $[2,100]$, $\forall i \in \mathcal{S}$. The constant $\bar{\eta}$ is fixed at $10^{-2}$ across all clients. We consider two levels of gradient sparsification, namely $\delta_s \in \{2,4\}$, which correspond to implementing a \texttt{TOP-50} and a \texttt{TOP-25} sparsification operator, respectively. For comparison, we also plot the case where no sparsification is implemented at the server side, i.e. $\delta_s=1$. \textit{\textbf{Left}}: The level of statistical heterogeneity is set to $\alpha=10$. As the value of $\delta_s$ increases, \texttt{FedLin} continues to guarantee linear convergence to the true minimizer \eqref{lsr_sol}, but at correspondingly slower rates. \textit{\textbf{Right}}: For a higher level of statistical heterogeneity $(\alpha=50)$, both the convergence speed and accuracy remain unchanged and unaffected for all the considered values of $\delta_s$.} 
\label{fig2_exp_server}
\end{figure}

\begin{figure}[t]
  \centering
  \includegraphics[width=1\linewidth]{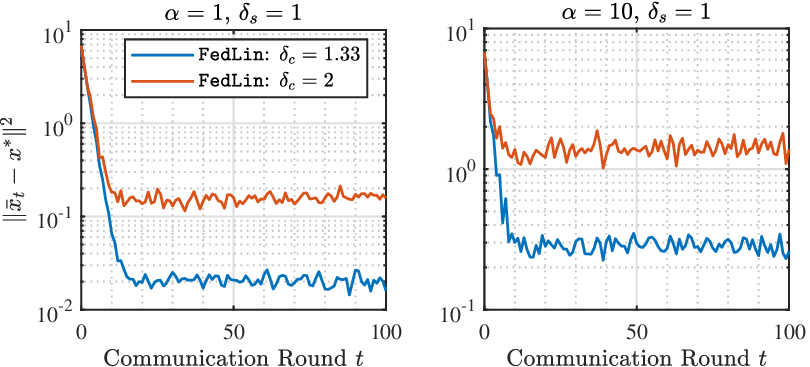}
  \caption{Simulation results for \texttt{FedLin} where gradient sparsification is implemented only at the clients' side, i.e. $\delta_s=1$.  The global objective function is an instance of \eqref{lsp} with $m=20$ clients. Each client performs $\tau_i$ local iterations per communication round, where $\tau_i$ is drawn independently and uniformly at random from $[2,100]$, $\forall i \in \mathcal{S}$. The constant $\bar{\eta}$ is fixed at $5\times 10^{-4}$ across all clients. We consider two levels of gradient sparsification, namely $\delta_c \in \{4/3,2\}$, which correspond to implementing a \texttt{TOP-75} and a \texttt{TOP-50} sparsification operator respectively. \textit{\textbf{Left}}: The level of statistical heterogeneity is set to $\alpha=1$. \texttt{FedLin} converges linearly with a non-vanishing convergence error which increases as the value of $\delta_c$ increases.  \textit{\textbf{Right}}: Keeping the level of gradient sparsification the same, increasing the level of statistical heterogeneity to  $\alpha=10$ causes the convergence error of \texttt{FedLin} to increase.}
\label{fig3_exp_client}
\end{figure}
\subsection{Simulation Results}
\textbf{Gradient Sparsification at Server:}
We first consider a variant of \texttt{FedLin} where gradient sparsification is implemented only at the server side without any error-feedback. In particular, we consider the cases where $\delta_s \in \{2,4\}$, which correspond to the implementation of a \texttt{TOP-50} and a \texttt{TOP-25} operator on the communicated gradients, respectively. For comparison, we also plot the resulting performance when no gradient sparsification is implemented at the server. To examine the effect of statistical heterogeneity on the performance of \texttt{FedLin}, we generate two synthetic datasets corresponding to two different levels of heterogeneity in the clients' local objectives, namely $\alpha=10$ and $\alpha=50$. As illustrated in Figure \ref{fig2_exp_server}, for any level of gradient sparsification on the server side, \texttt{FedLin} achieves linear convergence to the global  minimum in the presence of both objective and systems heterogeneity, conforming with Theorem \ref{thm:server_sprs1}.  Furthermore, for any level of gradient sparsification, both the convergence speed and accuracy of \texttt{FedLin} remain unaffected as the level of heterogeneity in the clients' objective functions increases.

\textbf{Gradient Sparsification at Clients:} Next, we implement gradient sparsification only at the clients' side, i.e. $\delta_s=1$.  In particular, we consider the cases where $\delta_c \in \{4/3, 2\}$, which correspond to the implementation of a \texttt{TOP-75} and a \texttt{TOP-50} operator on the communicated local gradients, respectively. Once again, we generate two synthetic datasets with different levels of objective  heterogeneity, namely $\alpha=1$ and $\alpha=10$. As illustrated in Figure \ref{fig3_exp_client}, unlike the server case, \texttt{FedLin} with sparsification at the clients' side converges linearly, but with a non-vanishing convergence error that increases as the value of $\delta_c$ increases.  Furthermore, the level of objective  heterogeneity has a direct impact on the convergence error. In particular, for the same level of gradient sparsification, higher levels of objective heterogeneity result in larger values of the convergence error. Our observations above align with the conclusions that were drawn from Theorem \ref{thm:client_sprs}. 
\section{Conclusions and Open Problems}
We studied the basic question of achieving linear convergence in a federated learning setup, subject to objective heterogeneity, systems heterogeneity, and  communication constraints. To do so, we proposed and systematically analyzed a new algorithmic framework called \texttt{FedLin}. In particular, we showed that \texttt{FedLin} (i) guarantees linear convergence to the global minimum under arbitrary objective and systems heterogeneity, and (ii) preserves linear convergence rates despite aggressive gradient sparsification. We also established a matching lower-bound for \texttt{FedLin}, thereby providing the first \textit{tight} linear convergence rate analysis in FL. We conclude with some open problems. 

\begin{itemize}
    \item Identifying classes of machine learning problems where performing local steps can provably help remains a major open research direction. For the finite-sum optimization problem that we considered in this work, our results were unable to quantify any benefits of local computation in the presence of objective heterogeneity. Whether this is a shortcoming of our algorithm, or something to be expected of \textit{any} federated learning algorithm, is unclear at this point. Deriving \textit{algorithm-independent} lower bounds will shed light on this topic. In this context, we aim to resolve the conjecture in Remark \ref{rem:conjecture} as part of future work.  
    
    \item Our results in Section \ref{sec:speed_acc} do not address a fundamental question: \textit{Do fast machines help or hurt?} In the homogeneous setting, the fact that fast client machines can only help is rather obvious, since each client seeks to optimize the same function, and hence, the faster a machine, the more progress it makes towards the global  minimum in each round. The situation is much more nuanced when the client objectives have different minima. Consider an example where the minimum $x^*_i$ of a certain client $i$'s local objective $f_i(x)$ is far away from the true minimum $x^*$. Now suppose client machine $i$ is fast, and performs several local steps in each round. In doing so, its local iterates may naturally drift towards its own local minimum $x^*_i$, which is far off from $x^*$. Clearly, this is undesirable. Thus, the fast machine $i$ does not contribute much towards minimizing the true objective in this case. This simple example illustrates that our understanding of objective heterogeneity coupled with device heterogeneity is still quite limited, and requires deeper exploration.
    
\item When gradients are compressed at each client, we can establish linear convergence only to a neighborhood of the global minimum based on our algorithm. Perhaps there are better techniques that can guarantee linear convergence \textit{exactly} to the global minimum? For instance, instead of transmitting the top-$k$ components of a gradient vector, what if we transmit those $k$ components of the gradient vector that have changed the most since the last round? Whether this scheme - which focuses on communicating ``innovations" - yields any provable improvements over our current approach remains to be seen. 

\item We plan to extend our framework to the finite-sum stochastic optimization setting, where one can hope for linear convergence guarantees. Also, given that privacy is a major concern in federated learning, it would be interesting to explore whether our algorithm can be coupled with techniques that guarantee differential-privacy. 
\end{itemize}

\bibliographystyle{unsrt} 
\bibliography{refs}
\newpage
\appendix

\section{Related Works on Compression Techniques in Distributed Optimization and Machine Learning}
\label{sec:compression}
In order to alleviate the communication bottleneck in large-scale distributed computing systems, a complementary direction to reducing communication rounds is to compress the information being exchanged. Broadly speaking, compression either involves quantization \cite{seide,strom, wen,qsgd,khirirat,DIANA,ADIANA,horvathq,FedPAQ,Fedcomgate,reisizadeh2019exact} to reduce the precision of transmitted information, or biased sparsification \cite{emp1,emp2,alistarhsparse,stichsparse,reddySignSGD,reddystich,beznosikov,lin_comp} to transmit only a few components of a vector with the largest magnitudes. While early work on quantization focused on  quantizing gradients \cite{wen,qsgd,khirirat}, improved convergence guarantees were later obtained  in \cite{DIANA}, where the authors proposed \texttt{DIANA} - a new algorithm based on quantizing \textit{differences of gradients}. The \texttt{DIANA} technique was further generalized in \cite{horvathq}, and an accelerated variant was proposed recently in \cite{ADIANA}. It should be noted that neither \texttt{DIANA} nor its variants account for the effect of performing local steps.  

As far as gradient sparsification is concerned, although empirical studies \cite{emp1,emp2} have revealed the benefits of aggressive sparsification techniques, theoretical guarantees for such methods, especially in a distributed setting, are few and far between. In our work, we study one such sparsification technique - the \texttt{TOP-k} operator - by exploiting the error-feedback mechanism. Notably, the error-feedback idea that we use was first introduced in \cite{seide} to study \texttt{1-bit SGD}. Follow-up work has explored this idea for studying \texttt{SignSGD}  in \cite{reddySignSGD}, and sparse SGD in \cite{alistarhsparse,stichsparse,reddystich} - all for single-worker settings. Very recently, \cite{beznosikov} and \cite{lin_comp} provide theoretical results on biased sparsification for a master-worker type distributed architecture; however, their analysis does not account for the effect of local steps. In light of the above discussion, we contribute to the strand of literature on compressed distributed learning by (i) performing the first analysis of gradient sparsification in a federated setting that involves local steps, and (ii) accounting for both objective and systems heterogeneity in our analysis. 
\newpage
\section{Proof of Propositions 1 and 2} 
The analysis in this section follows the techniques introduced in \cite{charles}. 
\label{app:proofprop1_prox}
\subsection*{Proof of Proposition \ref{prop1:FedProx}: FedProx}
We consider a deterministic version of \texttt{FedProx} where clients perform $H$ local iterations according to the update rule in equation \eqref{eqn:FedProx}. Let
$$g_{i,\ell}^{(t)}\triangleq \nabla f_i(x_{i,\ell}^{(t)})+\beta(x_{i,\ell}^{(t)}-\bar{x}_t), q_i^{(t)}\triangleq\sum\limits_{\ell=0}^{H-1}g_{i,\ell}^{(t)}, \hspace{2mm} \textrm{and} \hspace{2mm} q^{(t)}\triangleq \dfrac{1}{m}\sum\limits_{i\in \mathcal{S}} q_i^{(t)}.$$ Then, Lemma \ref{lemma1_prox} provides a recursive relationship in terms of $g^{(t)}_{i,\ell}$.\\
\begin{lemma}\label{lemma1_prox}
For each $i\in \mathcal{S}$, the following holds for all $\ell \in \{0, \cdots, H-1\}$: 
\begin{align}
    g^{(t)}_{i,\ell+1}=[I-\eta (A_i+\beta I)]g^{(t)}_{i,\ell}.
\end{align}
\end{lemma}
\begin{proof}
By definition of $g_{i,\ell}^{(t)}$, we have
\begin{align}\label{l11p}
    g^{(t)}_{i,\ell}=(A_i +\beta I)x^{(t)}_{i,\ell}-A_ic_i-\beta \bar{x}_{t}.
\end{align}
Hence, we may write
\begin{align}\label{l12p}
    x^{(t)}_{i,\ell}=(A_i+\beta I)^{-1}\big[g^{(t)}_{i,\ell}+A_ic_i+\beta \bar{x}_{t}\big],
\end{align}
where $(A_i+\beta I)$ is positive-definite since $A_i$ is positive-definite and $\beta\geq 0$. Plugging equation \eqref{l12p} into the local update rule, we get
\begin{align}\label{l13p}
    x^{(t)}_{i,\ell+1}&=x^{(t)}_{i,\ell}-\eta g^{(t)}_{i,\ell} \nonumber\\ 
                 &=\big[(A_i+\beta I)^{-1}-\eta I\big]g^{(t)}_{i,\ell}+(A_i+\beta I)^{-1}\big(A_ic_i+\beta \bar{x}_{t}\big).
\end{align}
Combining equations \eqref{l11p} and \eqref{l13p}, we get
\begin{align*}
    g^{(t)}_{i,\ell+1}&=(A_i+\beta I)x^{(t)}_{i,\ell+1}-A_ic_i-\beta \bar{x}_{t} \\
    &=[I-\eta(A_i+\beta I)]g^{(t)}_{i,\ell}.
\end{align*}
\end{proof}
For each client $i\in \mathcal{S}$, define the client distortion matrix
\begin{align}\label{Qi_prox}
    Q_i=\sum\limits_{\ell=0}^{H-1}[I-\eta (A_i+\beta I)]^{\ell},
\end{align}
and the client surrogate function
\begin{align}\label{SurrFi_prox}
    \tilde{f}_i(x)=\dfrac{1}{2}\|(Q_iA_i)^{1/2}(x-c_i)\|^2.
\end{align}
The global surrogate function $\tilde{f}(x)$ is hence defined as
\begin{align}\label{SurrF_prox}
   \tilde{f}(x)=\dfrac{1}{m}\sum\limits_{i \in \mathcal{S}} \tilde{f}_i(x)=\dfrac{1}{m}\sum\limits_{i \in \mathcal{S}}\dfrac{1}{2}\|(Q_iA_i)^{1/2}(x-c_i)\|^2.
\end{align}
We now provide the proof of Proposition \ref{prop1:FedProx}.
\begin{proof}
From Lemma \ref{lemma1_prox}, we have
\begin{align}\label{l21p}
    g^{(t)}_{i,\ell}&=[I-\eta(A_i+\beta I)]^\ell g^{(t)}_{i,0} \nonumber \\
               &=[I-\eta(A_i+\beta I)]^\ell\nabla f_i(\bar{x}_{t}).
\end{align}
Hence, we may write
\begin{align}\label{l22p}
    q_i^{(t)}&=\sum\limits_{\ell=0}^{H-1} g^{(t)}_{i,\ell}\nonumber\\
    &=\sum\limits_{\ell=0}^{H-1}[I-\eta(A_i+\beta I)]^\ell \nabla f_i(\bar{x}_{t})\nonumber\\
    &=Q_i\nabla f_i(\bar{x}_{t})\nonumber \\
    &=\nabla \tilde{f}_i(\bar{x}_{t}).
\end{align}
Moreover, by definition,
\begin{align}\label{l23p}
    q(t)&=\dfrac{1}{m}\sum\limits_{i \in \mathcal{S}} q_i^{(t)} \nonumber \\
    &=\dfrac{1}{m}\sum\limits_{i \in \mathcal{S}} \nabla \tilde{f}_i(\bar{x}_{t}) \nonumber \\
    &=\nabla \tilde{f}(\bar{x}_{t}),
\end{align}
where the second equality follows from equation \eqref{l22p}, and the last one follows from the  definition of the global surrogate function. As per equation \eqref{eqn:FedProx}, the global iterates in \texttt{FedProx} are updated as follows
\begin{align}\label{l24p}
    \bar{x}_{t+1}&=\bar{x}_{t}-\eta q^{(t)}\nonumber \\
    &=\bar{x}_{t}-\eta \nabla \tilde{f}(\bar{x}_{t}).
\end{align}
Equation \eqref{l24p} is equivalent to performing one step of GD on the surrogate function $\tilde{f}(x)$ with a starting point $\bar{x}_{t}$. It follows that $T$ communication rounds of \texttt{FedProx} are equivalent to performing $T$ steps of parallel GD with $m$ workers on the global surrogate function $\tilde{f}(x)$.
\end{proof}
In what follows, we analyze some properties of the client surrogate functions $\tilde{f}_i(x)$.
\begin{lemma}\label{lemma2_prox}
For $\eta<\dfrac{1}{L+\beta}$, the client surrogate function $\tilde{f}_i(x)$ shares the same unique minimizer as the true local function $f_i(x)$, that is,  $\tilde{x}_i^*=x_i^*=c_i$, $\forall i \in \mathcal{S}$, where $\tilde{x}_i^*$ is the minimizer of $\tilde{f}_i^*(x)$.
\end{lemma}
\begin{proof}
Let $(\lambda_i,v_i)$ be a tuple of eigenvalue $\lambda_i$, and corresponding eigenvector $v_i$, of $A_i$, where $A_i$ is symmetric positive-definite. Then, $(1-\eta (\lambda_i+\beta),v_i)$ is a tuple of eigenvalue and corresponding eigenvector of $D\triangleq [I-\eta (A_i+\beta)]$. For $0<\eta\leq \dfrac{1}{L+\beta}<\dfrac{1}{L_i+\beta}$, $\forall i \in  \mathcal{S}$, we have $0<1-\eta (\lambda_i+\beta)<1$, and hence,  $D$ is a symmetric positive-definite matrix. For any integer $\ell\geq1$, $([1-\eta (\lambda_i+\beta)]^\ell,v_i)$ is a tuple of eigenvalue and corresponding eigenvector of $D^\ell\triangleq [I-\eta (A_i+\beta)]^\ell$, and hence,  $D^\ell$ is also a symmetric positive-definite matrix since $0<\eta<\dfrac{1}{L_i+\beta}$. From equation \eqref{Qi_prox}, we have $Q_i=\sum\limits_{\ell=0}^{H-1} [I-\eta (A_i+\beta)]^\ell=\sum\limits_{\ell=0}^{H-1}D^\ell$. Hence, $(\sum\limits_{\ell=0}^{H-1}[1-\eta (\lambda_i+\beta)]^\ell,v_i)$ is a tuple of eigenvalue and corresponding eigenvector of $Q_i$ and, consequently, $Q_i$ is a symmetric positive-definite matrix. Based on the above  results, since both $A_i$ and $Q_i$ are symmetric positive-definite matrices that are simultaneously diagonalizable, it follows that they commute. Hence, for $0<\eta<\dfrac{1}{L+\beta}$, the product $Q_iA_i$ is a symmetric positive-definite matrix and the client surrogate function $\tilde{f}_i(x)$ admits a unique minimizer $\tilde{x}_i^*=x^*_i=c_i.$
\end{proof}

\subsection*{Proof of Proposition \ref{prop2:FedNova}: FedNova}
\label{app:proofprop2_nova}
The proof of Proposition \ref{prop2:FedNova} follows roughly the same steps as the proof of Proposition \ref{prop1:FedProx}. In particular, we consider a deterministic version of \texttt{FedNova} where client $i$ performs $\tau_i$ local iterations according to the update rule in equation \eqref{eqn:FedNova}. Let $g_{i,\ell}^{(t)}\triangleq\nabla f_i(x_{i,\ell}^{(t)})$, $q_i^{(t)}\triangleq\sum\limits_{\ell=0}^{\tau_i-1}\alpha_i g_{i,\ell}^{(t)}$, and $q(t)\triangleq\dfrac{1}{m}\sum\limits_{i \in \mathcal{S}} q_i^{(t)}$. Then, Lemma \ref{lemma1_nova} provides a recursive relationship in terms of $g^{(t)}_{i,\ell}$.
\begin{lemma}\label{lemma1_nova}
For every $i\in \mathcal{S}$, we have for all $\ell\in \{0, \cdots, \tau_i-1\}$: 
\begin{align}
    g^{(t)}_{i,\ell+1}=[I-\eta A_i]g^{(t)}_{i,\ell}.
\end{align}
\end{lemma}
For each client $i\in \mathcal{S}$, define the client distortion matrix
\begin{align}\label{Qi_nova}
    Q_i=\sum\limits_{\ell=0}^{\tau_i-1}[I-\eta  A_i]^{\ell}\alpha_i.
\end{align}
Let the client and global surrogate functions, namely $\tilde{f}_i(x)$ and $\tilde{f}(x)$, be defined as before. Then, the proof of Proposition \ref{prop2:FedNova} follows directly from Lemma \ref{lemma1_nova}. Finally, Lemma \ref{lemma2_nova} provides additional insight on the client surrogate functions $\tilde{f}_i(x)$. Once again, the proof is omitted as it closely follows the steps of the proof of Lemma \ref{lemma2_prox}.
\begin{lemma}\label{lemma2_nova}
For $\eta<\dfrac{1}{L}$, the client surrogate function $\tilde{f}_i(x)$ shares the same unique minimizer as the true local function $f_i(x)$, that is,  $\tilde{x}_i^*=x_i^*=c_i$, $\forall i \in \mathcal{S}$.
\end{lemma}
It should be noted that while $f_i(x)$ and $\tilde{f}_i(x)$ share the same minimizer, the minimizer of the surrogate global function may be far off from the true minimizer $x^*$ (see the numerical results in Section \ref{numb_Prox_Nova}).
\newpage
\section{Useful Results and Facts}
In this section, we will compile some results that will prove to be useful later in our analysis. We start by assembling some well-known facts about convex and smooth functions \cite{nesterov,bubeck}. 

\begin{itemize}
    \item (\textbf{Smoothness}): Suppose $f(x)$ is $L$-smooth. Then, by definition, the following inequalities hold for any two points $x,y\in\mathbb{R}^d$:
\begin{equation}
    \Vert \nabla f(x)- \nabla f(y) \Vert \leq L \Vert x-y \Vert, \, \textrm{and}
\label{eqn:smooth1}
\end{equation}
\begin{equation}
    f(y)-f(x) \leq \langle y-x, \nabla f(x) \rangle +\frac{L}{2}{\Vert y-x \Vert}^2.
\label{eqn:smooth2}
\end{equation}
\item (\textbf{Smoothness}): Suppose $f(x)$ is $L$-smooth, and $x^*\in\argmin_{x\in\mathbb{R}^{d}} f(x)$. Then, we can upper-bound the magnitude of the gradient at any given point $y\in\mathbb{R}^d$ in terms of the objective sub-optimality at $y$, as follows:
\begin{equation}
    {\Vert \nabla f(y) \Vert}^2 \leq 2L (f(y)-f(x^*)).
\label{eqn:grad_upp}
\end{equation}
\item (\textbf{Smoothness and Convexity}): Suppose $f(x)$ is $L$-smooth and convex. Then, the following holds for any two points $x,y \in \mathbb{R}^{d}$:
\begin{equation}
    \langle \nabla f(y) - \nabla f(x), y-x \rangle \geq \frac{1}{L} {\Vert \nabla f(y) - \nabla f(x)\Vert}^2.
\label{eqn:smooth_conv}
\end{equation}
\item (\textbf{Strong convexity}): Suppose $f(x)$ is $\mu$-strongly convex. Then, by definition, the following inequality holds for any two points $x,y\in\mathbb{R}^d$:
\begin{equation}
    f(y)-f(x) \geq \langle y-x, \nabla f(x) \rangle +\frac{\mu}{2}{\Vert y-x \Vert}^2.
\label{eqn:strconv1}
\end{equation}
Using the above inequality, one can easily conclude that
\begin{equation}
    \langle \nabla f(y) - \nabla f(x), y-x \rangle \geq \mu {\Vert y-x \Vert}^2.
\label{eqn:strconv2}
\end{equation}

 \item (\textbf{Strong convexity}): Suppose $f(x)$ is $\mu$-strongly convex, and $x^*=\argmin_{x\in\mathbb{R}^{d}} f(x)$. Then, we can lower-bound the magnitude of the gradient at any given point $y\in\mathbb{R}^d$ in terms of the objective sub-optimality at $y$, as follows:
\begin{equation}
    {\Vert \nabla f(y) \Vert}^2 \geq 2\mu (f(y)-f(x^*)).
\label{eqn:grad_low}
\end{equation}
\end{itemize}

In addition to the above results, we will have occasion to make use of the following facts.

\begin{itemize}
    \item Given any two vectors $x,y\in\mathbb{R}^d$, the following holds for any $\gamma >0$:
\begin{equation}
    {\Vert x+y \Vert}^2 \leq (1+\gamma){\Vert x \Vert}^2 + \left(1+\frac{1}{\gamma}\right){\Vert y \Vert}^2.
\label{eqn:rel_triangle}
\end{equation}
 \item Given $m$ vectors $x_1,\ldots,x_m\in\mathbb{R}^d$, the following is a simple application of Jensen's inequality:
 \begin{equation}
     \norm[\bigg]{ \sum\limits_{i=1}^{m} x_i}^2 \leq m \sum\limits_{i=1}^{m} {\Vert x_i \Vert}^2.
\label{eqn:Jensens}
 \end{equation}
\end{itemize}

\section{Analysis under Objective Heterogeneity: \texttt{FedLin} resolves the Speed-Accuracy Conflict} 
\label{app:proofthm_speed_acc}
In this section, we start by presenting a convergence analysis for \texttt{FedLin} that focuses solely on the aspect of heterogeneity in the clients' local objective functions. We do so to set up the basic proof structure that we will later build on for analyzing more involved settings. With this in mind, we will assume throughout this section that all clients perform the same number of local updates, i.e., $\tau_i=H,\forall i\in\mathcal{S}$. Additionally, we will assume that there is no gradient sparsification, i.e., $\delta_{c}=\delta_{s}=1$. Based on the second assumption, observe that $\rho_{i,t}=0,e_t=0,\forall i\in\mathcal{S},\forall t\in\{1,\ldots,T\}.$ Thus, the local update rule for each client in line 5 of \texttt{FedLin} simplifies to
\begin{equation}
    x^{(t)}_{i,\ell+1} = x^{(t)}_{i,\ell}-\eta_i(\nabla f_i(x^{(t)}_{i,\ell})-\nabla f_i(\bar{x}_{t})+\nabla f(\bar{x}_t)).
\end{equation}
Let us denote by $\kappa=L/\mu$ the condition number of an $L$-smooth and $\mu$-strongly convex function. 
\begin{theorem} \label{thm:speed_acc}  (\textbf{Heterogeneous Setting}) Suppose each $f_i(x)$ is $L$-smooth and $\mu$-strongly convex. Moreover, suppose $\tau_i=H,\forall i\in\mathcal{S}$, and $\delta_c=\delta_s=1$. Then, with $\eta_i=\eta=\frac{1}{6LH}, \forall i \in \mathcal{S}$, \texttt{FedLin} guarantees:
\begin{equation}
    f(\bar{x}_{T+1})-f(x^*) \leq {\left(1-\frac{1}{6\kappa}\right)}^T (f(\bar{x}_{1})-f(x^*)). 
\label{eqn:speed_acc}
\end{equation}
\end{theorem}

 When all clients optimize the same loss function, we have the following result. 
 
\begin{proposition} \label{prop:homogeneous}(\textbf{Homogeneous setting}) Suppose all client objective functions are identical, i.e., $f_i(x)=f(x), \forall x\in\mathbb{R}^{d}, \forall i\in\mathcal{S}$, and $f(x)$ is $L$-smooth and $\mu$-strongly convex.  Moreover, suppose $\tau_i=H,\forall i\in\mathcal{S}$, and $\delta_c=\delta_s=1$. Then, with $\eta_i=\eta=\frac{1}{L}, \forall i \in \mathcal{S}$, \texttt{FedLin} guarantees:
\begin{equation}
    f(\bar{x}_{T+1})-f(x^*) \leq {\left(1-\frac{1}{\kappa}\right)}^{TH} (f(\bar{x}_{1})-f(x^*)). 
\end{equation}
\end{proposition}

We first provide the proof of Proposition \ref{prop:homogeneous}.
\begin{proof} 
The proof follows from two simple observations. First, note that since $\nabla f_i(x)=\nabla f(x), \forall x\in\mathbb{R}^d$, the local update rule of each client $i\in\mathcal{S}$ reduces to $x^{(t)}_{i,\ell+1} =  x^{(t)}_{i,\ell}-\eta\nabla f(x^{(t)}_{i,\ell})$. Second, based on the steps of \texttt{FedLin}, observe that the local iterates of the clients remain synchronized within each communication round, i.e., for any $t\in\{1,\ldots,T\}$, and any two clients $i,j\in\mathcal{S}$, it holds that $x^{(t)}_{i,\ell}=x^{(t)}_{j,\ell}, \forall \ell\in\{0,\ldots,H-1\}$. Thus, \texttt{FedLin} boils down to $m$ parallel and identical implementations of gradient descent on the global loss function $f(x)$, with $TH$ iterations performed by each client over $T$ rounds. The claim of the proposition then follows from standard results of  centralized gradient descent applied to strongly convex and smooth deterministic objectives.
\end{proof}

\subsection{Proof of Theorem \ref{thm:speed_acc}} 
Let us fix a communication round $t\in\{1,\ldots,T\}$. Our goal will be to derive an upper-bound on the change in the value of the objective function $f(\cdot)$ over round $t$, i.e., we will be interested in bounding the quantity $f(\bar{x}_{t+1})-f(\bar{x}_t)$. To lighten the notation, we will drop the superscript $t$ on the local iterates $x^{(t)}_{i,\ell}$, and simply refer to them as $x_{i,\ell}$; it should be understood from context that all such iterates pertain to round $t$. Given that $\delta_c=\delta_s=1$, and $\eta_i=\eta, \forall i\in\mathcal{S}$, the local update rule in line 5 of \texttt{FedLin} takes the following simplified form:
    \begin{equation}
    x_{i,\ell+1} = x_{i,\ell}-\eta(\nabla f_i(x_{i,\ell})-\nabla f_i(\bar{x}_{t})+\nabla f(\bar{x}_t)).
\label{eqn:local_update}
\end{equation}

Based on the above rule, and smoothness of each $f_i(x)$, the following lemma provides an upper-bound on $f(\bar{x}_{t+1})-f(\bar{x}_t)$. 
\begin{lemma}
Suppose each $f_i(x)$ is $L$-smooth. Moreover, suppose $\tau_i=H,\forall i\in\mathcal{S}$, and $\delta_c=\delta_s=1$. Then,   \texttt{FedLin} guarantees:
\begin{equation}
\begin{aligned}
    f(\bar{x}_{t+1})-f(\bar{x}_t) &\leq -\eta H\left(1-\eta LH\right){\Vert \nabla f(\bar{x}_t) \Vert}^2 + \left(\frac{\eta L}{m} \sum\limits_{i=1}^{m}\sum\limits_{\ell=0}^{H-1}\Vert x_{i,\ell} - \bar{x}_t \Vert \right) \Vert \nabla f(\bar{x}_t) \Vert\\ 
    & \hspace{5mm} + \frac{\eta^2 L^3 H}{m} \sum\limits_{i=1}^{m}\sum\limits_{\ell=0}^{H-1}{\Vert x_{i,\ell} - \bar{x}_t \Vert}^2.
\end{aligned}
\label{eqn:f_bound_thm1}
\end{equation}
\label{lemma:f_boundthm1}
\end{lemma}
\begin{proof}
Based on \eqref{eqn:local_update} and the fact that $x_{i,0}=\bar{x}_t, \forall i\in\mathcal{S}$, we have:
\begin{equation}
    x_{i,H}=\bar{x}_t-\eta\sum\limits_{\ell=0}^{H-1}\nabla f_i(x_{i,\ell}) -\eta H (\nabla f(\bar{x}_t) - \nabla f_i(\bar{x}_t)), \forall i\in \mathcal{S}.
\end{equation}
Thus, 
\begin{equation}
\begin{aligned}
    \bar{x}_{t+1}=\frac{1}{m}\sum\limits_{i=1}^{m}x_{i,H}&=\bar{x}_t-\frac{\eta}{m}\sum\limits_{i=1}^{m} \sum\limits_{\ell=0}^{H-1} \nabla f_i(x_{i,\ell}) -\frac{\eta H}{m} \sum\limits_{i=1}^{m} \left(\nabla f(\bar{x}_t) - \nabla f_i(\bar{x}_t)\right)\\  
    &=\bar{x}_t-\frac{\eta}{m}\sum\limits_{i=1}^{m} \sum\limits_{\ell=0}^{H-1} \nabla f_i(x_{i,\ell}),
\label{eqn:xbart}
\end{aligned}
\end{equation}
where for the second equality, we used the fact that $\nabla f(y) = \frac{1}{m} \sum\limits_{i=1}^{m} \nabla f_i(y), \forall y\in\mathbb{R}^d$. 
Now since each $f_i(\cdot)$ is $L$-smooth, it is easy to verify that $f(\cdot)$ is also $L$-smooth. From \eqref{eqn:smooth2}, we then obtain:
\begin{equation}
    \begin{aligned}
    f(\bar{x}_{t+1})-f(\bar{x}_{t}) & \leq \langle \bar{x}_{t+1} - \bar{x}_{t}, \nabla f(\bar{x}_t) \rangle + \frac{L}{2} {\Vert \bar{x}_{t+1}-\bar{x}_t \Vert}^2\\
    & = -\frac{\eta}{m} \sum\limits_{i=1}^{m}\sum\limits_{\ell=0}^{H-1} \langle \nabla f_i(x_{i,\ell}), \nabla f(\bar{x}_{t}) \rangle + \frac{L}{2} \norm[\bigg]{\frac{\eta}{m}\sum\limits_{i=1}^{m}\sum\limits_{\ell=0}^{H-1} \nabla f_i(x_{i,\ell})}^2,
    \end{aligned}
\label{eqn:iterimbnd1}
\end{equation}
where in the last step we used \eqref{eqn:xbart}. To proceed, we will now separately bound each of the two terms that appear in \eqref{eqn:iterimbnd1}. For the first term, observe that
\begin{equation}
    \begin{aligned}
    -\frac{\eta}{m} \sum\limits_{i=1}^{m}\sum\limits_{\ell=0}^{H-1} \langle \nabla f_i(x_{i,\ell}), \nabla f(\bar{x}_{t}) \rangle &= - \eta \Big \langle \frac{1}{m}\sum\limits_{i=1}^{m}\sum\limits_{\ell=0}^{H-1} \left( \nabla f_i(x_{i,\ell})-\nabla f_i(\bar{x}_t) \right) + \frac{1}{m} \sum\limits_{i=1}^{m}\sum\limits_{\ell=0}^{H-1} \nabla f_i(\bar{x}_t), \nabla f(\bar{x}_{t}) \Big \rangle \\
    &= - \eta \Big \langle \frac{1}{m}\sum\limits_{i=1}^{m}\sum\limits_{\ell=0}^{H-1} \left( \nabla f_i(x_{i,\ell})-\nabla f_i(\bar{x}_t) \right) + H\nabla f(\bar{x}_t), \nabla f(\bar{x}_{t}) \Big \rangle \\
   &= -\frac{\eta}{m} \Big \langle \sum\limits_{i=1}^{m}\sum\limits_{\ell=0}^{H-1} \left( \nabla f_i(x_{i,\ell})-\nabla f_i(\bar{x}_t) \right), \nabla f(\bar{x}_{t}) \Big \rangle  - \eta H{\Vert \nabla f(\bar{x}_t) \Vert}^2 \\
  & \overset{(a)} \leq \frac{\eta}{m} \left(\norm[\bigg]{\sum\limits_{i=1}^{m}\sum\limits_{\ell=0}^{H-1} \left( \nabla f_i(x_{i,\ell})-\nabla f_i(\bar{x}_t) \right)}\right) \Vert \nabla f(\bar{x}_t) \Vert - \eta H{\Vert \nabla f(\bar{x}_t) \Vert}^2\\
 & \overset{(b)} \leq \frac{\eta}{m} \left( {\sum\limits_{i=1}^{m}\sum\limits_{\ell=0}^{H-1} \Vert \nabla f_i(x_{i,\ell})-\nabla f_i(\bar{x}_t)\Vert}\right) \Vert \nabla f(\bar{x}_t) \Vert - \eta H{\Vert \nabla f(\bar{x}_t) \Vert}^2\\
& \overset{(c)} \leq \frac{\eta L}{m} \left( {\sum\limits_{i=1}^{m}\sum\limits_{\ell=0}^{H-1} \Vert x_{i,\ell}-\bar{x}_t\Vert}\right) \Vert \nabla f(\bar{x}_t) \Vert - \eta H{\Vert \nabla f(\bar{x}_t) \Vert}^2.
    \end{aligned}
\label{eqn:T1_thm1}
\end{equation}
In the above steps, (a) follows from the Cauchy-Schwartz inequality, (b) follows from the triangle inequality, and (c) follows from the fact that each $f_i(\cdot)$ is $L$-smooth (see \eqref{eqn:smooth1}). Next, we bound the second term in \eqref{eqn:iterimbnd1} as follows.
\begin{equation}
    \begin{aligned}
 \frac{L}{2} \norm[\bigg]{\frac{\eta}{m}\sum\limits_{i=1}^{m}\sum\limits_{\ell=0}^{H-1} \nabla f_i(x_{i,\ell})}^2 &= \frac{\eta^2 L}{2}\norm[\bigg]{\frac{1}{m}\sum\limits_{i=1}^{m}\sum\limits_{\ell=0}^{H-1} \left( \nabla f_i(x_{i,\ell})-\nabla f_i(\bar{x}_t) \right) + \frac{1}{m} \sum\limits_{i=1}^{m}\sum\limits_{\ell=0}^{H-1} \nabla f_i(\bar{x}_t)}^2\\
 &=\frac{\eta^2 L}{2}\norm[\bigg]{\frac{1}{m}\sum\limits_{i=1}^{m}\sum\limits_{\ell=0}^{H-1} \left( \nabla f_i(x_{i,\ell})-\nabla f_i(\bar{x}_t) \right) + H \nabla f(\bar{x}_t)}^2\\
& \overset{(a)} \leq \eta^2 L \norm[\bigg]{\frac{1}{m}\sum\limits_{i=1}^{m}\sum\limits_{\ell=0}^{H-1} \left( \nabla f_i(x_{i,\ell})-\nabla f_i(\bar{x}_t) \right)}^2 + \eta^2 H^2 L {\Vert \nabla f(\bar{x}_t) \Vert}^2 \\
& \overset{(b)} \leq \frac{\eta^2 L}{m}  \sum\limits_{i=1}^{m} \norm[\bigg]{ \sum\limits_{\ell=0}^{H-1} \left( \nabla f_i(x_{i,\ell})-\nabla f_i(\bar{x}_t) \right)}^2 + \eta^2 H^2 L {\Vert \nabla f(\bar{x}_t) \Vert}^2 \\
& \overset{(c)} \leq \frac{\eta^2 L H}{m}  \sum\limits_{i=1}^{m}  \sum\limits_{\ell=0}^{H-1} {\Vert  \nabla f_i(x_{i,\ell})-\nabla f_i(\bar{x}_t) \Vert }^2 + \eta^2 H^2 L {\Vert \nabla f(\bar{x}_t) \Vert}^2 \\
& \overset{(d)} \leq \frac{\eta^2 L^3 H}{m}  \sum\limits_{i=1}^{m}  \sum\limits_{\ell=0}^{H-1} {\Vert  x_{i,\ell}-\bar{x}_t \Vert }^2 + \eta^2 H^2 L {\Vert \nabla f(\bar{x}_t) \Vert}^2.
    \end{aligned}
\label{eqn:T2boundthm1}
\end{equation}
In the above steps, (a) follows from \eqref{eqn:rel_triangle} with $\gamma=1$, (b) and (c) both follow from \eqref{eqn:Jensens}, and (d) is a consequence of the $L$-smoothness of $f_i(\cdot)$. Combining the bounds in equations \eqref{eqn:T1_thm1} and \eqref{eqn:T2boundthm1} immediately leads to the claim of the lemma. 
\end{proof}

To simplify the bound in \eqref{eqn:f_bound_thm1}, it is apparent that we need to estimate how much the client iterates $x_{i,\ell}$ drift off from their value $\bar{x}_t$ at the beginning of the communication round, i.e., we want to derive a bound on the quantity ${\Vert x_{i,\ell} - \bar{x}_t \Vert}$. To that end, we will make use of the following lemma. 
\begin{lemma}
Suppose $f(x)$ is $L$-smooth and convex. Then, for any $\eta\in(0,1)$ satisfying $\eta \leq \frac{1}{L}$, and any two points $x,y\in\mathbb{R}^{d}$, we have
\begin{equation}
    \Vert y-x-\eta(\nabla f(y) - \nabla f(x)) \Vert \leq \Vert y-x \Vert.
\label{eqn:two_pt_conv}
\end{equation}
If $f(x)$ is $\mu$-strongly convex, then the above inequality is strict, i.e., $\exists \lambda \in (0,1)$, such that
\begin{equation}
    \Vert y-x-\eta(\nabla f(y) - \nabla f(x)) \Vert \leq \lambda\Vert y-x \Vert. 
\label{eqn:two_pt_strconv}
\end{equation}
\label{lemma:two_pt_bnd}
\end{lemma}
\begin{proof}
Given any two points $x,y\in\mathbb{R}^{d}$, we have
\begin{equation}
    \begin{aligned}
{\Vert y-x-\eta(\nabla f(y) - \nabla f(x)) \Vert}^2 &= {\Vert y-x \Vert}^2 - 2\eta \langle y-x, \nabla f(y) - \nabla f(x) \rangle + \eta^2 {\Vert \nabla f(y) - \nabla f(x) \Vert}^2\\
& \leq {\Vert y-x \Vert}^2 - \eta (2-\eta L) \langle y-x, \nabla f(y) - \nabla f(x) \rangle\\
& \leq {\Vert y-x \Vert}^2,
    \end{aligned}
\end{equation}
where the first inequality follows from \eqref{eqn:smooth_conv}, and the second from the fact that $\eta L \leq 1$, and  $\langle y-x, \nabla f(y) - \nabla f(x) \rangle \geq 0$; the latter is a consequence of the convexity of $f(\cdot)$. This establishes \eqref{eqn:two_pt_conv} for the convex, smooth case. When $f(\cdot)$ is $\mu$-strongly convex, we can further use \eqref{eqn:strconv2} to obtain \begin{equation}
{\Vert y-x-\eta(\nabla f(y) - \nabla f(x)) \Vert}^2 \leq \left(1-\eta \mu (2-\eta L) \right)  {\Vert y-x \Vert}^2.
\end{equation}
Noting that $\eta L \leq 1$ establishes \eqref{eqn:two_pt_strconv} with $\lambda=\sqrt{1-\eta \mu}$.
\end{proof}

With the above lemma in hand, we will now show that the drift $\Vert x_{i,\ell} - \bar{x}_t \Vert$ can be bounded in terms of $\Vert \nabla f(\bar{x}_t) \Vert$ - a measure of the sub-optimality at the beginning of the communication round.  
\begin{lemma}
Suppose each $f_i(x)$ is $L$-smooth and $\mu$-strongly convex. Moreover, suppose $\tau_i=H,\forall i\in\mathcal{S}$, $\delta_c=\delta_s=1$, and $\eta \leq \frac{1}{L}$. Then,  \texttt{FedLin} guarantees the following bound for each $i\in\mathcal{S}$, and $\forall \ell\in\{0,\ldots,H-1\}$: 
\begin{equation}
    \Vert{x_{i,\ell} - \bar{x}_t \Vert} \leq \eta H \Vert \nabla f(\bar{x}_t) \Vert.
\end{equation}
\label{lemma:drift_conv}
\end{lemma}
\begin{proof}
Fix any client $i\in\mathcal{S}$. From \eqref{eqn:local_update}, we have
\begin{equation}
    \begin{aligned}
\Vert x_{i,\ell+1} - \bar{x}_t \Vert & = \Vert x_{i,\ell} - \bar{x}_t - \eta \left(\nabla f_i(x_{i,\ell})-\nabla f_i(\bar{x}_t)\right) -\eta \nabla f(\bar{x}_t) \Vert\\
& \leq \Vert x_{i,\ell} - \bar{x}_t - \eta \left(\nabla f_i(x_{i,\ell})-\nabla f_i(\bar{x}_t)\right) \Vert + \eta \Vert \nabla f(\bar{x}_t) \Vert \\
& \leq \lambda \Vert x_{i,\ell} - \bar{x}_t \Vert + \eta \Vert \nabla f(\bar{x}_t) \Vert,
    \end{aligned}
\label{eqn:driftbnd1}
\end{equation}
where $\lambda=\sqrt{1-\eta \mu} < 1$. The last inequality follows from Lemma \ref{lemma:two_pt_bnd} with $x=x_{i,\ell}$ and $y=\bar{x}_t$. Rolling out the final inequality in \eqref{eqn:driftbnd1} from $\ell=0$ yields:
\begin{equation}
  \begin{aligned}
\Vert x_{i,\ell} - \bar{x}_t \Vert & \leq \lambda^{\ell} \Vert x_{i,0} - \bar{x}_t \Vert + \left(\sum\limits_{j=0}^{\ell-1}\lambda^j \right) \eta \Vert \nabla f(\bar{x}_t) \Vert \\
& \overset{(a)}= \left(\sum\limits_{j=0}^{\ell-1}\lambda^j \right) \eta \Vert \nabla f(\bar{x}_t) \Vert \\
& \overset{(b)}\leq \eta \ell \Vert \nabla f(\bar{x}_t) \Vert\\
& \overset{(c)}\leq \eta H \Vert \nabla f(\bar{x}_t) \Vert.
    \end{aligned}  
\end{equation}
For (a), we used the fact that $x_{i,0}=\bar{x}_t, \forall i\in\mathcal{S}$; for (b), we used the fact that $\lambda <1$\footnote{Note that actually, in step (b), we used the looser bound for the convex setting, namely $\lambda \leq 1$. It seems unlikely that the tighter bound $\lambda <1$ will buy us anything other than a  potential improvement upon  the constant $\frac{1}{6}$ appearing in the exponent in equation \eqref{eqn:speed_acc}.}; and for (c), we note that $\ell\leq H-1$. 
\end{proof}

We are now equipped with all the ingredients needed to complete the proof of Theorem \ref{thm:speed_acc}. 

\textbf{Completing the proof of Theorem \ref{thm:speed_acc}}: Combining the bounds in Lemma's \ref{lemma:f_boundthm1} and \ref{lemma:drift_conv}, we obtain
\begin{equation}
\begin{aligned}
    f(\bar{x}_{t+1})-f(\bar{x}_t) &\leq -\eta H\left(1-\eta LH\right){\Vert \nabla f(\bar{x}_t) \Vert}^2 + \left(\frac{\eta L}{m} \sum\limits_{i=1}^{m}\sum\limits_{\ell=0}^{H-1} \eta H \Vert \nabla f(\bar{x}_t) \Vert \right) \Vert \nabla f(\bar{x}_t) \Vert\\ 
    & \hspace{5mm} + \frac{\eta^2 L^3 H}{m} \sum\limits_{i=1}^{m}\sum\limits_{\ell=0}^{H-1}{\left(  \eta H \Vert \nabla f(\bar{x}_t) \Vert \right)}^2\\
& = -\eta H\left(1-\eta LH\right){\Vert \nabla f(\bar{x}_t) \Vert}^2 + \eta^2 L H^2 {\Vert \nabla f(\bar{x}_t) \Vert}^2 + \eta^4 L^3 H^4 {\Vert \nabla f(\bar{x}_t) \Vert}^2 \\
& \leq -\eta H\left(1-3\eta LH\right){\Vert \nabla f(\bar{x}_t) \Vert}^2,
\end{aligned}
\label{eqn:thm1_final}
\end{equation}
where in the last step, we used the fact that the step-size $\eta$ satisfies $\eta L H \leq 1$. From  \eqref{eqn:thm1_final} and \eqref{eqn:grad_low}, we obtain
\begin{equation}
    f(\bar{x}_{t+1})-f(x^*) \leq \left(1-2\eta \mu H\left(1-3\eta L H\right) \right) \left( f(\bar{x}_t) - f(x^*) \right).
\end{equation}
With $\eta=\frac{1}{6LH}$, the above inequality takes the form
\begin{equation}
    f(\bar{x}_{t+1})-f(x^*) \leq \left(1-\frac{1}{6\kappa} \right) \left( f(\bar{x}_t) - f(x^*) \right), \textrm{where} \hspace{1mm} \kappa=\frac{L}{\mu}.
\end{equation}
Using the above inequality recursively leads to the claim of the theorem. 
\newpage
\section{Proof of Theorem \ref{thm:lowerbnd}: Lower bound for \texttt{FedLin}}
\label{app:lowerbnd}
To prove Theorem \ref{thm:lowerbnd},   we will construct an example involving two clients. Let us start by defining the objective functions of the clients as follows.

\begin{equation}
    f_1(x)=\frac{1}{2}x'\underbrace{\begin{bmatrix}1&0\\0&1\end{bmatrix}}_{\mathbf{A}_1}x + b'x; \hspace{3mm} f_2(x)=\frac{1}{2}x'\underbrace{\begin{bmatrix}L&0\\0&1\end{bmatrix}}_{\mathbf{A}_2}x - b'x,
\nonumber
\end{equation}
where $b\in\mathbb{R}^2$ is any arbitrary vector, and $L \geq 14$. Here, we have used the notation $x'$ to indicate the transpose of a column vector $x$. From inspection, it is clear that each of the client objective functions is $1$-strongly convex. Also, $f_1(x)$ and $f_2(x)$ are both $L$-smooth.\footnote{Strictly speaking, $f_1(x)$ is $1$-smooth, but since $L > 1$, it is also $L$-smooth.} The global loss function is then given by
   $$ f(x)=\frac{1}{2} (f_1(x)+f_2(x)) = \frac{1}{2}x'\begin{bmatrix}\frac{L+1}{2}&0\\0&1\end{bmatrix}x. $$

It is easy to see that the minimum of $f(x)$ is $x^*={\begin{bmatrix} 0&0\end{bmatrix}}'$. Let us fix a communication round $t\in\{1,\ldots,T\}$. Now given that $\delta_c=\delta_s=1$, and $\eta_i=\eta, i\in\{1,2\}$, the local update rule for each of the clients takes the following form:
\begin{equation}
    x_{i,\ell+1} = x_{i,\ell}-\eta(\nabla f_i(x_{i,\ell})-\nabla f_i(\bar{x}_{t})+\nabla f(\bar{x}_t)).
\nonumber
\end{equation}

Simple calculations then lead to the following recursions:
\begin{equation}
    \begin{aligned}
        x_{1,\ell+1}&=\left(\mathbf{I}-\eta \mathbf{A}_1\right)x_{1,\ell}-\eta \begin{bmatrix} \frac{L-1}{2} & 0 \\
        0 & 0 \end{bmatrix} \bar{x}_t,\\
        x_{2,\ell+1}&=\left(\mathbf{I}-\eta \mathbf{A}_2\right)x_{2,\ell}-\eta \begin{bmatrix} \frac{1-L}{2} & 0 \\
        0 & 0 \end{bmatrix} \bar{x}_t.\\
    \end{aligned}
\nonumber
\end{equation}
Given the diagonal structure of $\mathbf{A}_1$ and $\mathbf{A}_2$, we can easily roll out the above recursions. 
Accordingly, for client 1, we obtain
\begin{equation}
    \begin{aligned}
        x_{1,H} &= {\left(\mathbf{I}-\eta \mathbf{A}_1\right)}^H \bar{x}_t - \eta \left(\sum_{j=0}^{H-1} {\left(\mathbf{I}-\eta \mathbf{A}_1\right)}^j \right) \begin{bmatrix} \frac{L-1}{2} & 0 \\
        0 & 0 \end{bmatrix} \bar{x}_t\\
        &=\begin{bmatrix} {(1-\eta)}^H \left(\frac{L+1}{2}\right)- \frac{L-1}{2} & 0\\ \\ 
        0 & {(1-\eta)}^H \end{bmatrix}\bar{x}_t.
    \end{aligned}
\nonumber
\end{equation}
Similarly, for client 2, we have
\begin{equation}
    \begin{aligned}
        x_{2,H} &= {\left(\mathbf{I}-\eta \mathbf{A}_2\right)}^H \bar{x}_t - \eta \left(\sum_{j=0}^{H-1} {\left(\mathbf{I}-\eta \mathbf{A}_2\right)}^j \right) \begin{bmatrix} \frac{1-L}{2} & 0 \\
        0 & 0 \end{bmatrix} \bar{x}_t\\
        &=\begin{bmatrix} {(1-\eta L)}^H \left(\frac{L+1}{2L}\right)- \frac{1-L}{2L} & 0\\ \\
        0 & {(1-\eta)}^H \end{bmatrix} \bar{x}_t.
    \end{aligned}
\nonumber
\end{equation}
Thus,
\begin{equation}
    \bar{x}_{t+1}=\frac{1}{2} (x_{1,H}+x_{2,H})= \underbrace{\begin{bmatrix}\left( {(1-\eta)}^H + \frac{{(1-\eta L)}^H}{L}\right) \left(\frac{L+1}{4}\right)- \frac{{(L-1)}^2}{4L} & 0\\ \\
        0 & {(1-\eta)}^H \end{bmatrix}}_{\mathbf{M}} \bar{x}_t.
\label{eqn:recursion}
\end{equation}
In the rest of the proof, we will argue that for $\bar{x}_t$ to converge to $x^*$ based on the above recursion, $\eta$ must be chosen inversely proportional to $H$; the lower bound will then naturally follow. Let us start by noting that $\bar{x}_{t+1}=\mathbf{M}\bar{x}_{t}$ can be viewed as a discrete-time linear time-invariant (LTI) system where $\mathbf{M}$ is the state transition matrix. Since $x^*={\begin{bmatrix} 0&0\end{bmatrix}}'$, guaranteeing $\bar{x}_t$ converges to $x^*$ regardless of the initial condition $\bar{x}_1$ is equivalent to arguing that $\mathbf{M}$ is a Schur stable matrix, i.e., all the eigenvalues of $\mathbf{M}$ lie strictly inside the unit circle. It is easy to see that the eigenvalues $\lambda_1(\eta,H,L)$ and $\lambda_2(\eta,H)$ of $\mathbf{M}$ are
$$ \lambda_1(\eta,H,L) = \left( {(1-\eta)}^H + \frac{{(1-\eta L)}^H}{L}\right) \left(\frac{L+1}{4}\right)- \frac{{(L-1)}^2}{4L}; \hspace{2mm} \lambda_2(\eta,H) =  {(1-\eta)}^H.$$
In order for $\mathbf{M}$ to be Schur stable, $\lambda_1(\eta,H,L) > -1$ is a necessary condition. We will now show that to satisfy this necessary condition, $\eta$ must scale inversely with $H$. Observe that
\begin{equation}
    \begin{aligned}
      \lambda_1(\eta,H,L) > -1 & \implies \left( {(1-\eta)}^H + \frac{{(1-\eta L)}^H}{L}\right) \left(\frac{L+1}{4}\right)- \frac{{(L-1)}^2}{4L} > -1\\
      & \implies {(1-\eta)}^H \left(\frac{{(L+1)}^2}{4L}\right)- \frac{{(L-1)}^2}{4L} > -1 \\
      & \implies {(1-\eta)}^H > \frac{L^2-6L+1}{{(L+1)}^2}\\
      &  \implies {(1-\eta)}^H > \frac{1}{2}\\
     & \implies \frac{1}{1+\eta H} > \frac{1}{2}\\
     & \implies \eta < \frac{1}{H}. 
    \end{aligned}
\end{equation}
For the second implication, we used the fact that $L >1$. For the fourth implication, we note that 
$$g(z)=\frac{z^2-6z+1}{{(z+1)}^2}$$
is a monotonically increasing function of its argument for all $z>1$. The claim then follows by noting that $g(14) > \frac{1}{2}$, and $L \geq 14$. For the second-last implication, we used the fact that for any $z\in(0,1)$, and any positive integer $r$, the following is true
$${(1-z)}^r \leq \frac{1}{1+rz}.$$ 
We conclude that $\eta < \frac{1}{H}$ is a necessary condition for $\mathbf{M}$ to be Schur stable.\footnote{Notice that when $H=1$, this necessary condition translates to the trivial condition $\eta <1$.} Now let us look at the implication of this necessary condition on the eigenvalue $\lambda_2(\eta,H)={(1-\eta)}^H$. Since $\lambda_2(\eta,H)$ is  monotonically decreasing in $\eta$, the following holds for any $\eta$ that stabilizes $\mathbf{M}$:\footnote{That a stabilizing $\eta$ exists follows from Theorem \ref{thm:speed_acc}. In particular, $\eta=\frac{1}{6LH}$ will render $\mathbf{M}$ Schur stable. However, our main goal is to show that even if there do exist values of $\eta$ larger than $\frac{1}{6LH}$ that guarantee convergence of $\bar{x}_t$ to $x^*$, such step-size values must obey $\eta < \frac{1}{H}$, which in turn leads to the $H$-independent lower bound in equation \eqref{eqn:lbnd}.}
$$ \lambda_2(\eta,H) > \lambda_2(\eta_c,H) = {\left(1-\frac{1}{H}\right)}^H \geq {\left(1-\frac{1}{2}\right)}^2 > \exp(-2), \forall H \geq 2,$$
where $\eta_c=\frac{1}{H}$. Here, we have used the fact that ${\left(1-\frac{1}{z}\right)}^z$ is a monotonically increasing function in $z$ for $z>1$, and that $H\geq 2$. Now suppose \texttt{FedLin} is initialized with $\bar{x}_1 = {\begin{bmatrix} 0 & \beta \end{bmatrix}}'$, where $\beta >0$. Based on \eqref{eqn:recursion}, for any $T\geq1$, we then have
$$ \bar{x}_{T+1}= \begin{bmatrix} 0\\ \lambda^T_2 \beta\end{bmatrix} \geq \exp(-2T) \begin{bmatrix} 0\\ \beta\end{bmatrix} = \exp(-2T)  \bar{x}_1 \implies {\Vert \bar{x}_{T+1}-x^* \Vert}^2 \geq \exp{(-4T)} {\Vert \bar{x}_{1}-x^* \Vert}^2.$$ 
Moreover, substituting the value of $\bar{x}_{T+1}$ above in the expression for $f(x)$, we obtain
$$
f(\bar{x}_{T+1})=\frac{1}{2}{\lambda}^{2T}_2 \beta^2 = {\lambda}^{2T}_2 f(\bar{x}_1) \geq \exp(-4T) f(\bar{x}_1) \implies f(\bar{x}_{T+1}) - f(x^*) \geq \exp(-4T) (f(\bar{x}_{1}) - f(x^*)),
$$
since $f(x^*)=0$. We have thus established the desired lower bounds. Finally, note that all our arguments above hold for any $L \geq 14$, any $H\geq 2$, and any $T \geq 1$. This concludes the proof. 
\newline
\newline
\begin{remark}
Note that the quantity $b$ in the objective functions of the clients does not feature anywhere in the above analysis. As such, one possible choice of $b$ is simply $b={\begin{bmatrix} 0&0\end{bmatrix}}'$. For this choice of $b$, it is clear from inspection that both $f_1(x)$ and $f_2(x)$ have the same minimum, namely $x^*={\begin{bmatrix} 0&0\end{bmatrix}}$. Our analysis thus reveals that even when the client local objectives share the same minimum, the lower bound in equation  \eqref{eqn:lbnd} continues to hold. 
\label{rem:lowerbnd}
\end{remark}

\begin{remark}
Let us look more closely at the role played by $L$ in the above example. From inspection, larger the value of $L$, the less smooth the overall function $f(x)$, and more the objective heterogeneity. Specifically,  increasing $L$ increases the heterogeneity between the client local objectives by increasing $\Vert \nabla f_1(x) - \nabla f_2 (x) \Vert$. In terms of the impact of $L$ on the dynamics of $\bar{x}_t$, we note that $T_2$ in the expression for $\lambda_1(\eta,H,L)$ below becomes larger as we increase $L$.
$$ \lambda_1(\eta,H,L) = \underbrace{\left( {(1-\eta)}^H + \frac{{(1-\eta L)}^H}{L}\right) \left(\frac{L+1}{4}\right)}_{T_1}- \underbrace{\frac{{(L-1)}^2}{4L}}_{T_2}$$
To keep $\lambda_1(\eta,H,L) > - 1$, and thereby ensure stability of the recursion \eqref{eqn:recursion}, we then need the coefficient of $(L+1)/4$ in $T_1$ above to be adequately large, despite potentially large values of $H$: this is precisely what necessitates $\eta$ to scale inversely with $H$. To summarize, in the above example, large $L$ leads to less smoothness, more objective heterogeneity, and small step-size. 
\end{remark}
\newpage
\section{Proofs pertaining to Systems Heterogeneity in Section \ref{sec:speed_acc}}
\label{app:proofs_CT}
As in Appendix \ref{app:proofthm_speed_acc}, we will focus on a fixed communication round $t\in\{1,\ldots,T\}$; all our subsequent arguments will apply identically to  each such round. Given that $\delta_c=\delta_s=1$, the local update rule that will be of relevance to us throughout this section is 
    \begin{equation}
    x_{i,\ell+1} = x_{i,\ell}-\eta_i(\nabla f_i(x_{i,\ell})-\nabla f_i(\bar{x}_{t})+\nabla f(\bar{x}_t)),
\label{eqn:local_update_ct}
\end{equation}
where we have dropped the superscript of $t$ on the local iterates $x_{i,\ell}$ to simplify notation. Based on the discussion in Section \ref{sec:algo}, recall that the client step-sizes are chosen as follows:
$$
\eta_i=\frac{\bar{\eta}}{\tau_i}, \forall i\in\mathcal{S},
$$
where $\bar{\eta}\in(0,1)$ is a flexible design parameter that we will specify based on the context.

\subsection{Proof of Theorem \ref{thm:strconv_dh}: Analysis for Strongly Convex loss functions}
The proof of Theorem \ref{thm:strconv_dh} largely mirrors that of Theorem \ref{thm:speed_acc}. We start with the following analogue of Lemma \ref{lemma:f_boundthm1}. \begin{lemma}
Suppose each $f_i(x)$ is $L$-smooth. Moreover, suppose $\tau_i \geq 1, \forall i\in\mathcal{S}$,  $\delta_c=\delta_s=1$, and $\eta_i=\frac{\bar{\eta}}{\tau_i}, \forall i \in \mathcal{S}$, where $\bar{\eta}\in (0,1)$. Then, \texttt{FedLin} guarantees:
\begin{equation}
\begin{aligned}
    f(\bar{x}_{t+1})-f(\bar{x}_t) &\leq -\bar{\eta} \left(1-\bar{\eta} L\right){\Vert \nabla f(\bar{x}_t) \Vert}^2 + \left(\frac{L}{m} \sum\limits_{i=1}^{m} \eta_i \sum\limits_{\ell=0}^{\tau_i-1}\Vert x_{i,\ell} - \bar{x}_t \Vert \right) \Vert \nabla f(\bar{x}_t) \Vert\\ 
    & \hspace{5mm} + \frac{\bar{\eta} L^3 }{m} \sum\limits_{i=1}^{m} \eta_i  \sum\limits_{\ell=0}^{\tau_i-1}{\Vert x_{i,\ell} - \bar{x}_t \Vert}^2.
\end{aligned}
\label{eqn:f_bound_ct}
\end{equation}
\label{lemma:f_bound_ct}
\end{lemma}
\begin{proof}
Using \eqref{eqn:local_update_ct} and the fact that $x_{i,0}=\bar{x}_t, \forall i\in\mathcal{S}$, we have:
\begin{equation}
\begin{aligned}
    x_{i,\tau_i}&=\bar{x}_t-\eta_i\sum\limits_{\ell=0}^{\tau_i-1}\nabla f_i(x_{i,\ell}) -\eta_i \tau_i (\nabla f(\bar{x}_t) - \nabla f_i(\bar{x}_t))\\
    &=\bar{x}_t-\eta_i\sum\limits_{\ell=0}^{\tau_i-1}\nabla f_i(x_{i,\ell}) - \bar{\eta} (\nabla f(\bar{x}_t) - \nabla f_i(\bar{x}_t)), \forall i\in\mathcal{S},
\end{aligned}
\end{equation}
where we used $\eta_i \tau_i =\bar{\eta}$ in the second step. Averaging the above iterates across clients, we obtain: 
\begin{equation}
\begin{aligned}
    \bar{x}_{t+1}=\frac{1}{m}\sum\limits_{i=1}^{m}x_{i,\tau_i}&=\bar{x}_t-\frac{1}{m} \sum\limits_{i=1}^{m} \eta_i \sum\limits_{\ell=0}^{\tau_i-1} \nabla f_i(x_{i,\ell}) -\frac{\bar{\eta}}{m} \sum\limits_{i=1}^{m} \left(\nabla f(\bar{x}_t) - \nabla f_i(\bar{x}_t)\right)\\  
    &=\bar{x}_t-\frac{1}{m}  \sum\limits_{i=1}^{m} \eta_i \sum\limits_{\ell=0}^{\tau_i-1} \nabla f_i(x_{i,\ell}),
\label{eqn:xbart_ct}
\end{aligned}
\end{equation}
where for the second equality, we used the fact that $\nabla f(y) = \frac{1}{m} \sum\limits_{i=1}^{m} \nabla f_i(y), \forall y\in\mathbb{R}^d$. 
Based on the $L$-smoothness of $f(\cdot)$ and \eqref{eqn:xbart_ct}, we then have
\begin{equation}
    \begin{aligned}
    f(\bar{x}_{t+1})-f(\bar{x}_{t}) & \leq \langle \bar{x}_{t+1} - \bar{x}_{t}, \nabla f(\bar{x}_t) \rangle + \frac{L}{2} {\Vert \bar{x}_{t+1}-\bar{x}_t \Vert}^2\\
    & = - \Big \langle \frac{1}{m} \sum\limits_{i=1}^{m} \eta_i \sum\limits_{\ell=0}^{\tau_i-1}  \nabla f_i(x_{i,\ell}), \nabla f(\bar{x}_{t}) \Big \rangle + \frac{L}{2} \norm[\bigg]{\frac{1}{m}\sum\limits_{i=1}^{m} \eta_i \sum\limits_{\ell=0}^{\tau_i-1} \nabla f_i(x_{i,\ell})}^2.
    \end{aligned}
\label{eqn:iterimbnd1_ct}
\end{equation}
Just as in the proof of Theorem \ref{thm:speed_acc}, we now proceed to separately bound each of the two terms above, starting with the first term, as follows. 
\begin{equation}
    \begin{aligned}
    - \Big \langle \frac{1}{m} \sum\limits_{i=1}^{m} \eta_i \sum\limits_{\ell=0}^{\tau_i-1}  \nabla f_i(x_{i,\ell}), \nabla f(\bar{x}_{t}) \Big \rangle &= - \Big \langle \frac{1}{m}\sum\limits_{i=1}^{m} \eta_i \sum\limits_{\ell=0}^{\tau_i-1} \left( \nabla f_i(x_{i,\ell})-\nabla f_i(\bar{x}_t) \right) + \frac{1}{m} \sum\limits_{i=1}^{m} \eta_i \sum\limits_{\ell=0}^{\tau_i-1} \nabla f_i(\bar{x}_t), \nabla f(\bar{x}_{t}) \Big \rangle \\
    &= - \Big \langle \frac{1}{m}\sum\limits_{i=1}^{m} \eta_i \sum\limits_{\ell=0}^{\tau_i-1} \left( \nabla f_i(x_{i,\ell})-\nabla f_i(\bar{x}_t) \right) + \frac{1}{m} \sum\limits_{i=1}^{m} \underbrace{\eta_i \tau_i}_{\bar{\eta}} \nabla f_i(\bar{x}_t), \nabla f(\bar{x}_{t}) \Big \rangle \\
    &= - \Big \langle \frac{1}{m}\sum\limits_{i=1}^{m} \eta_i \sum\limits_{\ell=0}^{\tau_i-1} \left( \nabla f_i(x_{i,\ell})-\nabla f_i(\bar{x}_t) \right), \nabla f(\bar{x}_{t}) \Big \rangle  - \bar{\eta} {\Vert \nabla f(\bar{x}_t) \Vert}^2 \\ 
  & \overset{(a)} \leq \frac{1}{m} \left(\norm[\bigg]{\sum\limits_{i=1}^{m} \eta_i \sum\limits_{\ell=0}^{\tau_i-1} \left( \nabla f_i(x_{i,\ell})-\nabla f_i(\bar{x}_t) \right)}\right) \Vert \nabla f(\bar{x}_t) \Vert - \bar{\eta} {\Vert \nabla f(\bar{x}_t) \Vert}^2\\
 & \overset{(b)} \leq \frac{1}{m} \left( {\sum\limits_{i=1}^{m} \eta_i \sum\limits_{\ell=0}^{\tau_i-1} \Vert \nabla f_i(x_{i,\ell})-\nabla f_i(\bar{x}_t)\Vert}\right) \Vert \nabla f(\bar{x}_t) \Vert - \bar{\eta} {\Vert \nabla f(\bar{x}_t) \Vert}^2\\
& \overset{(c)} \leq \frac{ L}{m} \left( {\sum\limits_{i=1}^{m} \eta_i \sum\limits_{\ell=0}^{\tau_i-1} \Vert x_{i,\ell}-\bar{x}_t\Vert}\right) \Vert \nabla f(\bar{x}_t) \Vert - \bar{\eta} {\Vert \nabla f(\bar{x}_t) \Vert}^2,
    \end{aligned}
\label{eqn:T1_thm_ct}
\end{equation}
where (a) follows from the Cauchy-Schwartz inequality, (b) follows from the triangle inequality, and (c) follows from the $L$-smoothness of each $f_i(\cdot)$. For the second term in \eqref{eqn:iterimbnd1_ct}, observe that 
\begin{equation}
    \begin{aligned}
 \frac{L}{2} \norm[\bigg]{\frac{1}{m}\sum\limits_{i=1}^{m} \eta_i \sum\limits_{\ell=0}^{\tau_i-1} \nabla f_i(x_{i,\ell})}^2 &= \frac{L}{2}\norm[\bigg]{\frac{1}{m}\sum\limits_{i=1}^{m} \eta_i \sum\limits_{\ell=0}^{\tau_i-1} \left( \nabla f_i(x_{i,\ell})-\nabla f_i(\bar{x}_t) \right) + \frac{1}{m} \sum\limits_{i=1}^{m} \eta_i \sum\limits_{\ell=0}^{\tau_i-1} \nabla f_i(\bar{x}_t)}^2\\
 &=\frac{ L}{2}\norm[\bigg]{\frac{1}{m}\sum\limits_{i=1}^{m} \eta_i \sum\limits_{\ell=0}^{\tau_i-1} \left( \nabla f_i(x_{i,\ell})-\nabla f_i(\bar{x}_t) \right) + \bar{\eta} \nabla f(\bar{x}_t)}^2\\
& \overset{(a)} \leq L \norm[\bigg]{\frac{1}{m}\sum\limits_{i=1}^{m}\eta_i \sum\limits_{\ell=0}^{\tau_i-1} \left( \nabla f_i(x_{i,\ell})-\nabla f_i(\bar{x}_t) \right)}^2 + {\bar{\eta}}^2 L {\Vert \nabla f(\bar{x}_t) \Vert}^2 \\
& \overset{(b)} \leq \frac{L}{m}  \sum\limits_{i=1}^{m} {\eta_i}^2 \norm[\bigg]{ \sum\limits_{\ell=0}^{\tau_i-1} \left( \nabla f_i(x_{i,\ell})-\nabla f_i(\bar{x}_t) \right)}^2 + {\bar{\eta}}^2 L {\Vert \nabla f(\bar{x}_t) \Vert}^2\\
& \overset{(c)} \leq \frac{L }{m}  \sum\limits_{i=1}^{m} {\eta_i}^2 \tau_i \sum\limits_{\ell=0}^{\tau_i-1} {\Vert  \nabla f_i(x_{i,\ell})-\nabla f_i(\bar{x}_t) \Vert }^2 + {\bar{\eta}}^2 L {\Vert \nabla f(\bar{x}_t) \Vert}^2 \\
& \overset{(d)} \leq \frac{\bar{\eta} L^3}{m}  \sum\limits_{i=1}^{m} \eta_i \sum\limits_{\ell=0}^{\tau_i-1} {\Vert  x_{i,\ell}-\bar{x}_t \Vert }^2 + {\bar{\eta}}^2 L {\Vert \nabla f(\bar{x}_t) \Vert}^2.
    \end{aligned}
\label{eqn:T2bound_ct}
g\end{equation}
For (a), we used \eqref{eqn:rel_triangle} with $\gamma=1$; for (b) and (c), we used Jensen's inequality \eqref{eqn:Jensens}; and for (d), we used $\eta_i \tau_i =\bar{\eta}$, and the $L$-smoothness of $f_i(\cdot)$. Plugging in the bounds \eqref{eqn:T1_thm_ct} and \eqref{eqn:T2bound_ct} in \eqref{eqn:iterimbnd1_ct}, and simplifying, we obtain \eqref{eqn:f_bound_ct}. 
\end{proof}

\begin{remark} Note that the proof of Lemma \ref{lemma:f_bound_ct} made no use of convexity. Thus, the same analysis will carry over when we will later  study the non-convex setting.
\end{remark} 

For both the strongly convex and convex settings, the following lemma provides a bound on the drift at client $i$. 

\begin{lemma}
Suppose each $f_i(x)$ is $L$-smooth and convex. Moreover, suppose $\tau_i \geq 1,\forall i\in\mathcal{S}$, $\delta_c=\delta_s=1$, and $\eta_i \leq \frac{1}{L}, \forall i \in \mathcal{S}$. Then,  \texttt{FedLin} guarantees the following bound for each $i\in\mathcal{S}$, and $\forall  \ell\in\{0,\ldots,\tau_i-1\}$: 
\begin{equation}
    \Vert{x_{i,\ell} - \bar{x}_t \Vert} \leq \eta_i \tau_i \Vert \nabla f(\bar{x}_t) \Vert.
\end{equation}
\label{lemma:drift_conv_ct}
\end{lemma}
\begin{proof}
Based on \eqref{eqn:local_update_ct}, for any client $i\in\mathcal{S}$, we have
\begin{equation}
    \begin{aligned}
\Vert x_{i,\ell+1} - \bar{x}_t \Vert & = \Vert x_{i,\ell} - \bar{x}_t - \eta_i \left(\nabla f_i(x_{i,\ell})-\nabla f_i(\bar{x}_t)\right) -\eta_i  \nabla f(\bar{x}_t) \Vert\\
& \leq \Vert x_{i,\ell} - \bar{x}_t - \eta_i \left(\nabla f_i(x_{i,\ell})-\nabla f_i(\bar{x}_t)\right) \Vert + \eta_i \Vert \nabla f(\bar{x}_t) \Vert \\
& \leq \Vert x_{i,\ell} - \bar{x}_t \Vert + \eta_i \Vert \nabla f(\bar{x}_t) \Vert,
    \end{aligned}
\label{eqn:driftbnd1_ct}
\end{equation}
where in the last step, we used the bound \eqref{eqn:two_pt_conv} from Lemma \ref{lemma:two_pt_bnd} that applies to both the convex and strongly convex settings. From \eqref{eqn:driftbnd1_ct}, and the fact that $x_{i,0}=\bar{x}_t, \forall i\in \mathcal{S}$, we immediately obtain
\begin{equation}
    \begin{aligned}
    \Vert x_{i,\ell} - \bar{x}_t \Vert & \leq \Vert x_{i,0} - \bar{x}_t \Vert + \eta_i \ell \Vert \nabla f(\bar{x}_t) \Vert \\
& \leq \eta_i \tau_i \Vert \nabla f(\bar{x}_t) \Vert,
    \end{aligned}
\end{equation}
which is the desired conclusion. 
\end{proof}

\textbf{Completing the proof of Theorem \ref{thm:strconv_dh}:} To complete the proof of Theorem \ref{thm:strconv_dh}, let us substitute the bound on the drift term from Lemma \ref{lemma:drift_conv_ct} in equation \eqref{eqn:f_bound_ct}.  This yields:
\begin{equation}
\begin{aligned}
    f(\bar{x}_{t+1})-f(\bar{x}_t) &\leq -\bar{\eta} \left(1-\bar{\eta} L\right){\Vert \nabla f(\bar{x}_t) \Vert}^2 + \left(\frac{L}{m} \sum\limits_{i=1}^{m} \eta_i \sum\limits_{\ell=0}^{\tau_i-1}\eta_i \tau_i \Vert \nabla f(\bar{x}_t) \Vert \right) \Vert \nabla f(\bar{x}_t) \Vert\\ 
    & + \frac{\bar{\eta} L^3 }{m} \sum\limits_{i=1}^{m} \eta_i  \sum\limits_{\ell=0}^{\tau_i-1}{\left(\eta_i \tau_i \Vert \nabla f(\bar{x}_t) \Vert\right)}^2\\
     &= -\bar{\eta} \left(1-\bar{\eta} L\right){\Vert \nabla f(\bar{x}_t) \Vert}^2 + \frac{L}{m} \sum\limits_{i=1}^{m} {(\eta_i \tau_i)}^2 { \Vert \nabla f(\bar{x}_t) \Vert}^2  + \frac{\bar{\eta} L^3 }{m} \sum\limits_{i=1}^{m} {(\eta_i \tau_i)}^3 {\Vert \nabla f(\bar{x}_t) \Vert}^2 \\
& = -\bar{\eta} \left(1-\bar{\eta} L\right){\Vert \nabla f(\bar{x}_t) \Vert}^2 + {\bar{\eta}}^2 L {\Vert \nabla f(\bar{x}_t) \Vert}^2 + {\bar{\eta}}^4 L^3 {\Vert \nabla f(\bar{x}_t) \Vert}^2\\
& \leq -\bar{\eta} \left(1-3\bar{\eta} L\right){\Vert \nabla f(\bar{x}_t) \Vert}^2,
\end{aligned}
\label{eqn:finalbnd_strconv_ct}
\end{equation}
where in the third step, we used $\eta_i \tau_i =\bar{\eta}$, and for the last inequality, we set the flexible parameter $\bar{\eta}$ to satisfy $\bar{\eta} L \leq 1$. In particular, setting  $\bar{\eta}=\frac{1}{6L}$, and using the fact that $f(\cdot)$ is $\mu$-strongly convex, we immediately obtain
    \begin{equation}
    f(\bar{x}_{t+1})-f(x^*) \leq \left(1-\frac{1}{6\kappa} \right) \left( f(\bar{x}_t) - f(x^*) \right), \textrm{where} \hspace{1mm} \kappa=\frac{L}{\mu}.
\end{equation}
Finally, note that $\bar{\eta}=\frac{1}{6L}$ implies that the step-size for client $i$ is $\eta_i=\frac{1}{6L\tau_i}$, which satisfies the requirement $\eta_i L \leq 1$ for Lemma \ref{lemma:drift_conv_ct} to hold, and is precisely the choice of step-size in the statement of the theorem. This completes the proof of Theorem \ref{thm:strconv_dh}. 
\subsection{Proof of Theorem \ref{thm:conv_dh}: Analysis for Convex loss functions}
In order to prove Theorem \ref{thm:conv_dh}, we will use a slightly different approach than that in Theorem \ref{thm:strconv_dh}. Instead of focusing on the quantity $f(\bar{x}_{t+1})-f(\bar{x}_{t})$, we will be interested in bounding  the distance to the optimal point $x^*$ at the end of the $t$-th communication round, namely ${\Vert \bar{x}_{t+1} - x^* \Vert}^2$. The following lemma is the starting point of our analysis.   

\begin{lemma}
Suppose each $f_i(x)$ is $L$-smooth and convex. Moreover, suppose $\tau_i \geq 1, \forall i\in\mathcal{S}$,  $\delta_c=\delta_s=1$, and $\eta_i=\frac{\bar{\eta}}{\tau_i}, \forall i \in \mathcal{S}$, where $\bar{\eta}\in (0,1)$. Then, \texttt{FedLin} guarantees:
\begin{equation}
\begin{aligned}
{\Vert \bar{x}_{t+1}-x^* \Vert}^2  \leq {\Vert \bar{x}_{t}-x^* \Vert}^2 -2 \bar{\eta} \left(f(\bar{x}_{t})-f(x^*) \right) + \frac{L(1
+2\bar{\eta} L)}{m} \sum\limits_{i=1}^{m} \eta_i \sum\limits_{\ell=0}^{\tau_i-1}{\Vert x_{i,\ell} - \bar{x}_t \Vert}^2 +  2{\bar{\eta}}^2 {\Vert \nabla f(\bar{x}_t) \Vert}^2.
\end{aligned}
\label{eqn:iterate_subopt_ct}
\end{equation}
\label{lemma:iterate_subopt_ct}
\end{lemma}
\begin{proof}
From \eqref{eqn:xbart_ct}, recall that
\begin{equation}
   \bar{x}_{t+1} =  \bar{x}_t-\frac{1}{m}  \sum\limits_{i=1}^{m} \eta_i \sum\limits_{\ell=0}^{\tau_i-1} \nabla f_i(x_{i,\ell}).
\end{equation}
Thus, we have
\begin{equation}
    {\Vert \bar{x}_{t+1}-x^* \Vert}^2 = {\Vert \bar{x}_{t}-x^* \Vert}^2 - \frac{2}{m} \langle \bar{x}_t - x^*, \sum\limits_{i=1}^{m} \eta_i \sum\limits_{\ell=0}^{\tau_i-1} \nabla f_i (x_{i,\ell}) \rangle + \norm[\bigg]{\frac{1}{m}\sum\limits_{i=1}^{m} \eta_i \sum\limits_{\ell=0}^{\tau_i-1} \nabla f_i (x_{i,\ell})}^2. 
\label{eqn:iteratebound1_ct}
\end{equation}
To proceed, let us first bound the cross-term as follows.
\begin{equation}
    \begin{aligned}
    - \frac{2}{m} \langle \bar{x}_t - x^*, \sum\limits_{i=1}^{m} \eta_i \sum\limits_{\ell=0}^{\tau_i-1} \nabla f_i (x_{i,\ell}) \rangle &= - \frac{2}{m} \sum\limits_{i=1}^{m} \eta_i \sum\limits_{\ell=0}^{\tau_i-1} \bigg(\langle \bar{x}_t - x_{i,\ell},   \nabla f_i (x_{i,\ell}) \rangle + \langle {x}_{i,\ell} - x^*,   \nabla f_i (x_{i,\ell}) \rangle \bigg) \\
    & \overset{(a)} \leq - \frac{2}{m} \sum\limits_{i=1}^{m} \eta_i \sum\limits_{\ell=0}^{\tau_i-1} \bigg(f_i(\bar{x}_t) - f_i(x_{i,\ell})-\frac{L}{2}{\Vert x_{i,\ell} -\bar{x}_t \Vert}^2 + \langle {x}_{i,\ell} - x^*,   \nabla f_i (x_{i,\ell}) \rangle \bigg)\\
  & \overset{(b)} \leq - \frac{2}{m} \sum\limits_{i=1}^{m} \eta_i \sum\limits_{\ell=0}^{\tau_i-1} \bigg(f_i(\bar{x}_t) - f_i(x_{i,\ell}) + f_i(x_{i,\ell}) -f_i(x^*)\bigg)\\
  & + \frac{L}{m} \sum\limits_{i=1}^{m} \eta_i \sum\limits_{\ell=0}^{\tau_i-1} {\Vert x_{i,\ell} -\bar{x}_t \Vert}^2 \\
& = - \frac{2}{m} \sum\limits_{i=1}^{m} \eta_i \tau_i \left(f_i(\bar{x}_t) - f_i(x^*)\right) + \frac{L}{m} \sum\limits_{i=1}^{m} \eta_i \sum\limits_{\ell=0}^{\tau_i-1} {\Vert x_{i,\ell} -\bar{x}_t \Vert}^2 \\
& = -2 \bar{\eta} \left(f(\bar{x}_t) - f(x^*)\right) + \frac{L}{m} \sum\limits_{i=1}^{m} \eta_i \sum\limits_{\ell=0}^{\tau_i-1} {\Vert x_{i,\ell} -\bar{x}_t \Vert}^2. 
    \end{aligned}
\label{eqn:T1_conv_ct}
\end{equation}
In the above steps, for (a) we used the fact that $f_i(\cdot)$ is $L$-smooth; see equation \eqref{eqn:smooth2}. For (b), we used the fact that $\langle {x}_{i,\ell} - x^*,   \nabla f_i (x_{i,\ell}) \rangle \geq f_i(x_{i,\ell}) - f_i(x^*)$ - a consequence of the convexity of $f_i(\cdot)$. Finally, to arrive at the last step, we used $\eta_i \tau_i = \bar{\eta}$. 

Now observe that while deriving the bound in Eq.  \eqref{eqn:T2bound_ct} of Lemma \ref{lemma:f_bound_ct}, we only used the fact that each $f_i(\cdot)$ is $L$-smooth. Thus, the same bound applies in our current setting. In particular, we have
\begin{equation}
    \begin{aligned}
  \norm[\bigg]{\frac{1}{m}\sum\limits_{i=1}^{m} \eta_i \sum\limits_{\ell=0}^{\tau_i-1} \nabla f_i(x_{i,\ell})}^2 \leq \frac{2 \bar{\eta} L^2}{m}  \sum\limits_{i=1}^{m} \eta_i \sum\limits_{\ell=0}^{\tau_i-1} {\Vert  x_{i,\ell}-\bar{x}_t \Vert }^2 + 2 {\bar{\eta}}^2 {\Vert \nabla f(\bar{x}_t) \Vert}^2.
    \end{aligned}
\end{equation}
Combining the above bound with that in \eqref{eqn:T1_conv_ct} immediately leads to \eqref{eqn:iterate_subopt_ct}. 
\end{proof}

\textbf{Establishing equation  \eqref{eqn:conv_ct_1} in Theorem \ref{thm:conv_dh}:} Suppose $\eta_i L \leq 1$. Then, based on Lemma \ref{lemma:drift_conv_ct}, recall that 
\begin{equation}
    \Vert{x_{i,\ell} - \bar{x}_t \Vert} \leq \eta_i \tau_i \Vert \nabla f(\bar{x}_t) \Vert, \forall i \in \mathcal{S}, \forall \ell \in \{0, \ldots \tau_i-1\}.
\end{equation}
Substituting the above bound on the drift term in \eqref{eqn:iterate_subopt_ct} yields:
\begin{equation}
\begin{aligned}
{\Vert \bar{x}_{t+1}-x^* \Vert}^2  & \leq {\Vert \bar{x}_{t}-x^* \Vert}^2 -2 \bar{\eta} \left(f(\bar{x}_{t})-f(x^*) \right) + \frac{L(1
+2\bar{\eta} L)}{m} \sum\limits_{i=1}^{m} \eta_i \sum\limits_{\ell=0}^{\tau_i-1} {(\eta_i \tau_i)}^2 {\Vert \nabla f(\bar{x}_t) \Vert}^2 +  2{\bar{\eta}}^2 {\Vert \nabla f(\bar{x}_t) \Vert}^2\\
& \overset{(a)} \leq {\Vert \bar{x}_{t}-x^* \Vert}^2 -2 \bar{\eta} \left(f(\bar{x}_{t})-f(x^*) \right) + \frac{3 L}{m} \sum\limits_{i=1}^{m} {(\eta_i \tau_i)}^3 {\Vert \nabla f(\bar{x}_t) \Vert}^2 +  2{\bar{\eta}}^2 {\Vert \nabla f(\bar{x}_t) \Vert}^2\\
& \overset{(b)} = {\Vert \bar{x}_{t}-x^* \Vert}^2 -2 \bar{\eta} \left(f(\bar{x}_{t})-f(x^*) \right) + 3{\bar{\eta}}^3 L {\Vert \nabla f(\bar{x}_t) \Vert}^2 +
2{\bar{\eta}}^2 {\Vert \nabla f(\bar{x}_t) \Vert}^2\\
& \overset{(c)} \leq {\Vert \bar{x}_{t}-x^* \Vert}^2 -2 \bar{\eta} \left(f(\bar{x}_{t})-f(x^*) \right) +
5{\bar{\eta}}^2 {\Vert \nabla f(\bar{x}_t) \Vert}^2\\
& \overset{(d)} \leq {\Vert \bar{x}_{t}-x^* \Vert}^2 -2 \bar{\eta} \left(1-5\bar{\eta} L \right) \left(f(\bar{x}_{t})-f(x^*) \right).
\end{aligned}
\label{eqn:conv_finalbd}
\end{equation}
In the above steps, for (a), we set $\bar{\eta}$ to satisfy $\bar{\eta} L \leq 1$; for (b), we used $\eta_i \tau_i =\bar{\eta}$; for (c), we once again used $\bar{\eta} L \leq 1$; and finally, for (d), we used the fact that $f(\cdot)$ is $L$-smooth; refer to equation \eqref{eqn:grad_upp}. Now rearranging terms in \eqref{eqn:conv_finalbd}, we obtain
\begin{equation}
 2 \bar{\eta} \left(1-5\bar{\eta} L \right) \left(f(\bar{x}_{t})-f(x^*) \right) \leq {\Vert \bar{x}_{t}-x^* \Vert}^2 - {\Vert \bar{x}_{t+1}-x^* \Vert}^2.
\end{equation}
Now let $\bar{\eta}=\frac{1}{10L}$, implying $\eta_i=\frac{1}{10L \tau_i}, \forall i\in\mathcal{S}$. Clearly, the conditions $\eta_i L \leq 1$ and $\bar{\eta} L \leq 1$ are then  satisfied, and we have
\begin{equation}
  f(\bar{x}_{t})-f(x^*)  \leq 10L \left({\Vert \bar{x}_{t}-x^* \Vert}^2 - {\Vert \bar{x}_{t+1}-x^* \Vert}^2 \right).
\end{equation}
Summing the above inequality from $t=1$ to $t=T$ leads to a telescoping sum on the R.H.S., and we obtain
\begin{equation} f\left(\frac{1}{T} \sum\limits_{t=1}^{T} \bar{x}_t \right) - f(x^*) \overset{(a)} \leq 
   \frac{1}{T} \sum\limits_{t=1}^{T}\big(f(\bar{x}_{t})-f(x^*)\big) \leq \frac{10L}{T} \left({\Vert \bar{x}_{1}-x^* \Vert}^2 - {\Vert \bar{x}_{T+1}-x^* \Vert}^2 \right),
\end{equation}
where (a) follows from the convexity of $f(\cdot)$. This establishes equation \eqref{eqn:conv_ct_1} in Theorem \ref{thm:conv_dh}. 

\textbf{Establishing equation  \eqref{eqn:conv_ct_2} in Theorem \ref{thm:conv_dh}:} In order to justify the claim made in equation \eqref{eqn:conv_ct_2}, let us revisit the proof of Theorem \ref{thm:strconv_dh} and note that for arriving at \eqref{eqn:finalbnd_strconv_ct}, we only require each $f_i(\cdot)$ to be $L$-smooth and convex, and $\eta_i$ and $\bar{\eta}$ to satisfy $\eta_i L \leq 1$, $\bar{\eta} L \leq 1$. Accordingly, if $\eta_i L \leq 1$ and $\bar{\eta} L \leq 1$, then the following inequality holds: 
\begin{equation}
  f(\bar{x}_{t+1}) - f(\bar{x}_t)  \leq -\bar{\eta} \left(1-3\bar{\eta} L\right){\Vert \nabla f(\bar{x}_t) \Vert}^2.
\label{eqn:claim_2_thm_conv}
\end{equation}
Now observe that if $\bar{\eta} \leq \frac{1}{5L}$, then based on \eqref{eqn:conv_finalbd}, the sequence $a_t\triangleq {\Vert \bar{x}_t - x^* \Vert}^2$ is non-increasing. We thus obtain:
\begin{equation}
    f(\bar{x}_t) - f(x^*) \overset{(a)} \leq  \langle \bar{x}_t - x^*, \nabla f(\bar{x}_t) \rangle \overset{(b)} \leq \Vert \bar{x}_t - x^* \Vert \Vert \nabla f(\bar{x}_t) \Vert \overset{(c)} \leq \Vert \bar{x}_1 - x^* \Vert \Vert \nabla f(\bar{x}_t) \Vert,
\end{equation}
where (a) follows from convexity of $f(\cdot)$, (b) follows from the Cauchy-Schwartz inequality, and (c) follows from the fact that $a_t$ is non-increasing. We can now lower-bound  $\Vert \nabla f(\bar{x}_t) \Vert$ as follows:\footnote{To avoid trivialities, we have assumed that $\bar{x}_1 \neq x^*.$}
\begin{equation}
     \Vert \nabla f(\bar{x}_t) \Vert \geq \frac{f(\bar{x}_t) - f(x^*)}{\Vert \bar{x}_1 - x^* \Vert}.
\end{equation}
Plugging in the above bound in \eqref{eqn:claim_2_thm_conv} yields:
\begin{equation}
    f(\bar{x}_{t+1}) - f(\bar{x}_t)  \leq -\bar{\eta} \left(1-3\bar{\eta} L\right) {\bigg(\frac{f(\bar{x}_t) - f(x^*)}{\Vert \bar{x}_1 - x^* \Vert}\bigg)}^2.
\end{equation}
Let $\bar{\eta}=\frac{1}{6L}$ and  $r_t \triangleq f(\bar{x}_t)-f(x^*)$. This leads to the following recursion:
\begin{equation}
    r_{t+1} \leq r_t - c r^2_t, \, \textrm{where} \hspace{1.5mm} c=\frac{1}{12{\Vert \bar{x}_1 - x^* \Vert}^2 L}.
\end{equation}
Thus, we have
\begin{equation}
    \frac{1}{r_t} \leq \frac{1}{r_{t+1}} - c \frac{r_t}{r_{t+1}} \implies \frac{1}{r_t} \leq \frac{1}{r_{t+1}} - c \implies c \leq \frac{1}{r_{t+1}}- \frac{1}{r_t},  
\end{equation}
where we used $r_{t+1} \leq r_{t}$. Summing the final inequality from $t=1$ to $t=T-1$ yields
\begin{equation}
    c(T-1) \leq \frac{1}{r_T} - \frac{1}{r_1}.
\end{equation}
Rearranging and simplifying the above inequality, we obtain \eqref{eqn:conv_ct_2}. 
\subsection{Proof of Theorem \ref{thm:nonconv_dh}: Analysis for Non-convex loss functions}
To analyze the non-convex setting, the first key observation that we make is that the claim in Lemma \ref{lemma:f_bound_ct} relies only on smoothness of the client loss functions $f_i(\cdot)$, and requires no assumptions of convexity. Thus, the bound in \eqref{eqn:f_bound_ct} applies to the non-convex setting as well. However, the bound on the drift term $\Vert x_{i,\ell} - \bar{x}_t \Vert$ that we derived in Lemma \ref{lemma:drift_conv_ct} did make use of convexity of each $f_i(\cdot)$, and hence, is no longer applicable. We thus need a way to bound $\Vert x_{i,\ell} - \bar{x}_t \Vert$ without making any assumptions of convexity; this is precisely the subject of the following lemma.
\begin{lemma}
Suppose each $f_i(x)$ is $L$-smooth. Moreover, suppose $\tau_i \geq 1,\forall i\in\mathcal{S}$, $\delta_c=\delta_s=1$, and $\eta_i \leq \frac{1}{L \tau_i}, \forall i \in \mathcal{S}$. Then,  \texttt{FedLin} guarantees the following bound for each $i\in\mathcal{S}$, and $\forall \ell\in\{0,\ldots,\tau_i-1\}$: 
\begin{equation}
    \Vert{x_{i,\ell} - \bar{x}_t \Vert} \leq 3 \eta_i \tau_i \Vert \nabla f(\bar{x}_t) \Vert.
\end{equation}
\label{lemma:drift_nonconv_ct}
\end{lemma}
\begin{proof}
From \eqref{eqn:local_update_ct}, we have
\begin{equation}
    \begin{aligned}
\Vert x_{i,\ell+1} - \bar{x}_t \Vert & = \Vert x_{i,\ell} - \bar{x}_t - \eta_i \left(\nabla f_i(x_{i,\ell})-\nabla f_i(\bar{x}_t)\right) -\eta_i  \nabla f(\bar{x}_t) \Vert\\
& \leq \Vert x_{i,\ell} - \bar{x}_t \Vert +  \eta_i \Vert \nabla f_i(x_{i,\ell})-\nabla f_i(\bar{x}_t) \Vert + \eta_i \Vert \nabla f(\bar{x}_t) \Vert \\
& \overset{(a)} \leq \left(1+\eta_i L \right) \Vert x_{i,\ell} - \bar{x}_t \Vert + \eta_i \Vert \nabla f(\bar{x}_t) \Vert\\
& \overset{(b)} \leq \left(1+\frac{1}{\tau_i} \right) \Vert x_{i,\ell} - \bar{x}_t \Vert + \eta_i \Vert \nabla f(\bar{x}_t) \Vert,
\end{aligned}
\end{equation}
where (a) follows from the  $L$-smoothness of $f_i(\cdot)$, and (b) follows by noting $\eta_i\leq \frac{1}{L\tau_i}$. Let $\alpha_i=1+\frac{1}{\tau_i}$. We then have
\begin{equation}
  \begin{aligned}
\Vert x_{i,\ell} - \bar{x}_t \Vert & \leq \alpha^{\ell}_i \underbrace{\Vert x_{i,0} - \bar{x}_t \Vert}_{= 0} + \left(\sum\limits_{j=0}^{\ell-1}\alpha^j_i \right) \eta_i \Vert \nabla f(\bar{x}_t) \Vert \\
&= \left(\sum\limits_{j=0}^{\ell-1} \alpha^j_i \right) \eta_i \Vert \nabla f(\bar{x}_t) \Vert \\
& = \left(\frac{\alpha^{\ell}_i-1}{\alpha_i-1}\right)\eta_i \Vert \nabla f(\bar{x}_t) \Vert\\
& \leq \eta_i \tau_i {\left(1+\frac{1}{\tau_i} \right)}^{\ell} \Vert \nabla f(\bar{x}_t) \Vert \\
& \leq \eta_i \tau_i {\left(1+\frac{1}{\tau_i} \right)}^{\tau_i} \Vert \nabla f(\bar{x}_t) \Vert \\
& \leq \exp{(1)} \eta_i \tau_i \Vert \nabla f(\bar{x}_t) \Vert \leq 3 \eta_i \tau_i \Vert \nabla f(\bar{x}_t) \Vert,
    \end{aligned}  
\end{equation}
which is the desired conclusion.
\end{proof}

\begin{remark}
In comparison with Lemma \ref{lemma:drift_conv_ct} where we derived bounds on the drift for the strongly convex and convex settings, the requirement on the step-size for bounding the drift in Lemma \ref{lemma:drift_nonconv_ct} is more stringent: whereas $\eta_i\leq \frac{1}{L}$ sufficed for Lemma \ref{lemma:drift_conv_ct} to hold, we need $\eta_i\leq \frac{1}{L\tau_i}$ for the non-convex setting in Lemma \ref{lemma:drift_nonconv_ct}. Also, note that the bound in Lemma \ref{lemma:drift_nonconv_ct} is worse than that in Lemma \ref{lemma:drift_conv_ct} by a factor of $3$. 
\end{remark}

\textbf{Completing the proof of Theorem \ref{thm:nonconv_dh}:} By substituting the bound on the drift from Lemma \ref{lemma:drift_nonconv_ct} in equation \eqref{eqn:f_bound_ct}, we obtain 
\begin{equation}
\begin{aligned}
    f(\bar{x}_{t+1})-f(\bar{x}_t) &\leq -\bar{\eta} \left(1-\bar{\eta} L\right){\Vert \nabla f(\bar{x}_t) \Vert}^2 + \left(\frac{L}{m} \sum\limits_{i=1}^{m} \eta_i \sum\limits_{\ell=0}^{\tau_i-1} 3 \eta_i \tau_i \Vert \nabla f(\bar{x}_t) \Vert \right) \Vert \nabla f(\bar{x}_t) \Vert\\ 
    & + \frac{\bar{\eta} L^3 }{m} \sum\limits_{i=1}^{m} \eta_i  \sum\limits_{\ell=0}^{\tau_i-1}{\left(3 \eta_i \tau_i \Vert \nabla f(\bar{x}_t) \Vert\right)}^2\\
     &= -\bar{\eta} \left(1-\bar{\eta} L\right){\Vert \nabla f(\bar{x}_t) \Vert}^2 + \frac{3 L}{m} \sum\limits_{i=1}^{m} {(\eta_i \tau_i)}^2 { \Vert \nabla f(\bar{x}_t) \Vert}^2  + \frac{9 \bar{\eta} L^3 }{m} \sum\limits_{i=1}^{m} {(\eta_i \tau_i)}^3 {\Vert \nabla f(\bar{x}_t) \Vert}^2 \\
& = -\bar{\eta} \left(1-\bar{\eta} L\right){\Vert \nabla f(\bar{x}_t) \Vert}^2 + 3 {\bar{\eta}}^2 L {\Vert \nabla f(\bar{x}_t) \Vert}^2 + 9 {\bar{\eta}}^4 L^3 {\Vert \nabla f(\bar{x}_t) \Vert}^2\\
& \leq -\bar{\eta} \left(1-13\bar{\eta} L\right){\Vert \nabla f(\bar{x}_t) \Vert}^2,
\end{aligned}
\label{eqn:nonconv_finalbd}
\end{equation}
where we used $\eta_i \tau_i =\bar{\eta}$ and $\bar{\eta} L \leq 1$ in the above steps. Now suppose $\bar{\eta}=\frac{1}{26L}$, implying $\eta_i=\frac{1}{26L\tau_i} < \frac{1}{L \tau_i}$. Observe that this choice of $\bar{\eta}$ fulfils the step-size requirements for Lemma \ref{lemma:drift_nonconv_ct} to hold. Plugging $\bar{\eta}=\frac{1}{26L}$ in \eqref{eqn:nonconv_finalbd} yields
\begin{equation}
{\Vert \nabla f(\bar{x}_t) \Vert}^2 \leq  52L \big(f(\bar{x}_t)-f(\bar{x}_{t+1})\big). 
\end{equation}
Summing the above inequality from $t=1$ to $t=T$, we obtain as desired
\begin{equation}
    \min_{t\in[T]} {\Vert \nabla f(\bar{x}_t) \Vert}^2 \leq \frac{1}{T} \sum\limits_{t=1}^{T} {\Vert \nabla f(\bar{x}_t) \Vert}^2 \leq \frac{52L}{T} \big(f(\bar{x}_1)-f(\bar{x}_{T+1})\big). 
\end{equation}

\textbf{Proof of Theorem \ref{thm:nonconv_PL}:} The proof of Theorem \ref{thm:nonconv_PL} is exactly the same as that of Theorem \ref{thm:nonconv_dh} up to equation \eqref{eqn:nonconv_finalbd}. The rest simply follows from the PL condition. 
\newpage
\subsection{Proof of Theorem \ref{thm:noisy}: Analysis for Strongly Convex loss functions with Noise}
In this section, we focus on analyzing the performance of \texttt{FedLin} under a general stochastic oracle model, subject to arbitrary objective and systems heterogeneity. For each $i\in \mathcal{S}$ and $x\in\mathbb{R}^d$, let $q_i(x)$ be an unbiased estimate of the gradient $\nabla f_i(x)$ with variance bounded above by $\sigma^2$. Our goal is to then analyze the following noisy update rule for \texttt{FedLin}: 
\begin{equation}
x^{(t)}_{i,\ell+1} = x^{(t)}_{i,\ell}-\eta_i(q_i(x^{(t)}_{i,\ell})- q_i(\bar{x}_{t})+q(\bar{x}_t)),
\label{eqn:noisyFedLin}
\end{equation}
where $q(x) \triangleq 1/m \sum_{i\in\mathcal{S}} q_i(x), \forall x\in\mathbb{R}^d.$ For our subsequent analysis, we will use $\mathcal{F}^{(t)}_{i,\ell}$ to denote the filtration that captures all the randomness up to the $\ell$-th local step of client $i$ in round $t$. We will also use $\mathcal{F}^{(t)}$ to represent the filtration capturing all the randomness up to the end of round $t-1$. With a slight abuse of notation, $\mathcal{F}^{(t)}_{i,-1}$ is to be interpreted as $\mathcal{F}^{(t)}, \forall i \in \mathcal{C}$. 

We begin our analysis of Theorem \ref{thm:noisy} with the following lemma.
\begin{lemma}
\label{thm4:lemma1}
Suppose each $f_i(x)$ is $L$-smooth. Moreover, suppose $\tau_i \geq 1$, $\forall i \in \mathcal{S}$, $\delta_c=\delta_s=1$ and $\eta_i=\dfrac{\bar{\eta}}{\tau_i}$, $\forall i \in \mathcal{S}$, where $\bar{\eta} \in (0,1)$. Under the stochastic oracle model defined in Section \ref{sec:speed_acc}, \texttt{FedLin} guarantees:
\begin{equation}
\label{thm4:driftbound}
\begin{aligned}
    \mathbb{E}\Big[\|\bar{x}_{t+1}-x^*\|^2\Big] &\leq \mathbb{E}\Big[\|\bar{x}_{t}-x^*\|^2\Big] -2\bar{\eta} \mathbb{E}\Big[f(\bar{x}_{t})-f(x^*)\Big] + 4 \bar{\eta}^2\mathbb{E}\Big[\|\nabla f(\bar{x}_t)\|^2\Big] \\
    & +\dfrac{L(1+6\bar{\eta}L)}{m}\sum \limits_{i=1}^m \eta_i \sum \limits_{\ell=0}^{\tau_i-1}\mathbb{E}\Big[\|\bar{x}_{t}-x_{i,\ell}\|^2\Big]+16\bar{\eta}^2 \sigma^2.
\end{aligned}
\end{equation}
\end{lemma}
\begin{proof}
From the noisy local update rule of \texttt{FedLin} in Eq. \eqref{eqn:noisyFedLin}, we have:
\begin{equation}
    \begin{aligned}
        \bar{x}_{t+1}&=\bar{x}_t-\dfrac{1}{m}\sum \limits_{i=1}^m \eta_i \sum \limits_{\ell=0}^{\tau_i-1}\Big[q_i(x_{i,\ell})-q_i(\bar{x}_t)+q(\bar{x}_t)\Big] \\
        &=\bar{x}_t-\dfrac{1}{m}\sum \limits_{i=1}^m \eta_i \sum \limits_{\ell=0}^{\tau_i-1}q_i(x_{i,\ell}).
    \end{aligned}
\end{equation}
Hence, we obtain
\begin{equation}
\label{lemma12:eqn1}
    \begin{aligned}
        \|\bar{x}_{t+1}-x^*\|^2&=\|\bar{x}_{t}-x^*\|^2+\underbrace{\Big\|\dfrac{1}{m}\sum \limits_{i=1}^m \eta_i \sum \limits_{\ell=0}^{\tau_i-1}q_i(x_{i,\ell})\Big\|^2}_{\mathcal{A}_1}\\
        &\underbrace{-2 \big\langle \bar{x}_t-x^*,\dfrac{1}{m}\sum \limits_{i=1}^m \eta_i \sum \limits_{\ell=0}^{\tau_i-1}q_i(x_{i,\ell}) \big\rangle}_{\mathcal{A}_2}.
    \end{aligned}
\end{equation}
We begin by bounding the term $\mathcal{A}_1$ in \eqref{lemma12:eqn1} as follows:
\begin{equation}
\label{lemma12:eqn2}
    \begin{aligned}
        \mathcal{A}_1&=\Big\|\dfrac{1}{m}\sum \limits_{i=1}^m \eta_i \sum \limits_{\ell=0}^{\tau_i-1}q_i(x_{i,\ell})\Big\|^2\\
        &=\Big\|\dfrac{1}{m}\sum \limits_{i=1}^m \eta_i \sum \limits_{\ell=0}^{\tau_i-1}(q_i(x_{i,\ell})-q_i(\bar{x}_t)+q_i(\bar{x}_t))\Big\|^2\\
        &=\Big\|\dfrac{1}{m}\sum \limits_{i=1}^m \eta_i \sum \limits_{\ell=0}^{\tau_i-1}(q_i(x_{i,\ell})-q_i(\bar{x}_t))+\bar{\eta}q(\bar{x}_t)\Big\|^2\\
        &\leq\underbrace{2\Big\|\dfrac{1}{m}\sum \limits_{i=1}^m \eta_i \sum \limits_{\ell=0}^{\tau_i-1}(q_i(x_{i,\ell})-q_i(\bar{x}_t))\Big\|^2}_{\mathcal{T}_1}+\underbrace{2\bar{\eta}^2\|q(\bar{x}_t)\|^2}_{\mathcal{T}_2}.
    \end{aligned}
\end{equation}
The term $\mathcal{T}_1$ in \eqref{lemma12:eqn2} can be upper bounded as follows:
\begin{equation}
\label{lemma12:eqn3}
    \begin{aligned}
        \mathcal{T}_1&=2\Big\|\dfrac{1}{m}\sum \limits_{i=1}^m \eta_i \sum \limits_{\ell=0}^{\tau_i-1}\big(q_i(x_{i,\ell})-q_i(\bar{x}_t)\big)\Big\|^2 \\
        & {\leq} \dfrac{2}{m}\sum \limits_{i=1}^m \eta_i^2 \tau_i \sum \limits_{\ell=0}^{\tau_i-1}\|q_i(x_{i,\ell})-q_i(\bar{x}_t)\|^2\\
        &= \dfrac{2\bar{\eta}}{m}\sum \limits_{i=1}^m \eta_i\sum \limits_{\ell=0}^{\tau_i-1}\|q_i(x_{i,\ell})-q_i(\bar{x}_t)\|^2 \\
        &= \dfrac{2\bar{\eta}}{m}\sum \limits_{i=1}^m \eta_i\sum \limits_{\ell=0}^{\tau_i-1}\|q_i(x_{i,\ell})-\nabla f_i(x_{i,\ell})+\nabla f_i(x_{i,\ell})-\nabla f_i(\bar{x}_t)+\nabla f_i(\bar{x}_t)-q_i(\bar{x}_t)\|^2\\
        &\overset{(b)}\leq \dfrac{6\bar{\eta}}{m}\sum \limits_{i=1}^m \eta_i\sum \limits_{\ell=0}^{\tau_i-1}\Big(\|q_i(x_{i,\ell})-\nabla f_i(x_{i,\ell})\|^2+\|\nabla f_i(x_{i,\ell})-\nabla f_i(\bar{x}_t)\|^2+\|\nabla f_i(\bar{x}_t)-q_i(\bar{x}_t)\|^2\Big),
    \end{aligned}
\end{equation}
where steps $(a)$ and $(b)$ follow from an application of equation \eqref{eqn:Jensens}. 
Taking expectation on both sides of equation \eqref{lemma12:eqn3}, we obtain:
\begin{equation}
\label{lemma12:eqn6}
    \begin{aligned}
    \mathbb{E}[\mathcal{T}_1]&\leq \dfrac{6\bar{\eta}}{m}\sum \limits_{i=1}^m \eta_i\sum \limits_{\ell=0}^{\tau_i-1}\bigg(\mathbb{E}\bigg[\mathbb{E}\Big[\|q_i(x_{i,\ell})-\nabla f_i(x_{i,\ell})\|^2|\mathcal{F}^{(t)}_{i,\ell-1}\Big]\bigg]+\mathbb{E}\Big[\|\nabla f_i(x_{i,\ell})-\nabla f_i(\bar{x}_t)\|^2\Big]\\
    &+\mathbb{E}\bigg[\mathbb{E}\Big[\|\nabla f_i(\bar{x}_t)-q_i(\bar{x}_t)\|^2|\mathcal{F}^{(t)}\Big]\bigg]\bigg)\\
    &\overset{(a)}\leq \dfrac{6\bar{\eta}}{m}\sum \limits_{i=1}^m \eta_i\sum \limits_{\ell=0}^{\tau_i-1}\Big(2\sigma^2+L^2\mathbb{E}\Big[\|x_{i,\ell}-\bar{x}_t\|^2\Big]\Big) \\
    &\leq 12\bar{\eta}^2\sigma^2+\dfrac{6\bar{\eta}L^2}{m}\sum \limits_{i=1}^m \eta_i\sum \limits_{\ell=0}^{\tau_i-1}\mathbb{E}\Big[\|x_{i,\ell}-\bar{x}_t\|^2\Big].
    \end{aligned}
\end{equation}
For $(a)$, we used the following facts: (i) $x^{(t)}_{i,\ell}$ is $\mathcal{F}^{(t)}_{i,\ell-1}$-adapted; (ii) $\bar{x}_t$ is $\mathcal{F}^{(t)}$-adapted; (iii) the variance of the noise model is bounded above by $\sigma^2$, and (iv) the loss functions are $L$-smooth.

Next, the term $\mathcal{T}_2$ in \eqref{lemma12:eqn2} can be upper bounded as follows:
\begin{equation}
\label{lemma12:eqn4}
    \begin{aligned}
        \mathcal{T}_2&=2\bar{\eta}^2\|q(\bar{x}_t)\|^2\\
                    & =  2\bar{\eta}^2\|q(\bar{x}_t)-\nabla f(\bar{x}_t)+\nabla f(\bar{x}_t)\|^2\\
                    &=2\bar{\eta}^2\Big\|\dfrac{1}{m}\sum \limits_{i=1}^m (q_i(\bar{x}_t)-\nabla f_i(\bar{x}_t))+\nabla f(\bar{x}_t)\Big\|^2\\
                    &\leq 4\bar{\eta}^2\Big\|\dfrac{1}{m}\sum \limits_{i=1}^m (q_i(\bar{x}_t)-\nabla f_i(\bar{x}_t))\Big\|^2+4\bar{\eta}^2\|\nabla f(\bar{x}_t)\|^2\\
                    &\overset{(a)} \leq \dfrac{4 \bar{\eta}^2}{m}\sum \limits_{i=1}^m\|q_i(\bar{x}_t)-\nabla f_i(\bar{x}_t)\|^2+4\bar{\eta}^2\|\nabla f(\bar{x}_t)\|^2,
    \end{aligned}
\end{equation}
where step $(a)$ follows from an application of equation \eqref{eqn:Jensens}.
Taking expectation on both sides of equation \eqref{lemma12:eqn4}, and using the bounded-variance property, we obtain:
\begin{equation}
\label{lemma12:eqn7}
    \begin{aligned}
        \mathbb{E}[\mathcal{T}_2]&\leq 4\bar{\eta}^2\bigg(\sigma^2+\mathbb{E}\Big[\|\nabla f(\bar{x}_t)\|^2\Big]\bigg).
    \end{aligned}
\end{equation}
We now proceed to bound the expectation of the term $\mathcal{A}_2$ in \eqref{lemma12:eqn1} as follows:
\begin{equation}
\label{lemma12:eqn5}
    \begin{aligned}
       \mathbb{E}[\mathcal{A}_2]&= \mathbb{E}\bigg[-2 \big\langle \bar{x}_t-x^*,\dfrac{1}{m}\sum \limits_{i=1}^m \eta_i \sum \limits_{\ell=0}^{\tau_i-1}q_i(x_{i,\ell})
       \big\rangle\bigg]\\
        &=\dfrac{-2}{m}\sum \limits_{i=1}^m \eta_i \sum \limits_{\ell=0}^{\tau_i-1}\mathbb{E}\Big[\langle \bar{x}_t-x^*,q_i(x_{i,\ell})\rangle \Big]\\
       &=\dfrac{-2}{m}\sum \limits_{i=1}^m \eta_i \sum \limits_{\ell=0}^{\tau_i-1}\mathbb{E}\bigg[\mathbb{E}\Big[\langle \bar{x}_t-x^*,q_i(x_{i,\ell})\rangle| \mathcal{F}^{(t)}_{i,\ell-1}\Big] \bigg]\\
       &\overset{(a)}=\dfrac{-2}{m}\sum \limits_{i=1}^m \eta_i \sum \limits_{\ell=0}^{\tau_i-1}\mathbb{E}\Big[\langle \bar{x}_t-x^*,\nabla f_i(x_{i,\ell})\rangle \Big] \\
       &\overset{(b)}\leq \dfrac{-2}{m}\sum \limits_{i=1}^m \eta_i \sum \limits_{\ell=0}^{\tau_i-1}\mathbb{E}\Big[ f_i(\bar{x}_t)-f_i(x^*)-\dfrac{L}{2}\|x_{i,\ell}-x^*\|^2\Big]\\
       &=-2\bar{\eta}\mathbb{E}\Big[ f(\bar{x}_t)-f(x^*)\Big]+\dfrac{L}{m}\sum \limits_{i=1}^m \eta_i \sum \limits_{\ell=0}^{\tau_i-1}\mathbb{E}\Big[\|x_{i,\ell}-\bar{x}_t\|^2\Big].
    \end{aligned}
\end{equation}
For step $(a)$, we used the unbiasedness property of the noise model, and for step $(b)$, we employed the same reasoning as that used to arrive at equation \eqref{eqn:T1_conv_ct}. 

Taking expectation on both sides of equation \eqref{lemma12:eqn1}, and  combining the bounds in equations \eqref{lemma12:eqn6}, \eqref{lemma12:eqn7} and \eqref{lemma12:eqn5} establishes the claim of Lemma \ref{thm4:lemma1}.
\end{proof}

For $L$-smooth and $\mu$-strongly convex loss functions, the following lemma provides a bound on the drift at client $i$.
\begin{lemma}
\label{thm4:lemma2}
Suppose each $f_i(x)$ is $L$-smooth and $\mu$-strongly convex. Moreover, suppose $\tau_i \geq 1$, $\forall i \in \mathcal{S}$, $\delta_c=\delta_s=1$ and $\eta_i=\dfrac{\bar{\eta}}{\tau_i}$, $\forall i \in \mathcal{S}$, where $\bar{\eta} \in (0,1)$. Under the stochastic oracle model defined in Section \ref{sec:speed_acc}, \texttt{FedLin} guarantees the following bound for each $i \in \mathcal{S}$, and $\forall \ell \in \{0,\ldots, \tau_i-1\}$:
\begin{equation}
\label{lemma14:eqn1}
    \mathbb{E}\Big[\|x_{i,\ell}-\bar{x}_t\|^2\Big]\leq 6 \bar{\eta}^2\mathbb{E}\Big[\|\nabla f(\bar{x}_t)\|^2\Big]+27\eta_i\bar{\eta}\sigma^2.
\end{equation}
\end{lemma}
\begin{proof}
From the noisy local update rule of \texttt{Fedlin} in Eq. \eqref{eqn:noisyFedLin}, we have:
\begin{equation}
    \begin{aligned}
        x_{i,\ell+1}&=x_{i,\ell}-\eta_i\big(q_i(x_{i,\ell})-q_i(\bar{x}_t)+q(\bar{x}_t)\big)\\
        &=x_{i,\ell}-\eta_i\big(q_i(x_{i,\ell})-\nabla f_i(x_{i,\ell})+\nabla f_i(x_{i,\ell})-\nabla f_i(\bar{x}_t)\\&+\nabla f_i(\bar{x}_t)-q_i(\bar{x}_t)+q(\bar{x}_t)\big).
    \end{aligned}
\end{equation}
Hence, we have:
\begin{equation}
    \begin{aligned}
        x_{i,\ell+1}-\bar{x}_t&=(x_{i,\ell}-\bar{x}_t)-\eta_i\big(
        \underbrace{\nabla f_i(x_{i,\ell})-\nabla f_i(\bar{x}_t)+\nabla f(\bar{x}_t)}_{\mathcal{V}}\big)\\
        &-\eta_i\big(\underbrace{q_i(x_{i,\ell})-\nabla f_i(x_{i,\ell})+\nabla f_i(\bar{x}_t)-q_i(\bar{x}_t)+q(\bar{x}_t)-\nabla f (\bar{x}_t)}_{\mathcal{W}}\big)
    \end{aligned}
\end{equation}
Consequently, we may write:
\begin{equation}
\label{lemma13:eqn2}
    \begin{aligned}
        \|x_{i,\ell+1}-\bar{x}_t\|^2&=\|(x_{i,\ell}-\bar{x}_t-\eta_i{\mathcal{V}})-\eta_i{\mathcal{W}}\|^2\\
        &=\|(x_{i,\ell}-\bar{x}_t-\eta_i{\mathcal{V}})\|^2+\eta_i^2\|{\mathcal{W}}\|^2-2\langle x_{i,\ell}-\bar{x}_t-\eta_i{\mathcal{V}},\eta_i{\mathcal{W}} \rangle.
    \end{aligned}
\end{equation}
Taking expectation on both sides of equation \eqref{lemma13:eqn2}, we have:
\begin{equation}
\label{lemma13:eqn3}
    \begin{aligned}
        \mathbb{E}\big[\|x_{i,\ell+1}-\bar{x}_t\|^2\big]
        &=\mathbb{E}\big[\|(x_{i,\ell}-\bar{x}_t-\eta_i{\mathcal{V}})\|^2\big]+\eta_i^2\mathbb{E}\big[\|{\mathcal{W}}\|^2\big]\\&-2\mathbb{E}\bigg[\mathbb{E}\Big[\langle x_{i,\ell}-\bar{x}_t-\eta_i{\mathcal{V}},\eta_i{\mathcal{W}} \rangle|\mathcal{F}_{i,\ell-1}^{(t)}\Big]\bigg] \\ 
        &\overset{(a)}= \underbrace{\mathbb{E}\big[\|x_{i,\ell}-\bar{x}_t-\eta_i{\mathcal{V}}\|^2\big]}_{\mathcal{A}_1}+\eta_i^2\underbrace{\mathbb{E}\big[\|{\mathcal{W}}\|^2\big]}_{\mathcal{A}_2}.
    \end{aligned}
\end{equation}
For $(a)$, we used the following facts: (i) $x_{i,\ell}, \bar{x}_t$, and $\mathcal{V}$, are all adapted to $\mathcal{F}_{i,\ell-1}^{(t)}$, and (ii) $\mathbb{E}\big[\mathcal{W}|\mathcal{F}_{i,\ell-1}^{(t)}\big]=0$ based on fact (i) and the unbiasedness property of the noise model. 

We now proceed to bound the term $\mathcal{A}_1$ in equation \eqref{lemma13:eqn3} as follows:
\begin{equation}
\label{lemma13:eqn5}
    \begin{aligned}
       \mathbb{E}\big[\|(x_{i,\ell}-\bar{x}_t-\eta_i{\mathcal{V}})\|^2\big]&=\mathbb{E}\big[\|x_{i,\ell}-\bar{x}_t-\eta_i(\nabla f_i(x_{i,\ell})-\nabla f_i(\bar{x}_t))-\eta_i\nabla f(\bar{x}_t)\|^2\big] \\
       &\leq (1+\dfrac{1}{\gamma})\mathbb{E}\Big[\|x_{i,\ell}-\bar{x}_t\|^2\Big]+(1+\gamma)\eta_i^2\mathbb{E}\Big[\|\nabla f(\bar{x}_t)\|^2\Big].
    \end{aligned}
\end{equation}
The above inequality follows from the application of equation \eqref{eqn:rel_triangle} and Lemma \ref{lemma:two_pt_bnd}.

To bound the term $\mathcal{A}_2$ in equation \eqref{lemma13:eqn3}, we proceed as follows:
\begin{equation}
\label{lemma13:eqn4}
    \begin{aligned}
        \|\mathcal{W}\|^2&=\|q_i(x_{i,\ell})-\nabla f_i(x_{i,\ell})+\nabla f_i(\bar{x}_t)-q_i(\bar{x}_t)+q(\bar{x}_t)-\nabla f (\bar{x}_t)\|^2\\
        &\leq 3\|q_i(x_{i,\ell})-\nabla f_i(x_{i,\ell})\|^2+3\|\nabla f_i(\bar{x}_t)-q_i(\bar{x}_t)\|^2\\
        &+3\|q(\bar{x}_t)-\nabla f (\bar{x}_t)\|^2\\
        & \leq 3\|q_i(x_{i,\ell})-\nabla f_i(x_{i,\ell})\|^2+3\|\nabla f_i(\bar{x}_t)-q_i(\bar{x}_t)\|^2\\
        &+\dfrac{3}{m}\sum\limits_{i=1}^m\|q_i(\bar{x}_t)-\nabla f_i (\bar{x}_t)\|^2.
    \end{aligned}
\end{equation}
Taking expectation on both sides of Eq. \eqref{lemma13:eqn4} and using the bounded-variance property,  we obtain:
\begin{equation}
\label{lemma13:eqn6}
    \begin{aligned}
       \mathbb{E}\big[\|{\mathcal{W}}\|^2\big]& \leq 3(3\sigma^2)=9\sigma^2.
    \end{aligned}
\end{equation}

Combining equations \eqref{lemma13:eqn5} and \eqref{lemma13:eqn6}, equation \eqref{lemma13:eqn3} becomes:
\begin{equation}
\label{lemma13:eqn7}
    \begin{aligned}
        \mathbb{E}\big[\|x_{i,\ell+1}-\bar{x}_t\|^2\big]&\leq (1+\dfrac{1}{\gamma})\mathbb{E}\Big[\|x_{i,\ell}-\bar{x}_t\|^2\Big]+(1+\gamma)\eta_i^2\mathbb{E}\Big[\|\nabla f(\bar{x}_t)\|^2\Big]+9\eta_i^2\sigma^2.
    \end{aligned}
\end{equation}
Iterating equation \eqref{lemma13:eqn7} and using $x_{i,0}=\bar{x}_t$, we obtain:
\begin{equation}
\label{lemma13:eqn8}
    \begin{aligned}
        \mathbb{E}\big[\|x_{i,\ell}-\bar{x}_t\|^2\big]&\leq \Bigg[(1+\gamma)\eta_i^2\mathbb{E}\Big[\|\nabla f(\bar{x}_t)\|^2\Big]+9\eta_i^2\sigma^2\Bigg]\Bigg(\sum \limits_{j=0}^{\ell-1}(1+\dfrac{1}{\gamma})^j\Bigg)\\
        &\leq 6 \bar{\eta}^2\mathbb{E}\Big[\|\nabla f(\bar{x}_t)\|^2\Big]+27\eta_i\bar{\eta}\sigma^2,
    \end{aligned}
\end{equation}
where we set $\gamma=\tau_i$, and used $\eta_i \tau_i =\bar{\eta}$.

\textbf{Completing the proof of Theorem \ref{thm:noisy}:} By substituting the bound on the drift from Lemma \ref{thm4:lemma2} in equation \eqref{thm4:driftbound}, and for $12\bar{\eta}L\leq 1$, we obtain: 
\begin{equation}
\begin{aligned}
    \mathbb{E}\Big[\|\bar{x}_{t+1}-x^*\|^2\Big] &\leq \mathbb{E}\Big[\|\bar{x}_{t}-x^*\|^2\Big] -2\bar{\eta} \mathbb{E}\Big[f(\bar{x}_{t})-f(x^*)\Big] + 4 \bar{\eta}^2\mathbb{E}\Big[\|\nabla f(\bar{x}_t)\|^2\Big] \\
    & +2L\Big(6\bar{\eta}^3\mathbb{E}\Big[\|\nabla f(\bar{x}_t)\|^2\Big]+27\bar{\eta}^3\sigma^2\Big)+16\bar{\eta}^2 \sigma^2\\
    &\leq \mathbb{E}\Big[\|\bar{x}_{t}-x^*\|^2\Big] -2\bar{\eta}\mathbb{E}\Big[f(\bar{x}_{t})-f(x^*)\Big] + 6 \bar{\eta}^2\mathbb{E}\Big[\|\nabla f(\bar{x}_t)\|^2\Big]+25\bar{\eta}^2 \sigma^2\\
    &\overset{(a)}\leq \mathbb{E}\Big[\|\bar{x}_{t}-x^*\|^2\Big] -\mathbb{E}\Big[2\bar{\eta}(1-6\bar{\eta}L)(f(\bar{x}_{t})-f(x^*))\Big] + 25\bar{\eta}^2 \sigma^2\\
    & \overset{(b)} \leq (1-\frac{\bar{\eta}\mu}{2})\mathbb{E}\Big[\|\bar{x}_{t}-x^*\|^2\Big]+25\bar{\eta}^2 \sigma^2.
\end{aligned}
\end{equation}
In the above steps, $(a)$ and $(b)$ follow from the $L$-smoothness and $\mu$-strong convexity of the loss functions, respectively. This completes the proof of Theorem \ref{thm:noisy}. 
\end{proof}
\newpage
\section{Proof of Theorem \ref{thm:server_sprs1}: Gradient Sparsification at Server with no Error-Feedback}
\label{app:servsparse1}
In our subsequent analysis, we will make use of three basic properties of a \texttt{TOP-k} operator that are summarized in the following lemma.
\begin{lemma}
Let $\mathcal{C}_{\delta}:\mathbb{R}^d \rightarrow\mathbb{R}^d$ denote the \texttt{TOP-k}  operator, where $\delta=d/k$, and $k\in\{1,\ldots,d\}$. Then, given any vector $x\in\mathbb{R}^d$, the following three properties hold.
\begin{itemize} 
\item \textbf{Property 1}: $\langle \mathcal{C}_{\delta}(x), x\rangle = {\Vert \mathcal{C}_{\delta}(x) \Vert}^2$.
\item \textbf{Property 2}: ${\Vert \mathcal{C}_{\delta}(x) \Vert}^2 \geq \frac{1}{\delta} {\Vert x \Vert}^2$.
\item \textbf{Property 3}: ${\Vert x- \mathcal{C}_{\delta}(x) \Vert}^2 \leq \left(1-\frac{1}{\delta}\right) {\Vert x \Vert}^2$. 
\end{itemize}
\label{lemma:topK}
\end{lemma}

All three properties stated above follow almost directly from the definition of the \texttt{TOP-k} operator. For a formal proof, see \cite{beznosikov}. We start with the following lemma. 

\begin{lemma}
Suppose each $f_i(x)$ is $L$-smooth. Moreover, suppose $\tau_i \geq 1, \forall i\in\mathcal{S}$,  $\delta_c=1$, and $\eta_i=\frac{\bar{\eta}}{\tau_i}, \forall i \in \mathcal{S}$, where $\bar{\eta}\in (0,1)$. Then, the variant of \texttt{FedLin} described in the statement of Theorem \ref{thm:server_sprs1} guarantees:
\begin{equation}
\begin{aligned}
    f(\bar{x}_{t+1})-f(\bar{x}_t) &\leq -\bar{\eta} \left(1- \bar{\eta} L \right){\Vert g_t \Vert}^2 + \left(\frac{L}{m} \sum\limits_{i=1}^{m} \eta_i \sum\limits_{\ell=0}^{\tau_i-1}\Vert x_{i,\ell} - \bar{x}_t \Vert \right) \Vert \nabla f(\bar{x}_t) \Vert\\ 
    & \hspace{5mm} + \frac{\bar{\eta} L^3 }{m} \sum\limits_{i=1}^{m} \eta_i  \sum\limits_{\ell=0}^{\tau_i-1}{\Vert x_{i,\ell} - \bar{x}_t \Vert}^2.
\end{aligned}
\label{eqn:f_bound_serv1}
\end{equation}
\label{lemma:f_bound_serv1}
\end{lemma}
\begin{proof}
For the setting under consideration, the local update rule at client $i$ takes the form
\begin{equation}
    x_{i,\ell+1} = x_{i,\ell}-\eta_i(\nabla f_i(x_{i,\ell})-\nabla f_i(\bar{x}_{t})+g_t),
\label{eqn:local_update_serv1}
\end{equation}
where
\begin{equation}
g_t=\mathcal{C}_{\delta_s}\left(\frac{1}{m}\sum\limits_{i=1}^{m}\nabla f_i (\bar{x}_t)\right)=\mathcal{C}_{\delta_s}\left(\nabla f (\bar{x}_t)\right).
\end{equation}
Using $x_{i,0}=\bar{x}_t, \forall i\in\mathcal{S}$, we then have:
\begin{equation}
\begin{aligned}
    x_{i,\tau_i}&=\bar{x}_t-\eta_i\sum\limits_{\ell=0}^{\tau_i-1}\nabla f_i(x_{i,\ell}) -\eta_i \tau_i (g_t - \nabla f_i(\bar{x}_t))\\
    &=\bar{x}_t-\eta_i\sum\limits_{\ell=0}^{\tau_i-1}\nabla f_i(x_{i,\ell}) - \bar{\eta} (g_t - \nabla f_i(\bar{x}_t)), \forall i\in\mathcal{S},
\end{aligned}
\end{equation}
where we used $\eta_i \tau_i =\bar{\eta}$ in the second step. Thus,
\begin{equation}
\begin{aligned}
    \bar{x}_{t+1}=\frac{1}{m}\sum\limits_{i=1}^{m}x_{i,\tau_i}&=\bar{x}_t-\frac{1}{m} \sum\limits_{i=1}^{m} \eta_i \sum\limits_{\ell=0}^{\tau_i-1} \nabla f_i(x_{i,\ell}) -\frac{\bar{\eta}}{m} \sum\limits_{i=1}^{m} \left(g_t - \nabla f_i(\bar{x}_t)\right)\\  
    &=\bar{x}_t-\frac{1}{m} \sum\limits_{i=1}^{m} \eta_i \sum\limits_{\ell=0}^{\tau_i-1} \nabla f_i(x_{i,\ell}) - \bar{\eta} \left(g_t - \nabla f(\bar{x}_t)\right).
\label{eqn:xbart_serv1}
\end{aligned}
\end{equation}
Compared to \eqref{eqn:xbart_ct}, note that we have an additional error term $\bar{\eta}(g_t-\nabla f(\bar{x}_t))$ that shows up as a consequence of gradient sparsification at the server.  Nonetheless, we proceed exactly as before, and bound $f(\bar{x}_{t+1})-f(\bar{x}_t)$ as follows:
\begin{equation}
    \begin{aligned}
    f(\bar{x}_{t+1})-f(\bar{x}_{t}) & \leq \langle \bar{x}_{t+1} - \bar{x}_{t}, \nabla f(\bar{x}_t) \rangle + \frac{L}{2} {\Vert \bar{x}_{t+1}-\bar{x}_t \Vert}^2\\
    &=  \Big \langle - \frac{1}{m} \sum\limits_{i=1}^{m} \eta_i \sum\limits_{\ell=0}^{\tau_i-1}  \nabla f_i(x_{i,\ell}) + \bar{\eta}(\nabla f(\bar{x}_t) -g_t), \nabla f(\bar{x}_{t}) \Big \rangle \\
    & \hspace{5mm} + \frac{L}{2} \norm[\bigg]{\frac{1}{m}\sum\limits_{i=1}^{m} \eta_i \sum\limits_{\ell=0}^{\tau_i-1} \nabla f_i(x_{i,\ell})+\bar{\eta} \left(g_t - \nabla f(\bar{x}_t)\right)}^2\\
    & \overset{(a)} =  - \Big \langle \frac{1}{m} \sum\limits_{i=1}^{m} \eta_i \sum\limits_{\ell=0}^{\tau_i-1}  (\nabla f_i(x_{i,\ell})-\nabla f_i(\bar{x}_t)), \nabla f(\bar{x}_{t}) \Big \rangle -  \bar{\eta} \big \langle g_t, \nabla f(\bar{x}_{t}) \big \rangle \\
    & \hspace{5mm} + \frac{L}{2} \norm[\bigg]{\frac{1}{m}\sum\limits_{i=1}^{m} \eta_i \sum\limits_{\ell=0}^{\tau_i-1} (\nabla f_i(x_{i,\ell})-\nabla f_i(\bar{x}_t))+\bar{\eta} g_t }^2\\
   & \overset{(b)} =  - \Big \langle \frac{1}{m} \sum\limits_{i=1}^{m} \eta_i \sum\limits_{\ell=0}^{\tau_i-1}  (\nabla f_i(x_{i,\ell})-\nabla f_i(\bar{x}_t)), \nabla f(\bar{x}_{t}) \Big \rangle -  \bar{\eta} {\Vert g_t \Vert}^2 \\
    & \hspace{5mm} + \frac{L}{2} \norm[\bigg]{\frac{1}{m}\sum\limits_{i=1}^{m} \eta_i \sum\limits_{\ell=0}^{\tau_i-1} (\nabla f_i(x_{i,\ell})-\nabla f_i(\bar{x}_t))+\bar{\eta} g_t }^2\\
& \overset{(c)} \leq -\bar{\eta}\left(1-\bar{\eta}L\right) {\Vert g_t \Vert}^2 \underbrace{- \Big \langle \frac{1}{m} \sum\limits_{i=1}^{m} \eta_i \sum\limits_{\ell=0}^{\tau_i-1}  (\nabla f_i(x_{i,\ell})-\nabla f_i(\bar{x}_t)), \nabla f(\bar{x}_{t}) \Big \rangle}_{T_1}\\
& \hspace{5mm} + \underbrace{L \norm[\bigg]{\frac{1}{m}\sum\limits_{i=1}^{m} \eta_i \sum\limits_{\ell=0}^{\tau_i-1} (\nabla f_i(x_{i,\ell})-\nabla f_i(\bar{x}_t))}^2}_{T_2}\\
& \overset{(d)} \leq -\bar{\eta} \left(1- \bar{\eta} L \right){\Vert g_t \Vert}^2 + \left(\frac{L}{m} \sum\limits_{i=1}^{m} \eta_i \sum\limits_{\ell=0}^{\tau_i-1}\Vert x_{i,\ell} - \bar{x}_t \Vert \right) \Vert \nabla f(\bar{x}_t) \Vert\\ 
    & \hspace{5mm} + \frac{\bar{\eta} L^3 }{m} \sum\limits_{i=1}^{m} \eta_i  \sum\limits_{\ell=0}^{\tau_i-1}{\Vert x_{i,\ell} - \bar{x}_t \Vert}^2.
    \end{aligned}
\label{eqn:iterimbnd1_serv1}
\end{equation}
In the above steps, for arriving at (a), we made the following observation:
\begin{equation}
    \bar{\eta} \nabla f(\bar{x}_t)= \frac{1}{m} \sum\limits_{i=1}^{m} \bar{\eta} \nabla f_i(\bar{x}_t) = \frac{1}{m} \sum\limits_{i=1}^{m} \eta_i \tau_i \nabla f_i(\bar{x}_t) = \frac{1}{m} \sum\limits_{i=1}^{m} \eta_i \sum\limits_{\ell=0}^{\tau_i-1} \nabla f_i(\bar{x}_t).
\end{equation}
For (b), observe that $\big \langle g_t, \nabla f(\bar{x}_{t}) \big \rangle = \big \langle \mathcal{C}_{\delta_s}\left(\nabla  f(\bar{x}_{t})\right), \nabla f(\bar{x}_{t})  \big \rangle = {\Vert g_t \Vert}^2$, where the second equality follows from Property 1 of the \texttt{TOP-k} operator in Lemma \ref{lemma:topK}. For (c), we used \eqref{eqn:rel_triangle} with $\gamma=1$. For (d), we followed the arguments used to arrive at \eqref{eqn:T1_thm_ct} and \eqref{eqn:T2bound_ct} to bound $T_1$ and $T_2$, respectively. 
\end{proof}
\newpage
\textbf{Completing the proof of Theorem \ref{thm:server_sprs1}:} To complete the proof of Theorem \ref{thm:server_sprs1}, we start by noting that if the step-size at client $i$ satisfies $\eta_i \leq \frac{1}{L}$, then arguments identical to those used for proving Lemma \ref{lemma:drift_conv_ct} can be used to conclude that
\begin{equation}
    \Vert x_{i,\ell} - \bar{x}_t \Vert \leq \eta_i \tau_i \Vert g_t \Vert.
\end{equation}
Substituting the above bound in \eqref{eqn:f_bound_serv1} yields:
\begin{equation}
    \begin{aligned}
f(\bar{x}_{t+1})-f(\bar{x}_t) &\leq -\bar{\eta} \left(1- \bar{\eta} L \right){\Vert g_t \Vert}^2 + \left(\frac{L}{m} \sum\limits_{i=1}^{m} \eta_i \sum\limits_{\ell=0}^{\tau_i-1} \eta_i \tau_i \Vert g_t \Vert \right) \Vert \nabla f(\bar{x}_t) \Vert\\ 
    & \hspace{5mm} + \frac{\bar{\eta} L^3 }{m} \sum\limits_{i=1}^{m} \eta_i  \sum\limits_{\ell=0}^{\tau_i-1}{\left (\eta_i \tau_i \Vert g_t \Vert  \right)}^2 \\
& = -\bar{\eta} \left(1- \bar{\eta} L \right){\Vert g_t \Vert}^2 + \frac{L}{m} \sum\limits_{i=1}^{m} {\left(\eta_i \tau_i\right)}^2 \Vert g_t \Vert  \Vert \nabla f(\bar{x}_t) \Vert + \frac{\bar{\eta} L^3 }{m} \sum\limits_{i=1}^{m} {\left (\eta_i \tau_i \right)}^3 {\Vert g_t \Vert }^2\\
& \overset{(a)} \leq -\bar{\eta} \left(1- \bar{\eta} L \right){\Vert g_t \Vert}^2 + \frac{\sqrt{\delta_s} L}{m} \sum\limits_{i=1}^{m} {\left(\eta_i \tau_i\right)}^2 {\Vert g_t \Vert }^2 + \frac{\bar{\eta} L^3 }{m} \sum\limits_{i=1}^{m} {\left (\eta_i \tau_i \right)}^3 {\Vert g_t \Vert }^2\\
& \overset{(b)} \leq -\bar{\eta} \left(1- \left(2
+\sqrt{\delta_s}\right)\bar{\eta} L \right){\Vert g_t \Vert}^2\\
& \overset{(c)} \leq -\frac{\bar{\eta}}{\delta_s} \left(1- \left(2
+\sqrt{\delta_s}\right)\bar{\eta} L \right){\Vert \nabla f(\bar{x}_t) 
\Vert}^2\\
& \overset{(d)} \leq -\frac{2\bar{\eta} \mu}{\delta_s} \left(1- \left(2
+\sqrt{\delta_s}\right)\bar{\eta} L \right)\left(f(\bar{x}_t)-f(x^*)\right). 
    \end{aligned}
\end{equation}
In the above steps, for (a), we used Property 2 of the \texttt{TOP-k} operator in Lemma \ref{lemma:topK} to conclude that $\Vert \nabla f(\bar{x}_t) \Vert \leq \sqrt{\delta_s} \Vert \mathcal{C}_{\delta_s} (\nabla f(\bar{x}_t) ) \Vert = \sqrt{\delta_s} \Vert g_t \Vert$. For (b), we used $\eta_i \tau_i = \bar{\eta}$ and $\bar{\eta} L \leq 1$. For (c), we once again used the second property of the \texttt{TOP-k} operator, and for (d), we used the fact that $f(\cdot)$ is $\mu$-strongly convex (refer to \eqref{eqn:grad_low}). Setting $\bar{\eta}=\frac{1}{2\left(2
+\sqrt{\delta_s}\right)L}$ and rearranging terms then leads to
\begin{equation}
    f(\bar{x}_{t+1})-f(x^*) \leq {\left(1-\frac{1}{2\delta_s \left(2+\sqrt{\delta_s}\right)\kappa}\right)} (f(\bar{x}_{t})-f(x^*)), \textrm{where} \hspace{1mm}  \kappa=\frac{L}{\mu}.
\end{equation}
Using the above inequality recursively, we obtain the desired conclusion.  
\newpage
\section{Proof of Theorem \ref{thm:client_sprs}: Gradient Sparsification at Clients} 
\label{app:clientsparse}
The proof of Theorem \ref{thm:client_sprs} is somewhat more involved than Theorem \ref{thm:server_sprs1}. Let us begin by compiling the governing equations for the setting under consideration. 
\begin{equation}
    \begin{aligned}
    x_{i,\ell+1}&=x_{i,\ell}-\eta_i(\nabla f_i(x_{i,\ell}) - \nabla f_i(\bar{x}_t) + g_t)\\
    h_{i,t}&=\mathcal{C}_{\delta_c}\left(\rho_{i,t-1}+\nabla f_i(\bar{x}_t) \right) \\
    \rho_{i,t}&=\rho_{i,t-1}+\nabla f_i(\bar{x}_t) - h_{i,t}.
    \end{aligned}
\label{eqn:clientsprs_eqs}
\end{equation}
The first and the third equations  hold for every communication round $t\in\{1,\ldots,T\}$, whereas the second equation holds  $\forall t \in \{2,\ldots,T\}$. Moreover, we have $h_{i,1}=\nabla f_i(\bar{x}_1)$, and  $\rho_{i,0}=\rho_{i,1}=0, \forall i\in\mathcal{S}$, i.e., the initial gradient  compression errors are $0$ at each client. Since there is no further gradient sparsification at the server, we have $g_t=\frac{1}{m}\sum_{i=1}^{m}h_{i,t}$. It then follows that
\begin{equation}
    \rho_{t}=\rho_{t-1}+\nabla f(\bar{x}_t) - g_t,
\label{eqn:error_eqn}
\end{equation}
where $\rho_t\triangleq \frac{1}{m}   \sum_{i=1}^{m}\rho_{i,t}$. To simplify the analysis, let us define a virtual sequence $\{\tilde x_t\}$ as follows:\footnote{We note that virtual sequences and perturbed iterates are commonly used to simplify proofs in the context of analyzing compression schemes \cite{reddystich,beznosikov}, and asynchronous methods \cite{mania}.}
\begin{equation}
    \tilde{x}_t \triangleq \bar{x}_t - \bar{\eta} \rho_{t-1},
\label{eqn:tildex}
\end{equation}
where $\bar{\eta}= \eta_i \tau_i, \forall i\in\mathcal{S}$. Now observe that
\begin{equation}
    \begin{aligned}
  \tilde{x}_{t+1}&=\bar{x}_{t+1}-\bar{\eta} \rho_t\\
  &= \bar{x}_t-\frac{1}{m} \sum\limits_{i=1}^{m} \eta_i \sum \limits_{\ell=0}^{\tau_i-1} \nabla f_i(x_{i,\ell}) - \bar{\eta} \left(g_t - \nabla f(\bar{x}_t)\right) - \bar{\eta} \left( \rho_{t-1}+\nabla f(\bar{x}_t) - g_t \right)\\
&= \tilde{x}_t - \frac{1}{m} \sum\limits_{i=1}^{m} \eta_i \sum \limits_{\ell=0}^{\tau_i-1} \nabla f_i(x_{i,\ell}). 
    \end{aligned}
\label{eqn:recursion_xtilde}
\end{equation}
The second equality follows from \eqref{eqn:xbart_serv1} and \eqref{eqn:error_eqn}, and the third follows from the definition of $\tilde{x}_t$ in  \eqref{eqn:tildex}. Interestingly, note that the recursion for $\tilde{x}_t$ that we just derived in \eqref{eqn:recursion_xtilde} resembles that for $\bar{x}_t$ in \eqref{eqn:xbart_ct} where there was no effect of gradient sparsification. This simplified recursion reveals the utility of the virtual sequence. 

\textbf{Proof idea:} In order to argue that $\bar{x}_t$ converges to $x^*$, it clearly suffices to argue that the virtual sequence $\tilde{x}_t$ converges to $x^*$, and $\bar{x}_t$ converges to $\tilde{x}_t$. To achieve this, we will employ the following Lyapunov function in our analysis:
\begin{equation}
    \psi_{t} \triangleq { \Vert \tilde{x}_t - x^* \Vert }^2 + \bar{\eta}^2 V_{t-1}, \textrm{where} \hspace{2mm} V_t \triangleq \frac{1}{m} \sum_{i=1}^{m} {\Vert \rho_{i,t} \Vert}^2.
\label{eqn:Lyap}
\end{equation}
The choice of the above Lyapunov function is specific to our setting, and accounts for the effects of systems heterogeneity and gradient sparsification. In the following lemma, we bound the first part of the Lyapunov function, namely the distance of the virtual iterate $\tilde{x}_t$ from the optimal point $x^*$. 

\begin{lemma} Suppose each $f_i(x)$ is $L$-smooth and $\mu$-strongly convex, and suppose Assumption \ref{ass:bndgrad} holds.  Moreover, suppose $\tau_i\geq 1,\forall i\in\mathcal{S}$, and $\delta_s=1$. Let the step-size for client $i$ be chosen as $\eta_i=\frac{\bar{\eta}}{\tau_i}$, where $\bar{\eta}\in (0,1)$ satisfies $\bar{\eta}\leq\frac{1}{2LC}$. Then, \texttt{FedLin} guarantees:
\begin{equation}
    {\Vert \tilde{x}_{t+1} - x^* \Vert}^2 \leq  {\Vert \tilde{x}_{t} - x^* \Vert}^2 - 2 \bar{\eta} (f(\tilde{x}_t)-f(x^*)) + 14 \bar{\eta}^2 {\Vert \nabla f(\tilde{x}_t) \Vert}^2 + 24  \bar{\eta}^3 L V_{t-1} + 12 \bar{\eta}^3 L D. 
\end{equation}
\label{lemma:clientsprs_bnd}
\end{lemma}
\begin{proof}
From \eqref{eqn:recursion_xtilde}, we obtain
\begin{equation}
    {\Vert \tilde{x}_{t+1}-x^* \Vert}^2 = {\Vert \tilde{x}_{t}-x^* \Vert}^2 - \frac{2}{m} \langle \tilde{x}_t - x^*, \sum\limits_{i=1}^{m} \eta_i \sum\limits_{\ell=0}^{\tau_i-1} \nabla f_i (x_{i,\ell}) \rangle + \norm[\bigg]{\frac{1}{m}\sum\limits_{i=1}^{m} \eta_i \sum\limits_{\ell=0}^{\tau_i-1} \nabla f_i (x_{i,\ell})}^2.
\end{equation}
To bound the second and third terms in the above equation, we can follow exactly the same steps as those in Lemma \ref{lemma:iterate_subopt_ct}. In particular, Lemma \ref{lemma:iterate_subopt_ct} relied on $L$-smoothness and convexity of each $f_i(\cdot)$ - each of these properties apply to our current setting. Referring to \eqref{eqn:iterate_subopt_ct}, we thus have
\begin{equation}
{\Vert \tilde{x}_{t+1}-x^* \Vert}^2  \leq {\Vert \tilde{x}_{t}-x^* \Vert}^2 -2 \bar{\eta} \left(f(\tilde{x}_{t})-f(x^*) \right) + \frac{L(1
+2\bar{\eta} L)}{m} \sum\limits_{i=1}^{m} \eta_i \sum\limits_{\ell=0}^{\tau_i-1}{\Vert x_{i,\ell} - \tilde{x}_t \Vert}^2 +  2{\bar{\eta}}^2 {\Vert \nabla f(\tilde{x}_t) \Vert}^2.
\label{eqn:iteratebnd_clientsprs}
\end{equation}
To bound the term ${ \Vert \tilde{x}_t - x_{i,\ell} \Vert}^2$, start by observing that
\begin{equation}
    \begin{aligned}
{ \Vert \tilde{x}_t - x_{i,\ell} \Vert}^2 &= { \Vert \tilde{x}_t - \bar{x}_t + \bar{x}_t - x_{i,\ell} \Vert}^2 \\
& \leq 2 { \Vert \tilde{x}_t - \bar{x}_t \Vert}^2 + 2 { \Vert \bar{x}_t - x_{i,\ell} \Vert}^2\\
& = 2 \bar{\eta}^2 {\Vert \rho_{t-1} \Vert}^2 + 2 { \Vert \bar{x}_t - x_{i,\ell} \Vert}^2,
    \end{aligned}
\end{equation}
where the last equality follows from \eqref{eqn:tildex}. Since $\bar{\eta} \leq \frac{1}{L}$, following the same arguments as in Lemma \ref{lemma:drift_conv_ct}, we have $\Vert \bar{x}_t - x_{i,\ell} \Vert \leq \eta_i \tau_i \Vert g_t \Vert = \bar{\eta} \Vert g_t \Vert$. Thus,
\begin{equation}
    { \Vert \tilde{x}_t - x_{i,\ell} \Vert}^2 \leq 2\bar{\eta}^2 \left( {\Vert \rho_{t-1} \Vert}^2 + {\Vert g_t \Vert}^2 \right).
\label{eqn:clientsprsbnd1}
\end{equation}
Next, note that
\begin{equation}
    \begin{aligned}
    {\Vert g_t \Vert}^2 &=  \norm[\bigg]{\frac{1}{m} \sum_{i=1}^{m} h_{i,t}}^2 \\
    & \overset{(a)} \leq  \frac{1}{m} \sum_{i=1}^{m}{\Vert h_{i,t} \Vert}^2\\
    & \overset{(b)} \leq  \frac{1}{m} \sum_{i=1}^{m}{\Vert \rho_{i,t-1}+\nabla f_i(\bar{x}_t) \Vert}^2 \\
    & \overset{(c)} \leq \frac{2}{m} \sum_{i=1}^{m}{\Vert \rho_{i,t-1} \Vert}^2+
    \frac{2}{m} \sum_{i=1}^{m} {\Vert \nabla f_i(\bar{x}_t) \Vert}^2 \\
    & \overset{(d)} = 2 V_{t-1} + \frac{2}{m} \sum_{i=1}^{m} {\Vert \nabla f_i(\bar{x}_t) \Vert}^2 \\
    & \overset{(e)} \leq 2 V_{t-1} + 2C {\Vert \nabla f(\bar{x}_t) \Vert}^2 + 2D.
    \end{aligned}
\label{eqn:clientsprsbnd2}
\end{equation}
In the above steps, (a) follows from Jensen's inequality; (b) follows from the fact that $h_{i,t}=\mathcal{C}_{\delta_c}\left(\rho_{i,t-1}+\nabla f_i(\bar{x}_t) \right)$, and the definition of the \texttt{TOP-k} operation; (c) follows from \eqref{eqn:rel_triangle} with $\gamma=1$; (d) follows from the definition of $V_{t-1}$, and (e) follows from Assumption \ref{ass:bndgrad}. Finally, observe that
\begin{equation}
\begin{aligned}
    {\Vert \nabla f(\bar{x}_t) \Vert}^2 &= {\Vert \nabla f(\bar{x}_t) - \nabla f(\tilde{x}_t) + \nabla f(\tilde{x}_t) \Vert}^2\\ 
& \leq 2 {\Vert \nabla f(\bar{x}_t) - \nabla f(\tilde{x}_t)\Vert}^2 + 2 {\Vert \nabla f(\tilde{x}_t) \Vert}^2\\
& \overset{(a)} \leq 2L^2 {\Vert \bar{x}_t - \tilde{x}_t \Vert}^2 + 2 {\Vert \nabla f(\tilde{x}_t) \Vert}^2\\
& \overset{(b)} = 2\bar{\eta}^2 L^2 {\Vert \rho_{t-1} \Vert}^2 + 2 {\Vert \nabla f(\tilde{x}_t) \Vert}^2\\
& \overset{(c)} \leq 2 \bar{\eta}^2 L^2 V_{t-1} + 2 {\Vert \nabla f(\tilde{x}_t) \Vert}^2, 
\end{aligned}
\label{eqn:clientsprsbnd3}
\end{equation}
where for (a), we used the $L$-smoothness of $f(\cdot)$; for (b), we used \eqref{eqn:tildex}, and for (c), we used Jensen's inequality to bound ${\Vert \rho_{t-1} \Vert}^2$ by $V_{t-1}$. Combining the bounds \eqref{eqn:clientsprsbnd1}, \eqref{eqn:clientsprsbnd2} and \eqref{eqn:clientsprsbnd3}, we obtain:
\begin{equation}
    \begin{aligned}
{\Vert \tilde{x}_t - x_{i,\ell} \Vert}^2 & \leq 2 \bar{\eta}^2 \left( {\Vert \rho_{t-1} \Vert}^2 + 2 V_{t-1} + 2C {\Vert \nabla f(\bar{x}_t) \Vert}^2 + 2D \right)\\
& \leq 2 \bar{\eta}^2 \left( (3+4\bar{\eta}^2 L^2 C) V_{t-1} + 4C {\Vert \nabla f(\tilde{x}_t) \Vert}^2 + 2D \right)\\
& \leq 4 \bar{\eta}^2 \left( 2 V_{t-1} + 2C {\Vert \nabla f(\tilde{x}_t) \Vert}^2 + D \right).
    \end{aligned}
\end{equation}
For the second inequality, we once again used Jensen's to conclude ${\Vert \rho_{t-1} \Vert}^2 \leq V_{t-1}$; for the last inequality, we used $\bar{\eta}^2 L^2 C \leq \bar{\eta}^2 L^2 C^2 \leq \frac{1}{4}$, which in turn follows from $C\geq 1$, and the fact that $\bar{\eta} \leq \frac{1}{2LC}$ based on our choice of step-size. Plugging the bound on ${\Vert \tilde{x}_t - x_{i,\ell} \Vert}^2$ in \eqref{eqn:iteratebnd_clientsprs}, we have
\begin{equation}
\begin{aligned}
{\Vert \tilde{x}_{t+1}-x^* \Vert}^2  &\leq {\Vert \tilde{x}_{t}-x^* \Vert}^2 -2 \bar{\eta} \left(f(\tilde{x}_{t})-f(x^*) \right) +  2{\bar{\eta}}^2 {\Vert \nabla f(\tilde{x}_t) \Vert}^2\\
& \hspace{5mm} + \frac{L(1
+2\bar{\eta} L)}{m} \sum\limits_{i=1}^{m} \eta_i \sum\limits_{\ell=0}^{\tau_i-1} 4\bar{\eta}^2 \left( 2 V_{t-1} + 2C {\Vert \nabla f(\tilde{x}_t) \Vert}^2 + D \right) \\
& \leq {\Vert \tilde{x}_{t}-x^* \Vert}^2 -2 \bar{\eta} \left(f(\tilde{x}_{t})-f(x^*) \right) +  2{\bar{\eta}}^2 {\Vert \nabla f(\tilde{x}_t) \Vert}^2\\
& \hspace{5mm} + \frac{3L}{m} \sum\limits_{i=1}^{m}  4\bar{\eta}^3\left( 2 V_{t-1} + 2C {\Vert \nabla f(\tilde{x}_t) \Vert}^2 + D \right) \\
& = {\Vert \tilde{x}_{t}-x^* \Vert}^2 -2 \bar{\eta} \left(f(\tilde{x}_{t})-f(x^*) \right) + 2\bar{\eta}^2 \left(1+12\bar{\eta}LC\right) {\Vert \nabla f(\tilde{x}_t) \Vert}^2 + 24 \bar{\eta}^3 L V_{t-1} + 12 \bar{\eta}^3 L D\\
& \leq {\Vert \tilde{x}_{t}-x^* \Vert}^2 -2 \bar{\eta} \left(f(\tilde{x}_{t})-f(x^*) \right) + 14\bar{\eta}^2 {\Vert \nabla f(\tilde{x}_t) \Vert}^2 + 24 \bar{\eta}^3 L V_{t-1} + 12 \bar{\eta}^3 L D.
\end{aligned}
\end{equation}
For the second inequality, we used $\bar{\eta}L \leq 1$ and $\bar{\eta}=\eta_i\tau_i$, and for the last inequality, we used $\bar{\eta}LC \leq \frac{1}{2}$. This concludes the proof. 
\end{proof}
In the next lemma, we derive a recursion to bound $V_{t}$ - a measure of the sparsification error. 

\begin{lemma}
Suppose the conditions stated in Lemma \ref{lemma:clientsprs_bnd} hold. Then, we have
\begin{equation}
    V_t \leq \left(1-\frac{1}{2\delta_c}+4\bar{\eta}^2L^2\delta_c C\right)V_{t-1} + 4 \delta_c C {\Vert \nabla f(\tilde{x}_t) \Vert}^2 + 2\delta_c D.
\end{equation}
\label{lemma:sparse_err}
\end{lemma}
\begin{proof}
Let us observe that
\begin{equation}
    \begin{aligned}
    V_t &= \frac{1}{m} \sum_{i=1}^{m} {\Vert \rho_{i,t} \Vert}^2\\
    &= \frac{1}{m} \sum_{i=1}^{m} {\Vert \rho_{i,t-1}+\nabla f_i(\bar{x}_t) - h_{i,t} \Vert}^2\\
    &= \frac{1}{m} \sum_{i=1}^{m} {\Vert \rho_{i,t-1}+\nabla f_i(\bar{x}_t) - \mathcal{C}_{\delta_c} \left(\rho_{i,t-1}+\nabla f_i(\bar{x}_t)\right) \Vert}^2\\
    & \leq  \left(1-\frac{1}{\delta_c}\right) \frac{1}{m} \sum_{i=1}^{m} {\Vert \rho_{i,t-1}+\nabla f_i(\bar{x}_t) \Vert}^2\\
    & \leq \left(1-\frac{1}{\delta_c}\right) (1+\gamma) V_{t-1} + \left(1-\frac{1}{\delta_c}\right)\left(1+\frac{1}{\gamma}\right) \frac{1}{m} \sum_{i=1}^{m} {\Vert \nabla f_i(\bar{x}_t) \Vert}^2.
    \end{aligned}
\label{eqn:sparse_err_bnd}
\end{equation}
For the second-last inequality, we used Property 3 of the \texttt{TOP-k} operator in Lemma \ref{lemma:topK}; for the last inequality, we used the definition of $V_{t-1}$ and the relaxed triangle inequality \eqref{eqn:rel_triangle}. Now in order for $V_t$ to contract over time, we must have $$\left(1-\frac{1}{\delta_c}\right) (1+\gamma) < 1 \implies \gamma < \frac{1}{\delta_c -1}.$$
Accordingly, suppose $\gamma=\frac{1}{2(\delta_c-1)}$. Simple calculations then yield 
$$\left(1-\frac{1}{\delta_c}\right) (1+\gamma) = 1-\frac{1}{2\delta_c}; \hspace{5mm} \left(1-\frac{1}{\delta_c}\right) \left(1+\frac{1}{\gamma}\right) =  \left(1-\frac{1}{\delta_c}\right) (2\delta_c-1) < 2\delta_c. $$
Substituting the above bounds in \eqref{eqn:sparse_err_bnd}, and invoking Assumption \ref{ass:bndgrad}, we obtain
\begin{equation}
    \begin{aligned}
    V_t &\leq \left(1-\frac{1}{2\delta_c}\right)V_{t-1} + 2\delta_c \left(C{\Vert \nabla f(\bar{x}_t) \Vert}^2+D\right)\\
    & \leq \left(1-\frac{1}{2\delta_c}\right)V_{t-1} + 2\delta_c C \left(2 \bar{\eta}^2 L^2 V_{t-1} + 2 {\Vert \nabla f(\tilde{x}_t) \Vert}^2 \right) + 2\delta_c D \\
    & = \left(1-\frac{1}{2\delta_c}+4\bar{\eta}^2L^2\delta_c C\right)V_{t-1} + 4 \delta_c C {\Vert \nabla f(\tilde{x}_t) \Vert}^2 + 2\delta_c D,
    \end{aligned}
\end{equation}
where for the second inequality, we used \eqref{eqn:clientsprsbnd3}. 
\end{proof}
Now that we have a handle over each of the two components of the Lyapunov function $\psi_{t+1}$, we are in a position to complete the proof of Theorem \ref{thm:client_sprs}. 

\textbf{Completing the proof of Theorem \ref{thm:client_sprs}:} Suppose $\bar{\eta}$ is chosen such that $\bar{\eta}\leq \frac{1}{72L\delta_c C}$. Note that this choice of $\bar{\eta}$ meets the requirements for    Lemmas \ref{lemma:clientsprs_bnd} and \ref{lemma:sparse_err} to hold. Now based on Lemmas \ref{lemma:clientsprs_bnd} and \ref{lemma:sparse_err}, and the definition of $\psi_t$, we have
\begin{equation}
    \begin{aligned}
    \psi_{t+1}&={\Vert \tilde{x}_{t+1} - x^* \Vert}^2 + \bar{\eta}^2 V_t\\
   & \leq {\Vert \tilde{x}_{t}-x^* \Vert}^2 -2 \bar{\eta} \left(f(\tilde{x}_{t})-f(x^*) \right) + 2\bar{\eta}^2 (7+2\delta_c C) {\Vert \nabla f(\tilde{x}_t) \Vert}^2\\
   & \hspace{5mm} + \left(1-\frac{1}{2\delta_c}+ 4 \bar{\eta}^2 L^2 \delta_c C + 24 \bar{\eta} L \right) \bar{\eta}^2 V_{t-1} + 2 \bar{\eta}^2 (6\bar{\eta} L + \delta_c) D \\
  &  \overset{(a)} \leq {\Vert \tilde{x}_{t}-x^* \Vert}^2 -2 \bar{\eta} \left(f(\tilde{x}_{t})-f(x^*) \right) + 18 \bar{\eta}^2 \delta_c C {\Vert \nabla f(\tilde{x}_t) \Vert}^2\\
  & \hspace{5mm} + \left(1-\frac{1}{2\delta_c}+ 28 \bar{\eta} L \right) \bar{\eta}^2 V_{t-1} + 2 \bar{\eta}^2 \left(\frac{6}{\delta_c C} + \delta_c\right) D \\
 & \overset{(b)} \leq {\Vert \tilde{x}_{t}-x^* \Vert}^2 - 2\bar{\eta} \left(1-18\bar{\eta}L\delta_c C \right) \left(f(\tilde{x}_{t})-f(x^*) \right) + \left(1-\frac{1}{2\delta_c}+ 28 \bar{\eta} L \right) \bar{\eta}^2 V_{t-1}\\
 & \hspace{5mm}+ 2 \bar{\eta}^2\left(\frac{6}{\delta_c C} + \delta_c\right) D\\
 & \overset{(c)} \leq \left(1-\frac{3}{4}\bar{\eta} \mu \right){\Vert \tilde{x}_{t}-x^* \Vert}^2 + \left(1-\frac{1}{2\delta_c}+ 28 \bar{\eta} L \right) \bar{\eta}^2 V_{t-1} + 2 \bar{\eta}^2\left(\frac{6}{\delta_c C} + \delta_c\right) D.
    \end{aligned}
\end{equation}
For (a), we used $\delta_c \geq 1, C \geq 1$, and $\bar{\eta}L \delta_c C \leq 1$ to simplify the preceding inequality; for (b), we used the $L$-smoothness of $f(\cdot)$; for (c), we used the fact that $\bar{\eta}L \delta_c C \leq \frac{1}{72}$, and that $f(\cdot)$ is $\mu$-strongly convex. Now given our choice of $\bar{\eta}$, observe that
$$ 1-\frac{1}{2\delta_c}+28\bar{\eta}L \leq 1-\frac{1}{2\delta_c C}+\frac{28}{72\delta_c C} = 1-\frac{1}{9\delta_c C} < 1-\frac{1}{96 \delta_c C \kappa} \leq 1-\frac{3}{4}\bar{\eta}\mu,$$
where $\kappa=\frac{L}{\mu}.$ Thus,
\begin{equation}
\begin{aligned}
    \psi_{t+1} &\leq \left(1-\frac{3}{4}\bar{\eta} \mu\right)\left({\Vert \tilde{x}_{t}-x^* \Vert}^2 +  \bar{\eta}^2 V_{t-1}\right) + 2 \bar{\eta}^2\left(\frac{6}{\delta_c C} + \delta_c\right) D\\
 &=\left(1-\frac{3}{4}\bar{\eta} \mu\right)\psi_{t} + 2 \bar{\eta}^2\left(\frac{6}{\delta_c C} + \delta_c\right) D. 
\end{aligned}
\end{equation}
Using the above inequality recursively, we obtain
\begin{equation}
\begin{aligned}
    \psi_{T+1} &\leq {\left(1-\frac{3}{4}\bar{\eta} \mu\right)}^T \psi_{1} + 2 \bar{\eta}^2 \left(\frac{1-{\left(1-\frac{3}{4}\bar{\eta} \mu\right)}^T}{1-\left(1-\frac{3}{4}\bar{\eta} \mu\right)}\right) \left(\frac{6}{\delta_c C} + \delta_c\right) D \\
    & \leq {\left(1-\frac{3}{4}\bar{\eta} \mu\right)}^T \psi_{1} +\frac{8}{3}\bar{\eta} \left(\frac{6}{\delta_c C} + \delta_c\right) \frac{D}{\mu}. 
\end{aligned}
\label{eqn:clientpsrs_finalbd}
\end{equation}
Now since $\rho_{i,0}=0, \forall i\in\mathcal{S}$, we have $\rho_{0}=0$, and $V_0=0$. It thus follows that $\tilde{x}_1=\bar{x}_1-\bar{\eta}\rho_0=\bar{x}_1$, and $\psi_1={\Vert \tilde{x}_{1} - x^* \Vert}^2 + \bar{\eta}^2 V_0={\Vert \bar{x}_1-x^* \Vert}^2$. Finally, observe that
\begin{equation}
\begin{aligned}
    {\Vert \bar{x}_{T+1}-x^* \Vert}^2 & = {\Vert \bar{x}_{T+1}- \tilde{x}_{T+1} + \tilde{x}_{T+1} - x^* \Vert}^2 \\
    & \leq 2 {\Vert \tilde{x}_{T+1} - x^* \Vert}^2+ 2 {\Vert \bar{x}_{T+1}- \tilde{x}_{T+1} \Vert}^2 \\
    & = 2{\Vert \tilde{x}_{T+1} - x^* \Vert}^2 + 2 \bar{\eta}^2 {\Vert\rho_{T}\Vert}^2\\
    & \leq 2 \left({\Vert \tilde{x}_{T+1} - x^* \Vert}^2+\bar{\eta}^2V_T\right)\\
    & = 2\psi_{T+1}. 
\end{aligned}
\end{equation}
Based on the above discussion, and \eqref{eqn:clientpsrs_finalbd}, we have
    \begin{equation}
    {\Vert \bar{x}_{T+1} - x^* \Vert}^2 \leq 2 { \left(1-\frac{3}{4}\bar{\eta} \mu \right)}^T {\Vert \bar{x}_{1} - x^* \Vert}^2 + \frac{16}{3}\bar{\eta}\left(\frac{6}{\delta_c C}+\delta_c\right)\frac{D}{\mu},
\end{equation}
which is precisely the desired conclusion.
\newpage
\section{Proof of Theorem \ref{thm:serv_sprs2}: Gradient Sparsification at Server with Error-Feedback}
\label{app:servsparse2}
The proof of Theorem \ref{thm:serv_sprs2} follows similar conceptual steps as that of Theorem \ref{thm:client_sprs}. Thus, we will only sketch the main arguments, leaving the reader to verify the details. Given that $\delta_c=1$, we note that the governing equations for this setting are as follows. \begin{equation}
    \begin{aligned}
    x_{i,\ell+1}&=x_{i,\ell}-\eta_i(\nabla f_i(x_{i,\ell}) - \nabla f_i(\bar{x}_t) + g_t)\\
    g_t &= \mathcal{C}_{\delta_s}\left(e_{t-1} + \nabla f(\bar{x}_t) \right)\\
    e_t &= e_{t-1} + \nabla f(\bar{x}_t) - g_t.
    \end{aligned}
\label{eqn:servsprs_eqs}
\end{equation}
The first and the third equations above hold for every communication round $t\in\{1,\ldots,T\}$, whereas the second equation holds for every $t\in\{2,\ldots,T\}$. Initially, we have $g_1=\nabla f(\bar{x}_1)$, and $e_1=e_0=0$. Let us define the following virtual sequence $\{\tilde{x}_t\}$:
\begin{equation}
    \tilde{x}_t=\bar{x}_t-\bar{\eta} e_{t-1},
\label{eqn:virt_seq_serv}
\end{equation}
where $\bar{\eta}=\eta_i \tau_i, \forall i\in\mathcal{S}.$ Then, based on the definition of the virtual sequence, and \eqref{eqn:servsprs_eqs}, it is easy to verify that
\begin{equation}
 \tilde{x}_{t+1} = \tilde{x}_t - \frac{1}{m} \sum\limits_{i=1}^{m} \eta_i \sum \limits_{\ell=0}^{\tau_i-1} \nabla f_i(x_{i,\ell}).
\end{equation}
Following exactly the same steps as in the proof of Lemma \ref{lemma:iterate_subopt_ct}, we obtain
\begin{equation}
{\Vert \tilde{x}_{t+1}-x^* \Vert}^2  \leq {\Vert \tilde{x}_{t}-x^* \Vert}^2 -2 \bar{\eta} \left(f(\tilde{x}_{t})-f(x^*) \right) + \frac{L(1
+2\bar{\eta} L)}{m} \sum\limits_{i=1}^{m} \eta_i \sum\limits_{\ell=0}^{\tau_i-1}{\Vert x_{i,\ell} - \tilde{x}_t \Vert}^2 +  2{\bar{\eta}}^2 {\Vert \nabla f(\tilde{x}_t) \Vert}^2.
\label{eqn:iteratebnd_servsprs2}
\end{equation}
Just as in the proof of Theorem \ref{thm:client_sprs}, our next task is to derive a bound on ${\Vert x_{i,\ell} - \tilde{x}_t \Vert}^2$. To this end, we start with
\begin{equation}
    { \Vert \tilde{x}_t - x_{i,\ell} \Vert}^2 \leq 2\bar{\eta}^2 \left( {\Vert e_{t-1} \Vert}^2 + {\Vert g_t \Vert}^2 \right).
\label{eqn:servsprsbnd1}
\end{equation}
To arrive at the above inequality, we used \eqref{eqn:virt_seq_serv}, and the fact that $\Vert \bar{x}_t - x_{i,\ell} \Vert \leq \bar{\eta} \Vert g_t \Vert$. Next, observe that
\begin{equation}
    \begin{aligned}
    {\Vert g_t \Vert}^2 &= {\Vert \mathcal{C}_{\delta_s}\left(e_{t-1} + \nabla f(\bar{x}_t) \right) \Vert}^2\\ 
    & \leq {\Vert e_{t-1} + \nabla f(\bar{x}_t) \Vert}^2\\
    & \leq 2 {\Vert e_{t-1} \Vert}^2 + 2 {\Vert \nabla f(\bar{x}_t) \Vert}^2.
    \end{aligned}
\end{equation}
Using the smoothness of $\nabla f(\cdot)$, we also have
\begin{equation}
\begin{aligned}
    {\Vert \nabla f(\bar{x}_t) \Vert}^2 &= {\Vert \nabla f(\bar{x}_t) - \nabla f(\tilde{x}_t) + \nabla f(\tilde{x}_t) \Vert}^2\\ 
& \leq 2 {\Vert \nabla f(\bar{x}_t) - \nabla f(\tilde{x}_t)\Vert}^2 + 2 {\Vert \nabla f(\tilde{x}_t) \Vert}^2\\
& \leq 2L^2 {\Vert \bar{x}_t - \tilde{x}_t \Vert}^2 + 2 {\Vert \nabla f(\tilde{x}_t) \Vert}^2\\
&  = 2\bar{\eta}^2 L^2 {\Vert e_{t-1} \Vert}^2 + 2 {\Vert \nabla f(\tilde{x}_t) \Vert}^2.
\end{aligned}
\label{eqn:servsprsbnd2}
\end{equation}
Suppose $\bar{\eta}$ is chosen such that $\bar{\eta} L \leq \frac{1}{2}$. Combining the bounds we have derived above, we can then obtain
\begin{equation}
    { \Vert \tilde{x}_t - x_{i,\ell} \Vert}^2 \leq 8\bar{\eta}^2 \left( {\Vert e_{t-1} \Vert}^2 + {\Vert \nabla f(\tilde{x}_t) \Vert}^2 \right). 
\end{equation}
Substituting the above bound in \eqref{eqn:iteratebnd_servsprs2}, and simplifying the resulting inequality leads to
\begin{equation}
 {\Vert \tilde{x}_{t+1}-x^* \Vert}^2   \leq {\Vert \tilde{x}_{t}-x^* \Vert}^2 -2 \bar{\eta} \left(f(\tilde{x}_{t})-f(x^*) \right) + 14\bar{\eta}^2 {\Vert \nabla f(\tilde{x}_t) \Vert}^2 + 24 \bar{\eta}^3 L {\Vert e_{t-1} \Vert}^2.
 \label{eqn:iterate_bnd2_servsprs2}
\end{equation}
Given the dependence of the above inequality on the gradient  sparsification error $e_{t-1}$, we next proceed to derive a recursion for bounding $\Vert e_t \Vert$. We follow similar steps as in Lemma \ref{lemma:sparse_err}. 
\begin{equation}
    \begin{aligned}
    {\Vert e_t \Vert}^2 &= {\Vert e_{t-1}+\nabla f(\bar{x}_t) -  \mathcal{C}_{\delta_s} \left(e_{t-1}+\nabla f(\bar{x}_t)\right) \Vert}^2\\
    & \overset{(a)} \leq  \left(1-\frac{1}{\delta_s}\right)  {\Vert e_{t-1}+\nabla f(\bar{x}_t) \Vert}^2\\
    & \overset{(b)} \leq \left(1-\frac{1}{\delta_s}\right) (1+\gamma) { \Vert e_{t-1} \Vert }^2 + \left(1-\frac{1}{\delta_s}\right)\left(1+\frac{1}{\gamma}\right) {\Vert \nabla f(\bar{x}_t) \Vert}^2\\
    & \overset{(c)} \leq \left(1-\frac{1}{2 \delta_s}\right) { \Vert e_{t-1} \Vert }^2 + 2 \delta_s {\Vert \nabla f(\bar{x}_t) \Vert}^2 \\
   & \overset{(d)} \leq \left(1-\frac{1}{2 \delta_s}+4\bar{\eta}^2L^2 \delta_s \right) { \Vert e_{t-1} \Vert }^2 + 4 \delta_s {\Vert \nabla f(\tilde{x}_t) \Vert}^2. 
    \end{aligned}
\label{eqn:sparse_err_bnd_serv2}
\end{equation}
In the above steps, for (a), we used Property 3 of the \texttt{TOP-k} operator in Lemma \ref{lemma:topK}; for (b), we used \eqref{eqn:rel_triangle}; for (c), we set $\gamma=\frac{1}{2(\delta_s-1)}$; finally, for (d), we used the bound on ${\Vert \nabla f(\bar{x}_t) \Vert}^2$ in \eqref{eqn:servsprsbnd2}. We now have all the individual pieces required to complete the proof of Theorem \ref{thm:serv_sprs2}. To proceed, let us define the following Lyapunov function:
$$\psi_{t} \triangleq { \Vert \tilde{x}_t - x^* \Vert }^2 + \bar{\eta}^2 {\Vert e_{t-1} \Vert}^2.$$ 
Referring to \eqref{eqn:iterate_bnd2_servsprs2} and \eqref{eqn:sparse_err_bnd_serv2}, using the fact that $f(\cdot)$ is $\mu$-strongly convex, and following similar arguments as in the proof of Theorem \ref{thm:client_sprs}, we obtain
\begin{equation}
    \begin{aligned}
    \psi_{t+1} & \leq  {\Vert \tilde{x}_{t}-x^* \Vert}^2 - 2\bar{\eta} \left(1-18\bar{\eta}L\delta_s  \right) \left(f(\tilde{x}_{t})-f(x^*) \right) + \left(1-\frac{1}{2\delta_s}+ 28 \bar{\eta} L \right) \bar{\eta}^2 {\Vert e_{t-1} \Vert}^2\\
 & \leq \left(1-\frac{3}{4}\bar{\eta} \mu \right){\Vert \tilde{x}_{t}-x^* \Vert}^2 + \left(1-\frac{1}{2\delta_s}+ 28 \bar{\eta} L \right) \bar{\eta}^2 {\Vert e_{t-1} \Vert}^2\\
 & \leq \left(1-\frac{3}{4}\bar{\eta} \mu \right)\left({\Vert \tilde{x}_{t}-x^* \Vert}^2 + \bar{\eta}^2 {\Vert e_{t-1} \Vert}^2\right)\\
 & = \left(1-\frac{1}{96 \delta_s \kappa}\right) \psi_t.
    \end{aligned}
\end{equation}
In the last two steps, we used $\bar{\eta}=\frac{1}{72L\delta_s}$, implying ${\eta_i}=\frac{1}{72L\delta_s\tau_i}$. The rest of the proof follows by recursively using the above inequality in conjunction with the following easily verifiable facts:
$$ \psi_1= {\Vert \bar{x}_1 - x^* \Vert}^2 \hspace{2mm} \textrm{since} \hspace{1mm} e_0=0; \hspace{2mm} {\Vert \bar{x}_{T+1} - x^* \Vert}^2 \leq 2 \psi_{T+1}. 
$$

\newpage
\section{Simulation Results for \texttt{FedSplit}}
\label{fedsplit:num}
In \cite{fedsplit}, the authors introduce \texttt{FedSplit} - an algorithmic framework based on operator-splitting procedures. Given an initial global model $\bar{x}_1$, the update rule of 
\texttt{FedSplit} is given by
\begin{equation}
\begin{aligned}
    y_i^{(t)}&=\mbox{\texttt{prox\_update}}_i(2\bar{x}_t-z_i^{(t)}),\\
    z_i^{(t+1)}&=z_i^{(t)}+2(y_i^{(t)}-\bar{x}_t),\\
    \bar{x}_{t+1}&=\dfrac{1}{m}\sum\limits_{i \in \mathcal{S}} z_i^{(t+1)},
\end{aligned}
\end{equation}
where $z_i^{(1)}=\bar{x}_1$, $i \in \mathcal{S}$. The local update at client $i$ is defined in terms of a proximal solver $\texttt{prox\_update}_i(\cdot)$. Ideally, this proximal solver would be an exact evaluation of the following proximal operator for some step-size $s>0$:
\begin{equation}
    \mbox{\textbf{prox}}_{sf_i}(u)\coloneqq \argmin\limits_{x\in \mathbb{R}^d}\Big\{\underbrace{f_i(x)+\dfrac{1}{2s}\|u-x\|^2}_{h_i(x)}\Big\}. 
\end{equation}
As suggested in \cite{fedsplit}, in practice, \texttt{FedSplit} would be implemented using an approximate proximal solver. One way to do so, as clearly detailed in \cite{fedsplit}, is to run $e$ steps of gradient descent on $h_i(x)$ using a suitably chosen step-size $\alpha$. The latter is precisely the method we use to numerically implement \texttt{FedSplit}. As per Corollary 1 in \cite{fedsplit}, \texttt{FedSplit} achieves linear convergence to a neighberhood of the global minimum for any value of $e$. In what follows, we show that \texttt{FedSplit} may diverge even under the simplest of settings. In particular, we consider an instance of problem \eqref{eqn:objective} where two clients with simple quadratics attempt to minimize the global objective function \eqref{eqn:objective} using \texttt{FedSplit}. The local objective function of client $1$ is given by
\begin{equation*}
    f_1(x)=\dfrac{1}{2}x^{T}
    \underbrace{\begin{bmatrix}
L & 0\\
0 & \mu
\end{bmatrix}}_{A_1}
x-
 \underbrace{\begin{bmatrix}
1\\
1
\end{bmatrix}^T}_{B_1^T}x,
\end{equation*}
 and the local objective function of client $2$ is given by
 \begin{equation*}
    f_2(x)=\dfrac{1}{2}x^{T}
    \underbrace{\begin{bmatrix}
\mu & 0\\
0 & \mu
\end{bmatrix}}_{A_2}
x-
 \underbrace{\begin{bmatrix}
-1\\
2
\end{bmatrix}^T}_{B_2^T}x, 
\end{equation*}
where $L=1000$, and $\mu=1$. The step-sizes corresponding to the proximal operator and gradient descent, namely $s$ and $\alpha$, respectively, are chosen as per Corollary $1$ in \cite{fedsplit}. Furthermore, we run $e$ rounds of gradient descent per communication round $t$ for $e \in \{1, \cdots,41\}$. Given the fact that the implementation code of \texttt{FedSplit} is not publicly available, we note that our local implementation of the scheme has diverged for all the odd values of $e$ between $1$ and $41$, inclusive. It should be noted, however, that our implementation of \texttt{FedSplit} converged for some even values of $e$ in the considered range, as shown in Figure \ref{fedsplit_plot}. We have further observed that increasing the ratio $\kappa=L/\mu$ beyond $1000$ causes \texttt{FedSplit} to diverge for values of $e$ higher than $41$ as well. 
\begin{figure}[t!]
  \centering
  \includegraphics[width=\linewidth]{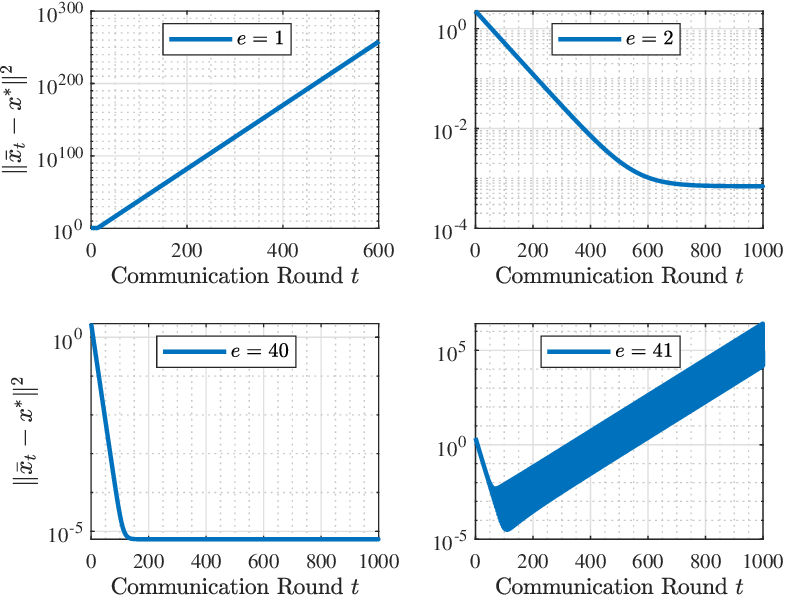}
  \caption{Simulation results for \texttt{FedSplit} for $e \in \{1,2,40,41\}$.}
\label{fedsplit_plot}
\end{figure}
\vspace*{3in}
\newpage
\section{Additional Experimental Results on Logistic Regression}
\label{add_exp}
In this section, we provide additional numerical results for \texttt{FedLin} on a logistic regression problem. For each client $i \in \mathcal{S}$, the design matrix $A_i \in \mathbb{R}^{m_i\times d}$ is a collection of $m_i$ feature vectors, with the $j$-th feature vector denoted by $a_{ji}$, $j \in \{1, \cdots, m_i\}$. In turn, every feature vector $a_{ji}$ is associated with a class label $b_{ji} \in \{+1,-1\}$. In a logistic regression problem, the conditional probability of observing a positive class label $b_{ji}=+1$ for a given feature vector $a_{ji}$ is 
\begin{equation}
\label{model_lr}
    \mathcal{P}\big(b_{ji}=+1\big)=\dfrac{1}{1+e^{-a_{ji}^Tx}},
\end{equation}
where $x \in \mathbb{R}^d$ is an unknown parameter vector to be estimated. The maximum likelihood estimate of the parameter vector $x$ is then the solution of the following convex optimization problem
\begin{equation}
\label{lrp}
    \min \limits_{x \in \mathbb{R}^d} f(x)=\min \limits_{x \in \mathbb{R}^d} \dfrac{1}{m}\sum \limits_{i=1}^m \underbrace{\sum \limits_{j=1}^{m_i}\log(1+e^{-b_{ji}a_{ji}^Tx})}_{f_i(x)}. 
\end{equation}
The client objective functions, $f_i(x)$, are both smooth and convex. To generate synthetic data, for each client $i \in \mathcal{S}= \{1, \cdots,10\}$, we generate the design matrix $A_i$ and the corresponding class labels according to the model \eqref{model_lr}, where $x \in \mathbb{R}^{100}$ and $A_i \in \mathbb{R}^{500 \times 100} $. In particular, the entries of the design matrix are modeled as $[A_i]_{jk}\stackrel{i.i.d.}{\sim}\mathcal{N}(0,1)$ for $j\in \{1, \cdots,500\}$ and $k \in \{1, \cdots, 100\}$. The entries of the true parameter vector $x$ are modeled as $[x]_{\ell} \stackrel{i.i.d.}{\sim} \mathcal{N}(0,1)$ for $\ell \in \{1,\cdots, 100\}$. To model the effect of systems heterogeneity, we allow clients to perform different numbers of local iterations. In particular, for each client $i \in \mathcal{S}$, the number of local iterations is drawn independently from a uniform distribution over the range $[2,50]$, i.e, $\tau_i \in [2,50]$, $\forall i \in \mathcal{S}$.

\textbf{Gradient Sparsification at Server.}
We first consider a variant of \texttt{FedLin} where gradient sparsification is implemented only at the server side without any error-feedback. In particular, we consider the cases where $\delta_s \in \{2,4\}$, which correspond to the implementation of a \texttt{TOP-50} and a \texttt{TOP-25} operator on the communicated gradients, respectively. For comparison, we also plot the resulting performance when no gradient sparsification is implemented at the server side, i.e. $\delta_s=1$. As illustrated in Figure \ref{lp_server}, regardless of the level of gradient sparsification at the server side, \texttt{FedLin} always converges to the true minimizer.

\textbf{Gradient Sparsification at Clients.} Next, we implement gradient sparsification only at the clients' side, i.e. $\delta_s=1$. In particular, we consider the cases where $\delta_c \in \{1.25, 1.67\}$, which correspond to the implementation of a \texttt{TOP-80} and a \texttt{TOP-60} operator on the communicated local gradients, respectively. As illustrated in Figure \ref{lp_client}, unlike the server case, \texttt{FedLin} with sparsification at the clients' side converges with a non-vanishing convergence error, which increases as the value of $\delta_c$ increases. 
\begin{figure}[t]
  \centering
  \includegraphics[width=0.6\linewidth]{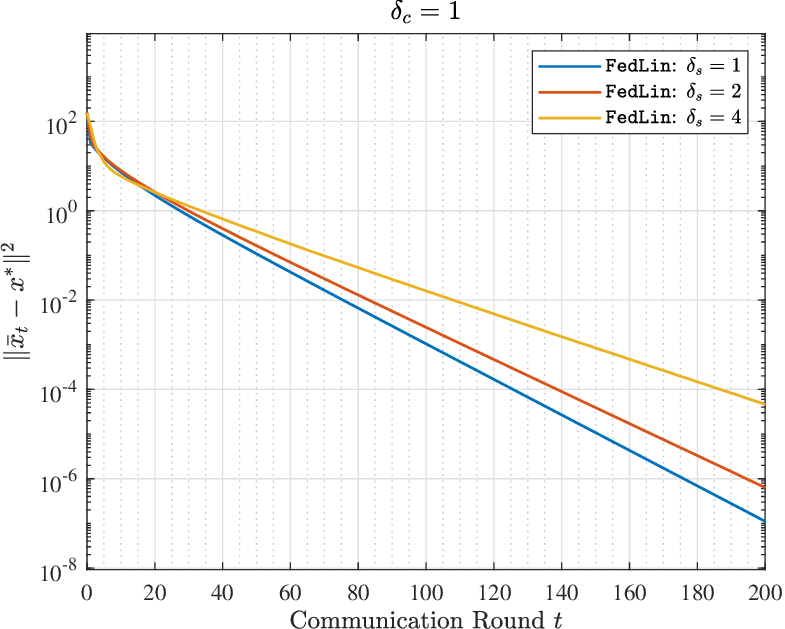}
  \caption{Simulation results for \texttt{FedLin} where gradient sparsification is implemented at the server side. The constant $\bar{\eta}$ is fixed at $0.15$ across all clients.}
\label{lp_server}
\end{figure}

\begin{figure}[t]
  \centering
  \includegraphics[width=0.6\linewidth]{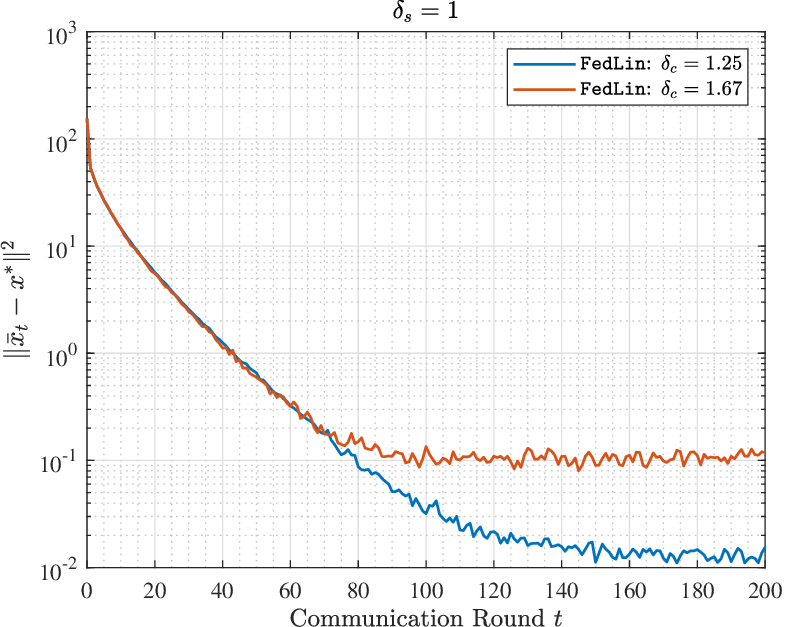}
  \caption{Simulation results for \texttt{FedLin} where gradient sparsification is implemented at the clients'  side. The constant $\bar{\eta}$ is fixed at $0.1$ across all clients.}
\label{lp_client}
\end{figure}
\clearpage
\section{Simulation Results for \texttt{FedLin} with Noisy Gradients}
\label{app:Noisy_FedLin_exp}
In this section, we provide numerical results for \texttt{FedLin} under noisy client gradients to validate the theoretical results of Theorem \ref{thm:noisy}. In particular, we consider the least square problem of Section \ref{sec:experiments} with $\delta_s=\delta_c=1$ and $\alpha=10$. All the remaining parameters are kept the same. To simulate noisy gradients, we add zero-mean Gaussian noise with variance $\sigma^2 \in \{10^{-5}, 10^{-3}, 10^{-1}\}$. 
\begin{figure}[!h]
  \centering
  \includegraphics[width=0.6\linewidth]{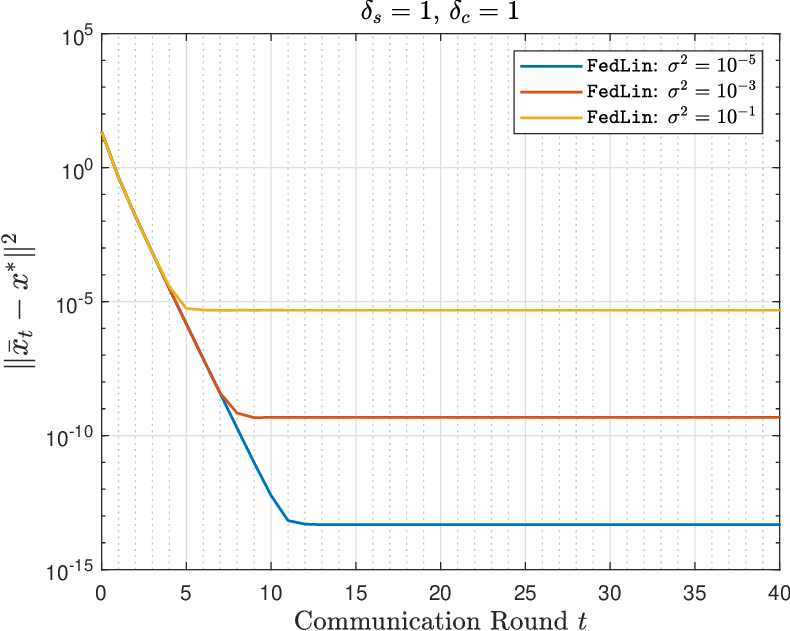}
  \vspace*{5 mm}
  \caption{Simulation results for \texttt{FedLin} under noisy client gradients. The constant $\bar{\eta}$ is fixed at $10^{-2}$ across all clients.}
\label{noisy_sim}
\end{figure}

As illustrated in Figure \ref{noisy_sim}, \texttt{FedLin} under noisy gradients converges with a non-vanishing error-floor, which increases as the variance of the noise increases. Thus, our simulations here corroborate the theory in Theorem \ref{thm:noisy}.
\end{document}